\documentclass[twoside,11pt]{article}

\usepackage{amsthm}
\usepackage{jmlr2e}

\usepackage[utf8]{inputenc}
\usepackage[T1]{fontenc}
\usepackage{microtype}
\usepackage{mathtools}
\usepackage{nicefrac}
\usepackage{dsfont}
\usepackage{booktabs}
\usepackage{enumitem}
\usepackage{pifont}
\usepackage{etoolbox}

\newcommand{\N}{\mathds{N}}

\newcommand{\R}{\mathds{R}}
\renewcommand{\P}{\mathds{P}}
\newcommand{\E}{\mathds{E}}
\newcommand{\iid}{\ensuremath{\stackrel{\mbox{\upshape\tiny i.i.d.}}{\sim}}}
\newcommand{\ind}[1]{\mathds{1}_{#1}} 

\newcommand\Sn{\mathfrak{S}_n} 
\newcommand\eps{\varepsilon}
\newcommand{\declareFonction}[5]{
	\begin{array}{l|rcl}
		#1: & #2 & \longrightarrow & #3 \\
			& #4 & \longmapsto     & #5 
	\end{array}
}
\newcommand{\Linf}{\text{\upshape L}^\infty}
\DeclareMathOperator{\Hom}{Hom}
\DeclareMathOperator{\Aut}{Aut}
\newcommand{\ins}[1]{\mathrm{#1}}

\newcommand{\eqdef}{\ensuremath{\stackrel{\mbox{\upshape\tiny def.}}{=}}}
\renewcommand{\d}{\ins{d}}

\DeclareMathOperator{\MAE}{MAE}
\DeclareMathOperator{\esssup}{ess\,sup}
\newcommand{\uesssup}[1]{\underset{#1}{\esssup}\;}
\DeclareMathOperator{\id}{\ins{id}}

\newcommand\GG{\mathcal{G}}
\newcommand\CC{\mathcal{C}}
\newcommand\X{\mathcal{X}}

\newcommand\FF{\mathcal{F}}
\newcommand\RRR{\mathcal{R}}

\newcommand\xmark{\ding{55}}

\newcommand\m{m}

\DeclareFontEncoding{LS2}{}{\noaccents@}
\DeclareFontSubstitution{LS2}{stix2}{m}{n}
\DeclareSymbolFont{largesymbolsstix}  {LS2}{stix2ex}   {m} {n}
\DeclareMathDelimiter{\lBrace}    
    {\mathopen}{largesymbolsstix}{"E8}{largesymbolsstix}{"0E}
\DeclareMathDelimiter{\rBrace}    
    {\mathclose}{largesymbolsstix}{"E9}{largesymbolsstix}{"0F}

\DeclarePairedDelimiterX\norm[1]\lVert\rVert{
	\ifblank{#1}{\:\cdot\:}{#1}
}
\DeclarePairedDelimiterX\abs[1]\lvert\rvert{
	\ifblank{#1}{\:\cdot\:}{#1}
}

\providecommand\given{}
\newcommand\SetSymbol[1][]{%
	\nonscript\:#1\vert
	\allowbreak
	\nonscript\:
	\mathopen{}}
\DeclarePairedDelimiterX\Set[1]\{\}{\renewcommand\given{\SetSymbol[\delimsize]}#1}

\DeclarePairedDelimiterX\multiSet[1]\lBrace\rBrace{
		\ifblank{#1}{\:\cdot\:}{\renewcommand\given{\SetSymbol[\delimsize]}#1}
}

\renewcommand{\leq}{\leqslant}
\renewcommand{\le}{\leqslant}
\renewcommand{\geq}{\geqslant}
\renewcommand{\ge}{\geqslant}

\newcommand{\qandq}{ \quad \text{and} \quad }

\newcommand{\ie}{{\em i.e.,~}}
\newcommand{\eg}{{\em e.g.,~}}

\usepackage[noabbrev, nameinlink, capitalise, english]{cleveref}
\makeatletter
\def\cleartheorem#1{\expandafter\let\csname#1\endcsname\relax
    \expandafter\let\csname c@#1\endcsname\relax
}
\makeatother
\cleartheorem{example}
\cleartheorem{theorem}
\cleartheorem{lemma}
\cleartheorem{proposition}
\cleartheorem{definition}
\cleartheorem{corollary}

\newtheorem{example}{Example}
\newtheorem{theorem}{Theorem}
\newtheorem{lemma}[theorem]{Lemma}
\newtheorem{proposition}[theorem]{Proposition}
\newtheorem{corollary}[theorem]{Corollary}
\newtheorem{definition}[theorem]{Definition}

\crefname{inequality}{Inequality}{Inequalities}                              
\creflabelformat{inequality}{#2(#1)#3}                                       
\crefname{assumption}{Assumption}{Assumptions}

\newtheorem{assumption}[theorem]{Assumption}
\newtheorem{myexample}{Example} 
\crefalias{myexample}{example}

\usepackage{lastpage}
\jmlrheading{25}{2024}{1-\pageref{LastPage}}{7/23; Revised 8/24}{11/24}{23-0965}{Matthieu Cordonnier, Nicolas Keriven, Nicolas Tremblay and Samuel Vaiter}

\ShortHeadings{Convergence of MPGNNs with Generic Aggregation}{Cordonnier, Keriven, Tremblay and Vaiter}
\firstpageno{1}

\begin{document}
	
	\title{Convergence of Message-Passing Graph Neural Networks with Generic Aggregation on Large Random Graphs}
	
	\author{\name Matthieu Cordonnier \email matthieu.cordonnier@gipsa-lab.fr \\
		\addr GIPSA-lab, Université Grenoble Alpes, Grenoble
		\AND
		\name Nicolas Keriven \email nicolas.keriven@cnrs.fr \\
		\addr CNRS, IRISA, Rennes
		\AND
		\name Nicolas Tremblay \email nicolas.tremblay@cnrs.fr \\
		\addr CNRS, GIPSA-lab, Grenoble
		\AND
		\name Samuel Vaiter \email samuel.vaiter@cnrs.fr \\
		\addr CNRS, LJAD, Nice}
	
	\editor{Joan Bruna}
	
	\maketitle
	
	\begin{abstract}
        We study the convergence of message-passing graph neural networks on random graph models toward their continuous counterparts as the number of nodes tends to infinity. 
        Until now, this convergence was only known for architectures with aggregation functions in the form of normalized means, or, equivalently, of an application of classical operators like the adjacency matrix or the graph Laplacian. 
        We extend such results to a large class of aggregation functions, that encompasses all classically used message-passing graph neural networks, such as attention-based message-passing, max convolutional message-passing, (degree-normalized) convolutional message-passing, or moment-based aggregation message-passing. 
        Under mild assumptions, we give non-asymptotic bounds with high probability to quantify this convergence. 
        Our main result is based on the McDiarmid inequality. 
        Interestingly, this result does not apply to the case where the aggregation is a coordinate-wise maximum. 
        We treat this case separately and obtain a different convergence rate.
	\end{abstract}
	


\section{Introduction} \label{intro}
Graph Neural Networks (GNNs)~\citep{scarselli_graph_2008,gori_new_2005} are deep learning architectures largely inspired by Convolutional Neural Networks, that aim to extend convolu\-tional methods to signals on graphs. 
Indeed, in many domains, the measured data live on a graph structure: examples for which GNNs have achieved state-of-the-art performance include molecules, proteins, and node clustering~\citep{gilmer_neural_2017,chen2018supervised,Fout2017ProteinIP}. 
Nevertheless, it has been observed that GNNs have limitations, both in practice \citep{wu_simplifying_19,hu_open_graph_benchmark_2020} and in their theoretical understanding~\citep{morris_futur_2024}. 
Hence, the design of more reliable and powerful architectures is a current active and fast evolving area of research.

From a theoretical perspective, a large part of the literature has focused on the \emph{expressive power} of GNNs, \ie what class of functions can GNNs approximate. 
This notion is fundamental in classical Deep Learning and is related to the so-called \emph{Universal Approxima\-tion Theorem}~\citep{hornik_approximation_1991,cybenkot_approximation_1989}. 
Studying the expressive power of GNNs is however more involved, as they are by definition designed to be invariant or equivariant to the relabeling of nodes in a graph (see~\cref{sec: MPGNN definition}). 
Hence, in~\cite{xu_how_WL_2019} the authors relate their expressivity to the \emph{graph isomorphism problem}, that is, deciding if two graphs are isomorphic of one another, a long-standing combinatorial problem in graph theory. 
The main avenue to analyze the expressive power of GNNs compares their performances to the traditional Weisfeiler-Lehman algorithm (WL)~\citep{weisfeiler_reduction_1968}, which process is very similar to the message-passing paradigm at the core of GNNs. 
Hence, by construction, basic GNNs are at most as powerful as WL~\citep{xu_how_WL_2019}. 
From this point, a lot of effort has been made to design innovative GNN architectures to outperform the classical WL~\citep{maron_universality_2019, maron_provably_2019,keriven_universal_2019,vignac_building_2020,papp2022theoretical,morris_weisfeiler_2019}.

Nevertheless, while this combinatorial approach is worth considering for reasonably small graphs, its relevance in the context of large graphs is somewhat limited. 
Two real large graphs may share similar patterns, but will never be isomorphic, one main simple reason being that they most likely do not even share the same number of nodes. 
Large graphs are better described by some global properties such as edge density or number of communities. 
To that extent, the privileged mathematical tools are random graph models~\citep{crane2018probabilistic, goldenberg2010survey}. 
A generic family of models of interest to study GNNs on large graphs is the class of Latent Position Models~\citep{keriven_convergence_2020,keriven_universality_2021, ruiz_graphon_neural_networks_2020,levie_transferability_2021,maskey_generalization_2022}. 
In such a model, the nodes of a random graph are first sampled as latent random variables (the latent positions) in a probability space $(\X,P)$, and then, the adjacency is decided via the sampling of a connectivity kernel $W:\X^2 \to [0,1]$ at the random latent positions. 
This encompasses models like stochastic block models~\citep{lei_consestency_2015} (SBM) or graphon models~\citep{lovasz_large_2012}, depending on how exactly we define the edge appearance procedure. 

The key idea in studying GNNs on large random graphs is to embed the discrete problem into a continuous setting for which we expect to understand their properties with more ease. 
For instance, some authors derived new properties of geometric stability on large random graphs \citep{keriven_convergence_2020, levie_transferability_2021, ruiz_stability_transferability_2021}, which are known to be key in other deep models \citep{mallat_scat_2012} but cannot easily be characterized on fixed deterministic graphs. 
Large random graphs also shed light on the expressive power of GNNs \citep{keriven_universality_2021} in a different manner than the discrete WL test, and some models are indeed proved to be more powerful than others \citep{keriven_what_2023}.
To study GNNs on large random graphs, we match the GNN to a ``continuous'' counterpart, referred to as a continuous-GNN (cGNN)~\citep{keriven_convergence_2020, ruiz_graphon_neural_networks_2020}. 
While the discrete GNN propagates a signal over the nodes of the graph, the cGNN propagates \emph{a mapping over the latent space $\X$}. 
Such a map can be interpreted as a signal over the graph where the ``continuum'' of nodes would be all the points of $\X$. 
Then, as the random graph grows large, the GNN must behave similarly to its cGNN counterpart. 
To justify this, it is necessary to describe the cGNN as a limit of GNNs on random graphs and to ensure that the GNN converges to the cGNN as the number of nodes increases~\citep{keriven_convergence_2020, maskey_generalization_2022}. 
This convergence problem is precisely the focus of the present work.  

The duality of the convolutional product has led to two ways of defining GNNs. 
On the one hand, convolution as a pointwise product of frequencies in the Fourier domain has justified the design of so-called Spectral Graph Neural Networks~\citep{defferrard_convolutional_2016} (SGNNs), in which one introduces a graph Fourier transform through a chosen graph shift operator~\citep{tremblay_design_2017} to legitimate the use of polynomial filters. 
On the other hand, the spatial interpretation sees the convolution as local aggregations of neighborhood information, leading to Message-Passing Neural Networks (MPGNNs)~\citep{gilmer_neural_2017,kipf_semi-supervised_2017, corso_principal_neighbourhood_agg_2020}. 
The message-passing paradigm consists of iteratively updating each node via the \textbf{aggregation} of messages from each of its neighbors. 
This framework is often favored due to its inherent flexibility: messages and aggregation functions are unconstrained as long as they stay invariant to node reordering, \ie as long as they match on isomorphic graphs. 
Besides, SGNNs layers are mostly made of polynomials of graph shift operators which are a form of message-passing, defined by a choice of graph shift operator and a polynomial degree. 
As such, SGNNs can be seen as a subcase of the more versatile message-passing framework.

\paragraph{Contributions.} 
In this paper, we study the convergence toward a continuous counterpart of MPGNNs with a \textbf{generic aggregation function} \citep{corso_principal_neighbourhood_agg_2020}, whereas previous work~\citep{keriven_convergence_2020,keriven_universality_2021,ruiz_stability_transferability_2021,maskey_generalization_2022} are restricted to SGNNs or MPGNNs with specific aggregations. 
We use a simple version of the Latent Space Model where random graphs are totally connected and weighted accordingly the sampling of the kernel $W$ at the latent positions. 
Our main result,~\cref{th: main result}, states that for MPGNNs  having a Lipschitz-type regularity, the discrete network on a large random graph is close to its continuous counterpart with high probability. 
We quantify this convergence via a non-asymptotic bound based on the McDiarmid concentration inequality for multivariate functions of independent random variables. 
A special treatment is given to the case where the aggregation is a coordinate-wise \emph{maximum}~\citep{fey_pytorch}. 
For that particular case,~\cref{th: main result} does not hold. 
Thus, we provide another proof of convergence based on a specific concentration inequality, in~\cref{th: main result max}. 
This results in a significantly different theoretical convergence rate.

\paragraph{Related work.} 
The closest related works to ours are the results from ~\cite{keriven_convergence_2020, keriven_universality_2021}, where they establish convergence of SGNNs on Latent Position random graphs. 
We also mention~\cite{maskey_generalization_2022} who study a particular case of MPGNN on large random graphs, where the aggregation is defined to be a mean normalized by the degree of the node, which is akin to an SGNN using the degree-normalized Laplacian matrix. 
The present paper can be considered as a direct extension of both these works in the setting of MPGNNs with generic aggregation.

Further, the concept of limit of a SGNN on large random graphs has shown fruitful to tackle several problems. 
For instance, multiple works from different authors, among which~\cite{keriven_convergence_2020, maskey_transferability_2023, levie_transferability_2021, ruiz_stability_transferability_2021, cervino_transference_2023}, have focused on stability to deformation or transferability. 
The idea is that, since the same GNN can be applied to any graph, no matter its size or structure, we expect the outputs to be close on similar graphs, which is particularly relevant for large random graphs drawn from the same (or almost the same) model. 
Concerning the expressive power on large graphs,~\cite{keriven_convergence_2020, keriven_universality_2021} exploit their convergence theorems to propose a description of the function space that SGNNs on random graphs can approximate in~\cite{keriven_what_2023} and derive certain properties of universality. 
About other topics related to the learning procedure such as generalization as well as oversmoothing, the authors in~\cite{maskey_generalization_2022} derive a generalization bound that gets tighter for large graphs, while the results described in~\cite{keriven_oversmoothing_2022} make use of Latent Position random graphs to search a threshold between beneficial finite smoothing and oversmoothing. 

Beyond random graphs, large (dense) graphs can be described through the theory of graphons~\citep{lovasz_large_2012}, and several works aim to characterize the convergence of GNNs on large graphs with these mathematical tools. 
In~\cite{ruiz_graphon_neural_networks_2020, maskey_transferability_2023}, the authors define limits of graph polynomial filters of SGNNs designed from graph shift operators as integral operators with regard to the underlying graphon, and make use of the theory of self-adjoint operators and Hilbert spaces to study them. 
More recently, authors in~\cite{boker_finegrained_2023} consider a continuous version of the WL test via graphon estimation to study expressive power and the paper~\cite{levie_graphonsignal_2023} is devoted to extending the concept of sampling graphs from graphon to sampling graph signals from graphon signals.

\paragraph{Outline.} 
In~\cref{sec: preli}, we give some basic definitions. 
In~\cref{sec: MPGNN definition} we define MPGNNs with a generic aggregation function, that is, any function on sets used to gather and combine neighborhood information in the message-passing paradigm. 
In~\cref{sec: cMPGNN} we introduce continuous-MPGNNs (cMPGNNs) which are the counterparts of discrete MPGNNs that propagate a function over a compact probability space, alongside a connectivity kernel. 
As a discrete MPGNN must be coherent with graph isomorphism, we give mild conditions under which the cMPGNN is coherent with respect to some notion of probability space isomorphism. 
In~\cref{sec: Cv}, we focus on MPGNNs when applied on random graphs and describe what class of cMPGNN would be their natural limit. 
Our main result is~\cref{th: main result}: it provides necessary conditions under which the discrete network converges to its continuous counterpart. 
We make use of the McDiarmid concentration inequality to derive a non asymptotic bound with high probability of the deviation between the outputs of the MPGNN and its limit cMPGNN. 
Overall, we conclude that a sufficient condition of convergence is for the aggregation to have sharp bounded differences. 
Along the paper, we illustrate our concepts on classical GNN examples from the basic Graph Convolutional Network to the more sophisticated Graph Attentional Network~\citep{velickovic_graph_2018}. 
We give a particular treatment to the case of \emph{maximum} aggregation. 
Indeed, its behavior turns out to be significantly different than the other examples and does not fit into the class of MPGNNs having sharp bounded differences. 
Nevertheless, in~\cref{th: main result max} we make use of other specific concentration bounds to prove another non asymptotic bound between max MPGNN and its limit cMPGNN, with a significantly different convergence rate.

\section{Notations and Definitions} \label{sec: preli}
We start by introducing the notations that will hold throughout the paper. 
The letter $d$ (and its derivatives $d_0,\ldots$) will represent the dimension of a real vector space, the letter $n$ will denote the number of nodes in a graph, and $L$ will refer to the total number of layers in a deep architecture. 
Whenever we need to index something relatively to the vertices of a graph, we use a subscript indexation (\eg $z_i$) and in the case of layers, we employ a parenthesized superscript (\eg $z^{(l)}$).

We fix a positive integer $d$ and $(\R^d, \Vert\cdot\Vert_{\infty}, \mathds{B}(\R^d))$  the $d$-dimensional real vector space endowed with the infinite norm $\Vert x\Vert_{\infty} \eqdef \max_i |x_i|$ as well as its Borel sigma algebra. 
Except when specified differently, any topological concept, such as balls, continuity, \emph{etc.}, will be considered relatively to the norm $\Vert\cdot\Vert_{\infty}$. 
All along this paper, $\X$ is a compact subset of $\R^d$ and $\mathds{B}(\X)$ its Borel sigma algebra defined as the sigma algebra generated by the $U\cap\X$, for the open sets $U$ of $\R^d$.

The group of permutations of $\{1,\dots,n\}$ is denoted as $\Sn$. 
If $x=(x_1,\dots,x_n)$ is an $n$-tuple and $\sigma$ an element of $\Sn$, we define the $n$-tuple $\sigma\cdot x$ as $\sigma\cdot x \eqdef (x_{\sigma^{-1}(1)},\dots,x_{\sigma^{-1}(n)})$.

The set of bijections $\phi$ of $\X$  such that both $\phi$ and  $\phi^{-1}$ are measurable is a group for the composition of functions. 
We call it the group of automorphisms of $\X$ and denote it as $\Aut(\X)$. 
We denote as $\mathcal{P}(\X)$ the set of probability measures on $(\X, \mathds{B}(\X))$. 
For a measure $P\in\mathcal{P}(\X)$ and a bijection $\phi\in \Aut(\X)$, the push forward measure of $P$ through $\phi$ is defined as $\phi_{\#}P(A) \eqdef P(\phi^{-1}(A))$ for all $A$ in $\mathds{B}(\X).$ Since this makes the group $\Aut(\X)$ acting on the set of probability measures on $\X$, we also use the notation $\phi\cdot P \eqdef \phi_{\#}P$, which is standard for a (left) group action.
For the same reason, we shall use the notations $\phi\cdot f \eqdef f\circ \phi^{-1}$ and $\phi\cdot W \eqdef W(\phi^{-1}(\cdot),\phi^{-1}(\cdot))$ whenever $f$ is a measurable function on $\X$ and $W$ is a bivariate measurable function on $\X \times \X$.

For $P\in\mathcal{P}(X)$, the space $\Linf_P(\X, \R^p)$ is the space of essentially bounded (equivalence classes of) maps from $\X$ to $\R^p$ endowed with the norm $\Vert f \Vert_{P,\infty} \eqdef \esssup_{x\in\X,\, P}\Vert f(x)\Vert_{\infty}$. 
When there is no ambiguity on $P$, the norm $\Vert \cdot\Vert_{P,\infty}$ is noted $\Vert \cdot \Vert_{\infty}$. 
The space $\CC(\X, \R^p)$ is made of the continuous functions from $\X$ to $\R^p$. 
Since $\X$ is compact, any continuous map is bounded thus essentially bounded, which makes $\CC(\X, \R^p)$ a subspace of $\Linf_P(\X, \R^p)$.

Sets are represented between braces $\{\cdot\}$, whereas multisets, that is, sets in which an element is allowed to appear twice or more, are represented by double braces $\left\lBrace\cdot\right\rBrace$.
We define the sampling operator in the following way. 
For any $f: \mathcal{E}_0\to \mathcal{E}_1$ and $X=(x_1,\dots,x_n)\in \mathcal{E}_0^n$, the sampling of $f$ at $X$, denoted as $\iota_Xf$, is defined as 
\begin{equation}\label{eq: sampling operator def}
	\iota_Xf \eqdef (f(x_i),\dots,f(x_n)) \in \mathcal{E}_1^n\,.
\end{equation}

\subsection{Graph-related Definitions}\label{sec: graph def}
In this subsection, we introduce the concepts of discrete graph, graph signal and graph isomorphism.

\paragraph{Graph.} 
A \emph{non oriented weighted graph} $G$ is defined by a triplet $(V,E,w)$, corresponding to a set of \emph{nodes} (or \emph{vertices}) $V$, a set of \emph{edges} $E$, and a \emph{weight function} $w: V^2 \to \R^{+}$.
The cardinality $|V|=n$ is the \emph{size} of the graph.
Formally, the set $V=\Set{v_1,\ldots,v_n}$ may contain arbitrary elements with an arbitrary numbering.
However, for simplicity of the presentation, whenever the exact nature of the vertices does not matter,\footnote{It will matter later when considering random graphs whose nodes are latent positions in a metric space.}
we assume that $V=\Set{1,\ldots, n}$.
The set of edges $E$ contains 2-element subsets of $V$.
The set of neighbors of a vertex $i$ in $G$ is referred to as $\mathcal{N}_G(i)$ or simply $\mathcal{N}(i)$ when the underlying graph is clear from the context. 
The weight function $w$ assigns a nonnegative number to each edge. 
It is represented by a symmetric function $w: V^2 \to \R^{+}$  and the abbreviation $w_{i,j}$ is used to denote the weight $w(i,j) = w(j,i)$ where $\Set{i,j}\in E$. 
In this paper, ``graph'' will always mean ``undirected and weighted graph''. 

\paragraph{Graph signal.}
A \emph{graph signal} is a map from the set of vertices $V$ of a graph to a $d$-dimensional vector space $\R^d$.
In other words, it is the data of $n$ vectors in $\R^d$, one assigned to each node.
Hence, we represent a graph signal by a $n$ times $d$ matrix $Z \in \R^{n\times d}$.

\paragraph{Graph isomorphism.} 
Two graphs $G_1 = (V,E_1,w_1)$ and $G_2 =  (V,E_2,w_2)$ are said to be isomorphic if there is a permutation $\sigma\in \Sn$ such that $E_2= \Set*{ \{i,j\}\given  \{\sigma^{-1}(i),\sigma^{-1}(j) \}\in E_1 }$ and $w_2(i,j) = w_1\left(\sigma^{-1}(i),\sigma^{-1}(j)\right)$. 
In this case, we note $G_2 = \sigma\cdot G_1$. Moreover, if $Z$ is a signal on $G_1$ and $\sigma\in \Sn$, then $\sigma\cdot Z$ is an isomorphic signal on the  graph $\sigma\cdot G_1$. 

\subsection{Random Graph Models}\label{sec: random graph def}
In this subsection, we define our random graph model of interest, as well as the concept of random graph model isomorphism.

\paragraph{Random graph model and random graphs.}
A \emph{random graph model} is a couple $(W,P)$ where $P$ is a Borel probability measure on the compact set $\X \subset \mathbb{R}^d$ and $W: \X \times \X \mapsto [0,1]$  is a symmetric measurable function called \emph{connectivity kernel}. 
A random graph model is used to generate random graphs as follows. 
Given a positive integer $n$, we first draw $n$ independent and identically distributed random variables from the distribution $P$, represented by ${X_1,\dots,X_n}$, which form the vertex set of the graph. 
The random graph is then fully connected and has weight function $W$:
$$ X_1,\dots,X_n \iid P\,, \qandq w_{i,j} = w_{j,i} = W(X_i,X_j)\,.$$
Often, $W$ is decreasing with the distance between the $X_i$, such that nodes with similar latent variables have stronger connections; classical examples include the Gaussian kernel $W(X,X') = e^{-\frac{\norm*{X-X'}^2}{2\sigma^2}}$ or the so-called $\eps$-graphs $W(X,X') = \ind{\norm*{X-X'}\leq \eps}$.
When convenient, we will use the short notation $X=(X_1,\dots,X_n)$ for the tuple of the vertices of a random graph. 
We call $\GG_n(W,P)$ the distribution from which random graphs with $n$ nodes are drawn.
We bring the reader's attention to the fact that in the above definition, a random graph is always fully connected \emph{but} edges may have a weight equal to zero (e.g., for $\eps$-graphs). 
Also note that, in the rest of the paper, we generally do not put any assumptions on the model $(W,P)$, except when using max-aggregation (Exemples~\ref{ex: 5}-\ref{ex: e} and Sec.~\ref{sec: cv for max}), where we assume that $W$ is Lipschitz-continuous  in addition to some conditions on the probability space $(\X,P)$ (see Definition~\ref{def: volume retaining}).

Another common model~\citep{keriven_convergence_2020, lei_consestency_2015} is to add a Bernoulli distribution on edges, similarly to SBM models, in order to model unweighted random graphs, potentially with prescribed expected sparsity. 
It is not done here for the sake of simplicity. 
In the literature, weighted random graphs without Bernoulli edges are routinely used to analyze machine learning algorithms \citep{vonluxburg_clust_2008, maskey_generalization_2022}, as they essentially model the underlying phenomena of interest in many cases.

\paragraph{Random graph model isomorphism.}  
Two probability measures $P_1$ and $P_2$ on $\X$ are said isomorphic if there is some $\phi$ in $\Aut(\X)$ such that $P_2=\phi_{\#}P_1$. 
Similarly, two random graph models $(W_1,P_1)$ and $(W_2,P_2)$ on $\X$ are said to be isomorphic if there is a $\phi$ in $\Aut(\X)$ such that $(W_2,P_2)=(\phi\cdot W_1,\phi\cdot P_1)$, in this case, we will note $(W_2,P_2) = \phi\cdot(W_1,P_1)$.

\section{Message-Passing Graph Neural Networks}\label{sec: MPGNN definition}

A multilayer Message-Passing Neural Network (MPGNN) iteratively propagates a signal over a graph. 
At each step, the current representation of every node's neighbors are gathered, transformed, and combined to update the node's representation. 
We synthesize these three operations into a single one for the sake of readability.
Broadly speaking, an MPGNN can be defined as a collection of $L$ applications $(F^{(l)})_{1\le l\le L}$ that act as follows. 
Let $G$ be a graph with $n$ nodes, and $Z=Z^{(0)} \in \R^{n\times d_0}$ be a signal on it. At each layer, denoting $Z^{(l)}$ as the current state of the signal, $Z^{(l+1)}$ is computed node-wise by:  
\begin{equation}\label{eq: general message-passing}
    z^{(l+1)}_{i} \eqdef F^{(l+1)}\left(z_i^{(l)}, \multiSet*{\left(z_{j}^{(l)}, w_{i,j}\right) \given j\in \mathcal{N}(i) }\right)  \in \R^{d_{l+1}}\,. 
\end{equation}

So $Z^{(l+1)}$ is a matrix in $\R^{n \times d_{l+1}}$. 
In~\cref{eq: general message-passing}, $F^{(l)}$ takes as arguments a vector, which is the current node's representation, and a multiset of \emph{pairs}. 
Each pair is composed of a node from the neighborhood of the running node, along with the corresponding edge weight. 
In the literature, the $F^{(l)}$ are often referred to as the \emph{aggregations}~\citep{jegelka_theory_2022}. 
Their main property is to ignore the order in which the neighborhood information is collected, through the use of a multiset.

Depending on the context, the final output of the MPGNN may be a signal over the graph, or a single vector representation for the entire graph. 
Following the literature, we call these two versions respectively the \emph{equivariant} and the \emph{invariant} versions of the network. 
We denote $\Theta_G(Z)$ as the output in the first case and $\overline{\Theta}_G(Z)$ in the second case, where $\overline{\Theta}_G$ use an additional pooling operation over the nodes, $R: \R^{n\times d_L} \to \R^{d_L}$, called the \emph{readout}~\citep{jegelka_theory_2022} function: 

\begin{equation}\label{eq: mpgnn output}
    \Theta_G(Z)\eqdef Z^{(L)} \in\R^{n\times d_L} ,\qandq \overline{\Theta}_G(Z)\eqdef  R\left(\multiSet*{z_1^{(L)},\dots, z_n^{(L)}} \right)\in\R^{d_L} \,.
\end{equation}

We know that a fundamental property of GNNs is that they are consistent with graph isomorphism. 
More precisely, relabeling the nodes of the input graph signal must be the same as relabeling the nodes of the output in the \emph{equivariant} case, and must leave the output unchanged in the \emph{invariant} case. 
This is stated the proposition below.
\begin{proposition}[Invariance and equivariance of MPGNNs]\label{prop: MPGNN equiv prop}
    With the definition of the message-passing from~\cref{eq: general message-passing,eq: mpgnn output}, $\Theta$ and $\overline{\Theta}$ are respectively $\Sn$-equivariant and $\Sn$-invariant, in the sense that for all $\sigma \in \Sn$, for all $Z\in\R^{n\times d_0}$, we have $\Theta_{\sigma\cdot G}(\sigma\cdot Z) = \sigma\cdot \Theta_G(Z)$ and $\overline{\Theta}_{\sigma\cdot G}(\sigma\cdot Z) = \overline{\Theta}_G(Z)$.
\end{proposition}

\begin{proof}
	We prove the equivariant case. 
    Let us introduce the layer functions $\Lambda_G^{(l)}: Z^{(l-1)} \mapsto Z^{(l)}$, such that $\Theta_G = \Lambda_G^{(L)}\circ\dots\circ\Lambda_G^{(1)}$ by construction. 
    Let $Z\in \R^{n\times d_{l-1}}$ be a signal on $G$. 

    On the one hand,  $\Lambda_{\sigma\cdot G}^{(l)}(\sigma\cdot Z)=Y$ is the signal on $\sigma\cdot G$, obtained from the message-passing of $\sigma\cdot Z$ with regard to the graph $\sigma\cdot G$ and the weight function $\sigma\cdot w$.
    Locally, for all $i$, $y_i$ is the result of the message-passing at the node labeled $i$ in $\sigma\cdot G$.
    Recalling that the $i$th row of $\sigma\cdot Z$ is $z_{\sigma^{-1}(i)}$, and that the $(i,j)$-entry of $\sigma\cdot w$ is $w_{\sigma^{-1}(i), \sigma^{-1}(j)}$, we get that,
    $$y_i = F^{(l)}\Bigl(z_{\sigma^{-1}(i)}, \multiSet*{\left(z_{\sigma^{-1}(j)}, w_{\sigma^{-1}(i),\sigma^{-1}(j)}\right) \given j\in \mathcal{N}_{\sigma\cdot G}(i) } \Bigr)\,.$$
    Then, we do the change of variable $j = \sigma(j')$ in the multiset indexing, we obtain
	$$y_i = F^{(l)}\Bigl(z_{\sigma^{-1}(i)}, \multiSet*{\left(z_{j}, w_{\sigma^{-1}(i),j}\right) \given \sigma(j)\in \mathcal{N}_{\sigma\cdot G}(i) } \Bigr)\,.$$
    Now, by definition of $\sigma\cdot G$, we have the equivalence $\sigma(j) \in \mathcal{N}_{\sigma\cdot G}(i)$ if and only if $j \in \mathcal{N}_{G}(\sigma^{-1}(i))$.
    This implies that the multiset 
    $\multiSet*{\left(z_{j}, w_{\sigma^{-1}(i),j}\right) \given \sigma(j) \in \mathcal{N}_{\sigma\cdot G}(i) }$ 
    is the same as the multiset 
    $\multiSet*{\left(z_{j}, w_{\sigma^{-1}(i),j}\right) \given j\in \mathcal{N}_{G}\left(\sigma^{-1}(i)\right) }$. 
    Thus,

	$$y_i=F^{(l)}\Bigl(z_{\sigma^{-1}(i)}, \multiSet*{\left(z_{j}, w_{\sigma^{-1}(i),j}\right) \given j\in \mathcal{N}_{G}\left(\sigma^{-1}(i)\right) } \Bigr)\,.$$

	On the other hand, $\sigma\cdot\Lambda_G^{(l)}(Z)=Y'$ is the signal $\Lambda_G^{(l)}(Z)$ to which the rows have been permuted a posteriori, so

	$$y'_i=F^{(l)}\Bigl(z_{\sigma^{-1}(i)}, \multiSet*{\left(z_{j}, w_{\sigma^{-1}(i),j}\right) \given j \in \mathcal{N}_{G}(\sigma^{-1}(i)) } \Bigr)\,.$$

    Hence $Y=Y'$, which means that $\Lambda_{\sigma\cdot G}^{(l)}(\sigma\cdot Z)= \sigma\cdot\Lambda_G^{(l)}(Z)$, and that $\Lambda^{(l)}$ is equivariant for all $l$. 
    Thereby $\Theta_{\sigma\cdot G}(\sigma\cdot Z) = \sigma\cdot \Theta_G(Z)$  by composition.
	For the invariant case, $R$ is clearly $\Sn$-invariant since it has a multiset as input. 
    The fact that the composition of an equivariant map followed by an invariant map is invariant yields the result.
\end{proof}

The role of the functions $F^{(l)}$ in~\cref{eq: general message-passing} is crucial and there is a wide range of designs for them~\citep{wu_comprehensive_2021}. 
Nevertheless, they mostly take the following ``message-then-combine'' form. 
At layer $l+1$, the signals of the neighbors of a node are transformed by a learnable message operation $\m^{(l+1)}$. 
Then these \emph{messages} $\m^{(l+1)}(z_j^{(l)})$ are \emph{aggregated} along with some optional \emph{weight coefficients}, whose expressions are very general here,
\begin{equation}\label{eq: coefficients}
	c_{i,j}^{(l+1)} = c^{(l+1)}\left(z_{i}^{(l)},z_{j}^{(l)},w_{i,j}\right)\,,
\end{equation}
in a way that is invariant to node relabeling. 
It appears that a natural way of doing the aggregation step is to perform a \emph{mean}, in a broad sense: an arithmetic mean, a weighted mean, a maximum, and so on~\citep{bullen_handbook_mean_2013}. 
Thus, we have a \emph{mean} operator $M^{(l+1)}$ such that~\cref{eq: general message-passing} is expressed as
\begin{equation}\label{eq: aggregation+message message-passing}
		 F^{(l+1)}\left(z_i^{(l)}, \multiSet*{\left(z_{j}^{(l)}, w_{i,j}\right) \given j\in \mathcal{N}(i)}\right) 
		= M^{(l+1)}\left(\multiSet*{\left(\m^{(l+1)}(z_{j}^{(l)}), c_{i,j}^{(l+1)}\right) \given j\in \mathcal{N}(i)}\right)
        \,.
\end{equation}

To our knowledge,~\cref{eq: aggregation+message message-passing} encompasses most of the existing popular MPGNN architectures of the literature. 
We note that it is essentially a more verbose reformulation of~\cref{eq: general message-passing}, the two different expressions mostly provide a different level of intuition on the message-passing process.
In the sequel, we discuss five examples that follow~\cref{eq: aggregation+message message-passing}. 
For each example, we also give the corresponding readout function that will be used in our results, for the invariant case. 

\begin{example}[Convolutional Message-Passing]\label[example]{ex: 1} 
	The $c_{i,j}$ are the graph weights $w_{i,j}$~\citep{kipf_semi-supervised_2017,defferrard_convolutional_2016, gilmer_neural_2017}. 
    Each neighbor representation is multiplied by its corresponding weight and we combine them with an arithmetic mean. 
    Notice that this is equivalent to a SGNN with polynomial filters of degree one.
	$$z^{(l+1)}_{i} = \frac{1}{\vert \mathcal{N}(v_i)\vert}  \sum_{v_j\in\mathcal{N}(v_i)} w_{i,j}\m^{(l+1)}\left(z_{j}^{(l)}\right).$$
	In the invariant case, the readout function is an arithmetic mean:
	$$ R\left(\left\lBrace z_1^{(L)},\dots, z_n^{(L)} \right\rBrace \right) = \frac{1}{n}\sum_{i=1}^n z_i^{(L)}.$$
\end{example}

\begin{example}[Degree normalized convolution]\label[example]{ex: 2} 
    The $c_{i,j}$ are still the graph weights $w_{i,j}$ but a weighted mean is performed~\citep{maskey_generalization_2022}.
	$$z^{(l+1)}_{i} = \sum_{j\in\mathcal{N}(v_i)} \frac{w_{i,j}}{\sum_{k\in \mathcal{N}(v_i)} w_{i,k}}\m^{(l+1)}\left(z_{j}^{(l)}\right).$$
	In the invariant case, the readout function is again an arithmetic mean.
\end{example}

\begin{example}[Attention based Message-Passing]\label[example]{ex: 3} 
	Unlike the two examples above, the \emph{attention} coefficients are learnable and depend on all the possible parameters mentioned in~\cref{eq: coefficients}~\citep{velickovic_graph_2018}. 
    A weighted mean is then used.
	$$z^{(l+1)}_{i} = \sum_{j\in\mathcal{N}(v_i)} \frac{c^{(l+1)}\left(z_{i}^{(l)},z_{j}^{(l)},w_{i,j}\right)}{\sum_{k\in \mathcal{N}(v_i)} c^{(l+1)}\left(z_{i}^{(l)},z_{k}^{(l)},w_{i,k}\right)}\m^{(l+1)}\left(z_{j}^{(l)}\right).$$
	In the invariant case, the readout function is again an arithmetic mean.
\end{example}

\begin{example}[Generalized mean]\label[example]{ex: 4} 
	Similar to~\cref{ex: 1}, but with additional functions used to compute ``generalized'' or semi-arithmetic means~\citep{kortvelesy_generalized_2023, de_carvalho_mean_2016}:
	$$z^{(l+1)}_{i} = h\left( \frac{1}{\vert \mathcal{N}(v_i)\vert}  \sum_{v_j\in\mathcal{N}(v_i)} h^{-1}\left(w_{i,j} \m^{(l+1)}\left(z_{j}^{(l)}\right)\right)\right) \, ,$$
	for some invertible function $h$ (possibly defined on a bounded domain). 
    In this case, following the popular formalism from~\cite{kolmogorov_moyenne_1930}, in order to legitimately be considered as a generalized mean, some regularity conditions are required on $h$, typically, $h$ must be a strictly increasing continuous function.
    For instance, taking $h^{-1}=x \mapsto \ln(x)$ (assuming positivity of the inputs) and $h=x \mapsto e^x$ yields the geometric mean, while taking $h^{-1}=x \mapsto x^p$ and $h=x \mapsto x^{\nicefrac{1}{p}}$ (again, assuming nonnegativity of the message if necessary) yields a power mean as employed in the moment-based aggregations from~\cite{corso_principal_neighbourhood_agg_2020} (up to centering, which we omit for simplicity).

	In the invariant case, the readout function is again an arithmetic mean (for simplicity).
\end{example}

\begin{example}[Max Convolutional Message-Passing]\label[example]{ex: 5} 
    The aggregation maximum is often mentioned as a possibility in the literature~\citep{hamilton_inductive_2017}, but we note that it is less common in practice. 
    Here the $c_{i,j}$ are also the graph weights $w_{i,j}$ but a \emph{coordinate-wise} maximum is used to combine the messages: 
	$$z^{(l+1)}_{i} = \max_{v_j\in\mathcal{N}(v_i)} w_{i,j}\m^{(l+1)}\left(z_{j}^{(l)}\right)\,.$$
	In the invariant case, the readout function is a coordinate-wise maximum.
	$$ R\left(\left\lBrace z_1^{(L)},\dots, z_n^{(L)} \right\rBrace \right) = \max_{i=1,\dots,n} z_i^{(L)}\,.$$
    Notice that, assuming positivity of the inputs, this aggregation is also obtained as the limit, for $p\to \infty$ of the generalized power mean from~\cref{ex: 4} with $h(x)=x^{\nicefrac{1}{p}}$.
\end{example}

\section{Continuous Message-Passing GNNs on Random Graph Models}\label{sec: cMPGNN}
We define the continuous counterpart of MPGNNs, which we call continuous MPGNNs. 
A cMPGNN is defined to be $L$ operators $ (\FF^{(l)})_{1\le l \le L}$ that propagate a \emph{function} on $\X$ relatively to a random graph model. 
Let $(W,P)$ be a random graph model and $f=f^{(0)}\in \Linf_P(\X,\R^{d_0})$, $f^{(l+1)}$ is recursively computed by:
\begin{equation}\label{eq: cmpgnn}
	f^{(l+1)}(x) \eqdef \FF_P^{(l+1)}\left(f^{(l)}(x),\left(f^{(l)},W(x,\cdot)\right)\right) \in\R^{d_{l+1}}\,,
\end{equation}
for all $x\in \X$.
Notice that $\FF^{(l+1)}$ depends on the measure $P$. 
Considering the functions $f^{(l)}$ as signals on the vertex set $\X$, the update $f^{(l+1)}(x)$ of a node $x\in\X$ is calculated from the knowledge of its current representation $f^{(l)}(x)$ and all its ``weighted neighborhood'' $\left(f^{(l)}, W(x,\cdot)\right)$. 
The latter is a short notation for the map $y\mapsto \left(f^{(l)}(y), W(x,y)\right)$ at $x$ fixed, which is the continuum equivalent of the multiset of pairs of weighted neighbors ${\multiSet*{(z_{j}^{(l)}, w_{i,j}) \given j \in \mathcal{N}(i)}}$ from~\cref{eq: general message-passing}.

For conciseness, we will often overload~\cref{eq: cmpgnn} with the short notation:
\begin{equation}\label{eq: cmpgnn overloaded notation}
	f^{(l+1)} \eqdef \FF^{(l+1)}_P\left(f^{(l)}, W\right)\,.   
\end{equation}

We denote $\Theta_{W,P}(f)$ the output in the equivariant case and $\overline{\Theta}_{W,P}(f)$ in the invariant case:
\begin{equation}\label{eq: continuum readout}
	\Theta_{W,P}(f) \eqdef f^{(L)} \in \Linf_P(\X,\R^{d_L}), \qandq \overline{\Theta}_{W,P}(f) \eqdef \RRR_P(\Theta_{W,P}(f)) \in \R^{d_L}\,,
\end{equation}
where $\overline{\Theta}_{W,P}$ involves an additional continuum readout operator 
$$\RRR: \mathcal{P}(\X) \times \Linf_P(\X,\R^{d_L}) \to  \R^{d_L} \,.$$
Naturally, we also demand the equivariant and invariant versions of the cMPGNN to respectively be equivariant and invariant to random graph model isomorphism. To that extent, we impose the following assumption on the operators  $\FF^{(l)}$ and on $\RRR$.
\begin{assumption}\label{ass: change of variable on omega}
	There is a subgroup $H\subset\Aut(\X)$ such that $\forall\;1\le l\le L$, $\forall f \in \Linf_P(\X,\R^{d_l})$, $\forall \phi\in H$,
	$$
	\FF^{(l)}_{\phi \cdot P}(\phi \cdot f, \phi \cdot W) = \phi \cdot \FF^{(l)}_P(f,W)\,,
	$$
	and,
	$$  \RRR_{\phi\cdot P}(\phi\cdot f)) =\RRR_P(f)\,.$$
\end{assumption}
\Cref{ass: change of variable on omega} is largely inspired by the classical change of variable formula by push forward measure in Lebesgue integration. 
This formula states that for any $\phi\in\Aut(\X)$ and any measurable map $f$,
\begin{equation}\label{eq: classical measure cdv}
	\int f  \d P = \int \phi\cdot f \d(\phi \cdot P)\, .
\end{equation}
It is not difficult to verify that if, for example, $\FF_{P}(f,W) = \int f(y) W(x,y)  \d P(y)$ and ${\RRR_P(f)=\int f  \d P }$, then~\cref{eq: classical measure cdv} implies~\cref{ass: change of variable on omega} with $H = \Aut(\X)$.

Contrary to the discrete case, where the symmetry is valid for the full group $\Sn$, we require here a symmetry for a subgroup of $\Aut(\X)$ only. 
Ideally, one would like~\cref{ass: change of variable on omega} to hold for $H=\Aut(\X)$. 
However, in the next section, we will interpret some cMPGNN as limits of discrete MPGNN, such that the graph isomorphism symmetry becomes a random graph model isomorphism symmetry, as the number of nodes tends to infinity. 
In this context, the example of maximum aggregation (\cref{ex: 5,,ex: e} in the next section) will highlight the fact that, for a matter of \emph{existence} of such a limit, one may have to impose some conditions on $P$, and thus restrict to a subgroup of $\Aut(\X)$.

\begin{proposition}[Invariance and equivariance of cMPGNNs]\label{prop: equiv of cmpgnn}
	Let $(W,P)$ be a random graph model on $\X$. 
    Then, under~\cref{ass: change of variable on omega}, $\Theta$ and $\overline{\Theta}$ are respectively $H$-equivariant and $H$-invariant. 
    In other words, for any $f$, for any $\phi\in H$, 
    $$\Theta_{\phi\cdot (W, P)}(\phi\cdot f)=\phi\cdot \Theta_{W,P}(f)\,,$$
    and,
    $$\overline{\Theta}_{\phi\cdot (W, P)}(\phi\cdot f)= \overline{\Theta}_{W,P}(f)\,.$$
\end{proposition}

\begin{proof}
	We start with the equivariant case, the invariant one will follow immediately by composition with $\RRR$. 
Let the $\mathcal{L}_{W,P}^{(l)}\colon \R^{n\times d_{l-1}} \to \R^{n\times d_l}$ be the layer operators such that $\Theta_{W,P}^{(L)} = \mathcal{L}_{W,P}^{(L)}\circ \dots\circ \mathcal{L}_{W,P}^{(1)}$. 
    Let $f\in \Linf_P(\X,\R^{d_{l-1}})$, using~\cref{ass: change of variable on omega} we obtain
	$$
	\phi\cdot \mathcal{L}^{(l)}_{W,P}(f) = \phi \cdot \FF_P^{(l)}\left(f,W \right) = \FF_{\phi\cdot P}^{(l)}\left(\phi\cdot f,\phi\cdot W\right) =  \mathcal{L}^{(L)}_{\phi\cdot (W, P)}(\phi\cdot f)\,.
	$$
	So the Proposition is true on all the $\mathcal{L}_{W,P}^{(l)}$, thus also true on $\Theta_{W,P}^{(L)}$ by composition.
	For the invariant case, it is clear from~\cref{ass: change of variable on omega} that $\RRR$ is $H$-invariant. 
    The fact that the composition of an equivariant map followed by an invariant map is invariant yields the result.
\end{proof}

In the following, are some examples of cMPGNN. 
The reader will, of course, see the intuitive connection to the previous~\cref{ex: 1,ex: 2,ex: 3,ex: 4,ex: 5}
In the next section, we will precisely see in what sense~\cref{ex: 1,ex: 2,ex: 3,ex: 4,ex: 5}, when applied on random graphs and as $n$ grows large, converge to the following cMPGNNs. 

\begin{myexample}[Convolutional Message-Passing]\label[example]{ex: a} 
    The arithmetic mean becomes an integral over the probability space
	$$f^{(l+1)}(x) = \int_{y\in \X} W(x,y)\m^{(l+1)}\left(f^{(l)}(y)\right) \d P(y)$$
	and, in the invariant case, the continuous readout is
	$$ \RRR_P\left(f^{(L)}\right) = \int_{\X} f^{(L)}  \d P\,.$$
\end{myexample}

\begin{myexample}[Degree Normalized Convolutional Message-Passing]\label[example]{ex: b} 
    The continuous counterpart is
	$$f^{(l+1)}(x) = \int_{y\in \X} \frac{W(x,y)}{\int_{t\in \X} W(x,t) \d P(t)}\m^{(l+1)}\left(f^{(l)}(y)\right) \d P(y)\,.$$
	In the invariant case, the readout is again the integral relatively to $P$.
\end{myexample}

\begin{myexample}[Attention based Message-Passing]\label[example]{ex: c} 
    The continuous counterpart is
	$$f^{(l+1)}(x) = \int_{y\in \X} \frac{c^{(l+1)}\left(f^{(l)}(x), f^{(l)}(y),W(x,y)\right)}{\int_{t\in \X} c^{(l+1)}\left(f^{(l)}(x), f^{(l)}(t),W(x,t)\right) \d P(t)}\m^{(l+1)}\left(f^{(l)}(y)\right) \d P(y)\,.$$
	In the invariant case, the readout is again the integral relatively to $P$.
\end{myexample}

\begin{myexample}[Generalized mean]\label[example]{ex: d} 
    We simply add the function $h$ to the mean example
    $$f^{(l+1)}(x) = h\left(\int_{y\in \X} h^{-1}\left(W(x,y) \m^{(l+1)}\left(f^{(l)}(y)\right)\right) \d P(y)\right)\,.$$
	In the invariant case, the readout is again the integral relatively to $P$.
\end{myexample}

\begin{myexample}[Max Convolutional Message-Passing]\label[example]{ex: e} 
    The maximum becomes a coordinate-wise essential supremum according to the probability measure $P$ (that is, a supremum $P$-almost everywhere)
	$$f^{(l+1)}(x) = \uesssup{y\in \X,\, P}W(x,y)\m^{(l+1)}\left(f^{(l)}(y)\right)\,,$$
	and, in the invariant case, the final readout is the coordinate-wise: 
	$$\RRR_P\left(f^{(L)}\right) = \uesssup{y\in \X,\, P}f^{(L)}(y) \,.$$
	Remark that we use the $P$-essential supremum and not the classical supremum here, since this will represent the limit maximum of variables that are $P$-distributed.
    Nevertheless, if the measure $P$ is strictly positive, notice that the essential supremum of a continuous function is nothing more than the usual supremum, see~\cref{lem: esssup = sup} in~\cref{app: Useful results}
\end{myexample}

\begin{remark}
	It can be easily verified that for all these examples, the underlying $\FF^{(l)}$ functions satisfy~\cref{ass: change of variable on omega} with $H=\Aut(\X)$. 
    For the integral, it is ensured by the classical change of variable formula~\cref{eq: classical measure cdv}. 

	As for the essential supremum, a similar formula holds. 
    Indeed, recall that for any measurable $g$, for any measurable bijection $\phi$, one has
	$$ \uesssup{P} g\circ \phi = \inf \Set*{M \given P\left(g\circ \phi > M\right) = 0} \,,$$
    where $\left(g\circ \phi > M\right)$ is the probabilistic notation for the set $\Set*{x \given g\circ \phi(x) > M }$.
	Hence, since by the bijectivity of $\phi$, 
    $$\Set*{x \given g\circ \phi(x) > M } = \Set*{ \phi^{-1}(y) \given g(y) > M} \,,$$ 
    one has $(g\circ \phi > M) = \phi^{-1}(g > M)$, which finally yields
	$$\inf\Set{M \given P(g\circ \phi > M)=0} = \inf\Set{M \given P(\phi^{-1}(g > M))=0} = \uesssup{\phi_{\#}P} g \,.$$
\end{remark}

\section{Continuous Message-Passing GNNs as Limits of Discrete Message-Passing GNNs on Large Random Graphs} \label{sec: Cv}
This section contains the core of our contributions. 
We focus on MPGNNs when applied on random graphs $G_n$ drawn from $\GG_n(W,P)$. 
Specifically, given such an MPGNN, we are interested in its limit as $n$ tends to infinity. 
We show that under mild regularity conditions, such a limit exists and is a cMPGNN. Furthermore, we provide some non-asymptotic bounds to control the deviation between an MPGNN and its limit cMPGNN with high probability. 

This section is divided in two parts. 
In the first part (\cref{sec: limit of mpgnn on large random graph}), given an MPGNN, we define, \emph{when it exists}, its associated canonical cMPGNN on $(W,P)$ that we call \textbf{continuous counterpart}. 
The precise definition of this central concept is~\cref{def: construction of continuous limit}: it states how that continuous counterpart is built out of the discrete network as a limit on random graphs $G_n\sim \GG_n(W,P)$ of growing sizes. 
Then, we show that under mild regularity conditions,~\cref{ex: a,ex: b,ex: c,ex: d} are indeed the continuous counterparts of~\cref{ex: 1,ex: 2,ex: 3,ex: 4} according to our definition.

In the second part of this section (\cref{sec: cv of mpgnn}), we study the convergence of MPGNNs towards their continuous counterpart, when it exists: we give sufficient conditions for this convergence to occur and provide convergence rates in the form of non-asymptotic bounds with high probability. 
Our first main result,~\cref{th: main result} in~\cref{sec: cv by mcdiarmid}, concerns a class of MPGNNs that have a certain kind of Lipschitz continuity among other mild assumptions: in a few words, it states that such MPGNNs have a continuous counterpart to which they converge as $n$ grows, with a controlled rate that we specify. 
As we will see, this applies to all examples \emph{but} max aggregation.
Our second result,~\cref{th: main result max} in~\cref{sec: cv for max}, is specific to the case of maximum aggregation, as in this case the bounded difference property is not verified and~\cref{th: main result} is not applicable. 
It is based on another concentration inequality and leads to a convergence rate with a dependence on the input dimension $d$ (recall $\X\subset \R^d$), contrary to the bounded differences method. 

\subsection{Definition of the Limit Message-Passing GNN}\label{sec: limit of mpgnn on large random graph}
Let $(W,P)$ be a random graph model and $f\in \Linf_{ P} (\X,\R^d)$. 
The main purpose of this subsection is to \emph{define} the natural limit of discrete MPGNNs to cMPGNNs, before examining more precisely the rate of convergence from one to the other in the next subsection.

To this end, we consider a single-layer MPGNN applied on a random graph $G_n\sim\GG_n(W,P)$ and input node features $Z = \iota_Xf$ as a sampling of some function (recall the definition of the sampling operator in~\cref{eq: sampling operator def}). 
We define a corresponding cMPGNN layer on $(W,P)$ with input map $f$. 
Since there is only one layer in this section, we drop the superscript indexation.
To motivate the next definition -- that may seem overly technical at first sight --  let us consider the simplest example, namely~\cref{ex: 1,ex: a}. 
Let us examine how~\cref{ex: a} can be recovered from~\cref{ex: 1} at the limit. 
The update of $f(x)$ is given by 
\begin{equation}\label{eq: motivation for definition of the limit 1}
	\FF_{P}(f,W)(x) = \int_{y\in \X} W(x,y)\m\left(f(y)\right) \d P(y)\,.
\end{equation}
It is fairly clear that, by the law of large numbers, this integral equals the limit of 
\begin{equation}\label{eq: motivation for definition of the limit 2}
	\frac{1}{n}\sum_{i=1}^{n} W(x,X_i)\m(f(X_i))\,,
\end{equation}
for $X_1,\dots,X_n \iid P$. 
Coincidentally,~\cref{eq: motivation for definition of the limit 2} is exactly the discrete message-passing of~\cref{ex: 1} around node $x$ of a random graph with node latent variables $\{x, X_1, \ldots, X_n\}$, signal $Z = \{f(x), f(X_1),\ldots, f(X_n)\}$ and kernel $W$.
We have thus naturally obtained the cMPGNN from~\cref{ex: a} via a limit of the MPGNN of~\cref{ex: 1} on random graphs.

Back to the general case, given an abstract discrete (single-layer) MPGNN with aggregation $F$, we want to define a cMPGNN from its limit on random graphs. 
Following the path of the above example, we could look at the almost sure limit of the message-passing equation 
$$
\lim_{n\to\infty} F \left(f(x), \multiSet*{\left(f(X_k),W(x,X_k)\right) }_{ 1\leq k\leq n}\right) \,.
$$
However, the existence of this limit is far from obvious in the general case. %
As we will see in the next definition, we will rather relax it to the convergence of the expectation instead:
\begin{equation}\label{eq: motivation for definition of the limit 4}
\E_{X_1,\dots,X_n}\left[F \left(f(x), \multiSet*{\left(f(X_i),W(x,X_i)\right) }_{ 1\leq i\leq n}\right)  \right]\,.
\end{equation}
Up to some details related to isomorphism invariance, the existence of this limit is how we \emph{define} the continuous counterpart of an MPGNN.

\begin{definition}[Continuous counterpart]\label{def: construction of continuous limit}
	Let $F$ be an MPGNN layer. 
    For any $f\in \Linf_P(\X,\R^{ d})$, $W:\X^2\to [0,1]$ and $P\in\mathcal{P}(\X)$, define the sequence of functions in  $\Linf_P(\X,\R^{ d'})$ by
	\begin{equation}\label{eq: def construction limite part 1}
    F_{P,n}(f,W): x\mapsto \E\left[F \left(f(x), \multiSet*{\left(g(X_i),W(x,X_i)\right) }_{ 2\leq i\leq n}\right)  \right] \,.
	\end{equation}
	where the expected value is taken over all the $X_2,\dots,X_n\iid P.$ 
	
	Let $\FF$ be an operator of the form~\cref{eq: cmpgnn overloaded notation} taking value in $\Linf_P(\X,\R^{ d'})$, and suppose that there exists $H$, a non-trivial subgroup of $\Aut(\X)$, such that for any $f\in \Linf_P (\X,\R^d)$ and any $\phi\in H$, the operator $\FF_{\phi \cdot P} \left( \phi\cdot f, \phi\cdot W \right)$ arises as the limit
	\begin{equation}\label{eq: def construction limite part 2}
        \FF_{\phi \cdot P} \left( \phi\cdot f, \phi\cdot W \right) = \lim_{n\to \infty} F_{\phi \cdot P, n} \left( \phi\cdot f, \phi\cdot W \right) \,,
	\end{equation}
	where the convergence occurs with respect to the norm of the space $\Linf_{ \phi \cdot P} (\X,\R^d)$.
	Then we say that $\FF$ is the \textbf{continuous counterpart} of $F$ for $H$. 
    When $H=\rm{Aut}(\X)$, or when $H$ is obvious from the context, we simply say that $\FF$ is the \textbf{continuous counterpart} of $F$.
\end{definition}

Note that we consider only $n-1$ random variables for convenience with later definitions (in an $n$-sized graph, each node has potentially $n-1$ neighbors). 
In this definition of limit GNNs, we have \emph{not} explicitly assumed that the $\FF$ defined above satisfies~\cref{ass: change of variable on omega} related to continuous isomorphisms and cMPGNN.
Fortunately, this assumption is in fact automatically verified, for the natural subgroup $H$ involved in the definition. 
In other words, every continuous counterpart of MPGNN as defined above is indeed a valid cMPGNN.

\begin{proposition}
	Let $\FF$ be the continuous counterpart of $F$ for $H$ as defined in~\cref{def: construction of continuous limit}. 
    Then it satisfies~\cref{ass: change of variable on omega} for any $\phi\in H$.
\end{proposition}

\begin{proof}
	Let $f \in \Linf_P (\X,\R^d)$, $\phi\in H$, and $X_2,\dots,X_n \iid P$, we have for $P$-almost all $x$,
	\begin{align} \label{eq: proof equivariance of continuous counterpart 1}
		 \phi\cdot F_{ P,n}( f, W) (x) 
        & = \E_{X_2,\dots,X_n}\left[ F \left(f(\phi^{-1}(x)), \multiSet*{\left( f(X_i), W(\phi^{-1}(x),X_i)\right) }_{ 2\leq i\leq n}\right)  \right] \notag\\
        & = \E_{Y_2,\dots,Y_n}\left[ F \left(\phi\cdot f(x), \multiSet*{\left(\phi\cdot f(Y_i),\phi\cdot W(x,Y_i)\right) }_{ 2\leq i\leq n}\right)  \right] \\
		& = F_{\phi\cdot P, n }\left(\phi\cdot f,\phi\cdot W\right)(x)\,. \notag
	\end{align}
	Where in~\cref{eq: proof equivariance of continuous counterpart 1}, we have set $Y_i = \phi(X_i)$ for all $i$, which are identically distributed random variable with common law $\phi\cdot P$, and used the classical change of variable formula~\cref{eq: classical measure cdv}. Thus, by taking the limit,~\cref{eq: def construction limite part 2} implies
	\begin{equation*}
		\FF_{\phi\cdot P}(\phi\cdot f, \phi\cdot W) = \phi\cdot \FF_P(f,W)\,,
	\end{equation*}
    which is the desired result.
\end{proof}

The same definitions and propositions as above can also be given for a readout final layer.

\begin{definition}\label{def: construction of the limit readout}
	Let $R$ be an MPGNN readout layer and $P\in\mathcal{P}(\X)$. 
    For $f\in \Linf_P(\X,\R^{ d})$, we define the sequence of functions
    $$R_{P,n}(f)= \E_{X_1,\dots,X_n}\left[R\left( \multiSet*{ f(X_1),\dots, f(X_n)} \right)\right] \in\R^{d'}\,,$$
	where the expected value is taken over all the  $X_1,\dots,X_n\iid P.$ 

	Let $\RRR$ be a continuum readout operator of the form~\cref{eq: continuum readout} taking values in $\R^{d'}$.
	Suppose we have $H$ a non-trivial subgroup of $\Aut(\X)$ such that  for any $f\in \Linf_P (\X,\R^d)$,  for any $\phi\in H$,   $R_{\phi\cdot P,n}(\phi\cdot f)$ converges to $ \RRR_{\phi\cdot P}(\phi\cdot f)$ in the $\Vert\cdot\Vert_{\infty}$ norm of $\R^d$
	$$R_{\phi\cdot P,n}(\phi\cdot f)\rightarrow \RRR_{\phi\cdot P}(\phi\cdot f)\,.$$
	Then we say that $\RRR$ is the \textbf{continuous counterpart} of $R$ for $H$, unless $H=\rm{Aut}(\X)$ or $H$ is obvious from context, in which case we simply say that $\RRR$ is the continuous counterpart of $R$.
\end{definition}

Naturally,~\cref{def: construction of the limit readout} also implies invariance with respect to the proper subgroup.

\begin{proposition}
	Let $\RRR$ be the continuous counterpart of $R$ as in~\cref{def: construction of the limit readout}. 
    Then it satisfies~\cref{ass: change of variable on omega} for any $\phi\in H$.
\end{proposition}

Going back to our four examples of~\cref{sec: MPGNN definition,sec: cMPGNN}, we now show that~\cref{ex: a,ex: b,ex: c,ex: d} are the continuous counterparts of~\cref{ex: 1,ex: 2,ex: 3,ex: 4} for the full $\Aut(\X)$.
We will assume some mild positivity conditions for the coefficients in the degree normalized and the GAT examples. 
These are satisfied for most instances of kernels (\eg Gaussian) or attention coefficients (\eg the typical exponentials of LeakyReLUs from~\cite{velickovic_graph_2018}), as the data live in a bounded domain $\X$. 
For the generalized mean of~\cref{ex: 4,ex: d}, we will rely on an Hölder type regularity assumption.

\Cref{ex: 5,ex: e} are however more involved, as one has to be careful with the shape of $\X$ and the properties of $P$ to avoid null set issues at the boundary of $\X$. 
We will show that if $\X$ contains no nonempty open null set, and if $W$ and $f$ are continuous, then~\cref{ex: e} is the continuous counterpart of~\cref{ex: 5} for the subgroup $H$ of $\Aut(\X)$ consisting of all the \emph{homeomorphisms} from $\X$ into itself. 

The proof of the next proposition is given in~\cref{sec: app examples proofs}.

\begin{proposition}[Continuous counterpart of Examples]\label{prop: continuous-examples}
    The following holds:
    \begin{itemize}
        \item \textbf{\Cref{ex: 1,ex: a}.} With no additional restriction on $W$, $f$, nor $P$,~\cref{ex: a} is the continuous counterpart of~\cref{ex: 1} for the full $\Aut(\X)$.
        \item \textbf{\Cref{ex: 2,ex: b}.} Suppose that $\m$ is bounded and that there is a strictly positive $a$ such that $0 < a \le  W$. 
            Then~\cref{ex: b} is the continuous counterpart of~\cref{ex: 2} for the full $\Aut(\X)$.
        \item \textbf{\Cref{ex: 3,ex: c}.} Suppose that $\m$ is bounded and that there is two positive constants $0<a<b$ such that $a\le  c(f(x),f(y),W(x,y)) \le  b$. 
            Then~\cref{ex: c} is the continuous counterpart of~\cref{ex: 3} for the full $\Aut(\X)$.
        \item \textbf{\Cref{ex: 4,ex: d}.} Suppose that $h^{-1}$ is bounded and that $h$ is $\alpha_h$-Hölder on the range of $h^{-1}$ with $0<\alpha_h\leq 1$, that is, $\|h(x) - h(y)\|_\infty \leq K_h\|x-y\|_\infty^{\alpha_h}$. 
            Then~\cref{ex: d} is the continuous counterpart of~\cref{ex: 4} for the full $\Aut(\X)$.
        \item \textbf{\Cref{ex: 5,ex: e}.} Suppose that $W$, $\m$, and $f$ are continuous and that the measure $P$ is strictly positive on $\X$, \ie any nonempty relative open of $\X$ has a strictly positive measure by $P$. 
            Then~\cref{ex: e} is the continuous counterpart of~\cref{ex: 5} for $\Hom(\X)$: the subgroup of $\Aut(\X)$ made of the $\phi\in\Aut(\X)$ that are homeomorphisms.
    \end{itemize}
\end{proposition}

\subsection{Convergence of Message-Passing GNNs to their Continuous Counterpart}\label{sec: cv of mpgnn}
Let $(W,P)$ be a random graph model, and $(G_n)_{n\geq 1}$ be a sequence of random graphs drawn from $\GG_n(W,P)$. 
We go back to the multi-layer setup: consider an MPGNN $(F^{(l)})_{1\le l \le L}$, a readout $R$ and assume that their continuous counterparts $(\FF^{(l)})_{1\le l \le L}$ and $\RRR$ in the sense of~\cref{def: construction of continuous limit,def: construction of the limit readout} exist. 
For an $f\in \Linf_P(\X,\R^{d_0})$, does the MPGNN on $G_n$ with input signal $\iota_Xf$ actually converge to the cMPGNN on $(W,P)$ with input signal $f$? 
If yes, at which speed? 
In this section, we provide non-asymptotic bounds with high probability to quantify this convergence.

Our main theorems state that, under mild regularity conditions and with high probability, $\Theta_{G_n}(\iota_X(f))$ is close to $\Theta_{W,P}(f)$ in the equivariant case and that $\overline{\Theta}_{G_n}(\iota_X(f))$ is close to $\overline{\Theta}_{W,P}(f)$ in the invariant case. 
For the latter, we can compare both outputs directly since they belong to the same vector space. 
The comparison is however more involved in the equivariant case since $\Theta_{G_n}(\iota_X(f))$ is a matrix of size $n\times d_L$ and $\Theta_{W,P}(f)$ is a function from $\X$ to $\R^{d_L}$. 
In this case, we choose to sample the function, which results in measuring the deviation with the Maximum Absolute Error (MAE) defined by 
\begin{equation}\label{eq: def of mae}
    \MAE_X(Z,f) = \max_{1\le i\le n} \norm*{z_i - f(X_i)}_{\infty}\,.
\end{equation}
Applied on an MPGNN and its cMPGNN counterpart, it can be written as
$$
\MAE_X\left( \Theta_{G_n}\left( \iota_X f \right) , \Theta_{W, P}(f) \right) = \norm*{ \left( \iota_X f \right)^{(L)}_i - \iota_X\left( f^{(L)} \right)_i}_\infty
\,,$$       
which corresponds to the following process.
Start with $f \in \Linf_P\left( \X, \R^{d_0} \right)$.
On the one hand, sample $f$ to get a graph signal on $G_n$, and then pass it through the discrete MPGNN $\Theta_{G_n}$.
On the other hand, first pass $f$ through the continuous counterpart cMPGNN $\Theta_{W, P}$, and then sample the result to get a graph signal on $G_n$.
Compare these two graph signals.

Our first theorem is based on the bounded differences method and the McDiarmid inequality (\cref{th: Multi dimendional McD Holder}). 
It encompasses numerous examples of MPGNNs that include~\cref{ex: 1,ex: 2,ex: 3,ex: 4}. 
For~\cref{ex: 5} however, we obtain a different bound based on other concentration inequalities.

\subsubsection{The bounded differences method}\label{sec: cv by mcdiarmid}

Let us detail the assumptions that we make on the MPGNN. 
Our result is based on the so-called McDiarmid inequality~\citep{mcdiarmid_inequality_1989} (\cref{th: McDiarmid,th: Multi dimendional McD Holder}), which says that a multivariate function of independent random variable has a sub-Gaussian concentration around its mean, provided that it satisfies the following notion of bounded differences.

\begin{definition}[Bounded Differences Property]\label{def: bounded diff}
	Let $f: \mathcal{E}^n \to \R$ be a function of $n$ variables. 
    We say that $f$ has the bounded differences property if there exist $n$ nonnegative constants  $c_1, \dots, c_n$ such that for any $1\leq i \leq n$
	\begin{equation}\label{eq: bounded diff}
        \abs*{f(x_1\dots,x_{i-1},x_i,x_{i+1},\dots,x_n)-f(x_1,\dots,x_{i-1},x_i',x_{i+1},\dots,x_n)} \leq c_i\,,
	\end{equation}
	for any $x_1,\dots,x_n,x_i' \in \mathcal{E}$.
\end{definition}
In plain terms, whenever one fixes all but one of the components of $f$, the variations should be bounded. 
For a fixed $x_1\in\X$, we are then interested in the bounded differences of the layers of our MPGNN
\begin{equation}\label{eq: map where we get D_n from}
(x_2,\dots,x_n)\mapsto F^{(l)}\left(f^{(l-1)}(x_1), \multiSet*{\left( f^{(l-1)}(x_i),W(x_1,x_i)\right) }_{ i\ge 2} \right)\,,
\end{equation}
as a map of the $n-1$ variables $x_2,\dots,x_n$. 
These bounded differences depend on $x_1$. 
Remark that, if we call them $c_2(x_1)\dots,c_n(x_1)$, since~\cref{eq: map where we get D_n from} is invariant to the permutations of $x_2,\dots,x_n$, they can be taken all equal $c_2(x_1)=\dots=c_n(x_1)$. 
To handle some examples, we allow to take a transform $\psi$ of the layers before computing the bounded difference, which, as we will see, may influence the final rate of convergence. 
For instance, for~\cref{ex: 4} with $h=(\cdot)^{1/p}$, one should lift back to the power $p$ \emph{before} computing the bounded differences to get the proper rate (\ie take $\psi=h$ in the assumption below). 
The precise assumption that we make is the following.

\begin{assumption}[Bounded differences of layers]\label{ass: bounded}
	There is a function $\psi$, assumed invertible on the proper domain such that for all layers $l$: 
    \begin{enumerate}[label=(\roman*)]
		\item\label{item 1 assumtion bounded} 
            We have the bounded difference inequality
        \begin{multline*}
        \uesssup{x_1\in \X,\,P} \Big\Vert \psi^{-1}\left(F^{(l)}\left(f^{(l-1)}(x_1), \multiSet*{( f^{(l-1)}(x_i),W(x_1,x_i)) }_{ i\ge 2} \right)\right) \\
        - \psi^{-1}\left(F^{(l)}\left(f^{(l-1)}(x_1), \multiSet*{( f^{(l-1)}(x'_i),W(x_1,x'_i))}_{i\ge 2} \right)\right)\Big\Vert_\infty \leq D^{(l)}_n
             \,,
        \end{multline*}
		for all $x_i,x'_i$, $i \geq 2$, and $x_i = x'_i$ except for $i=2$, where $D_n^{(l)}$ is some rate.
		\item\label{item 2 assumtion bounded} 
            We have
		\begin{multline*}
			\uesssup{x_1\in \X,\,P} \Big\Vert \psi\left(\E \psi^{-1} \left( F^{(l)}\left(f^{(l-1)}(x_1), \left\lBrace( f^{(l-1)}(X_i),W(x_1,X_i)) \right\rBrace_{i\ge 2} \right)\right)\right) \\
			- \E F^{(l)}\left(f^{(l-1)}(x_1), \left\lBrace( f^{(l-1)}(X_i),W(x_1,X_i)) \right\rBrace_{i\ge 2} \right)\Big\Vert_\infty \leq \widetilde D^{(l)}_n
            \,,
		\end{multline*}
		where the expectation is over $X_2, \ldots, X_n \sim P$ and $\widetilde D_n^{(l)}$ is some rate.
		\item\label{item 3 assumtion bounded} 
            The function $\psi$ is $\alpha_\psi$-Hölder on its domain with $0<\alpha_\psi\leq 1$:
		\begin{equation*}
            \norm*{\psi(y) - \psi(y')}_\infty \leq K_\psi \norm*{y-y'}_\infty^{\alpha_\psi}
            \,.
		\end{equation*}
	\end{enumerate}
	
\end{assumption}
In several examples, it will be enough to consider $\psi=\id$, such that the layers themselves satisfy the bounded difference property. 
In that case,~\cref{item 2 assumtion bounded,item 1 assumtion bounded} from~\cref{ass: bounded} above are automatically satisfied with exponent $\alpha_\psi=1$. 
However, we allow for a transform $\psi^{-1}$ before computing the bounded differences, to handle more general cases (namely, the generalized mean~\cref{ex: 4,ex: d}) through the modified McDiarmid inequality~\cref{th: Multi dimendional McD Holder} in~\cref{app: Useful results}. 
The next assumption relates to the existence of the continuous counterpart of the layers as defined by~\cref{def: construction of continuous limit}. 
We recall that this is defined by the convergence of the expectation of the layers, here we just denote the rate of convergence by $s_n$.

\begin{assumption}[Continuous counterpart of layers]\label{ass: continuous layer}
	For any layer $l$, $\FF^{(l)}$ is the continuous counterpart of $F^{(l)}$, and using the notations of~\cref{def: construction of continuous limit}, we let $\left(s_n^{(l)}\right)$ be a sequence of positive reals such that
    \begin{equation}\label[inequality]{eq: def of s_n}
		\left\lVert F_{P,n}^{(l)}\left(f^{(l-1)},W\right) - \FF_{P}^{(l)}\left(f^{(l-1)},W\right)\right\rVert_{\infty} \le s_n^{(l)} \to 0.
	\end{equation}
	for all $n$. 
\end{assumption}

Finally, we state a final assumption on the Hölder stability of the layers, which is usually easy to verify.

\begin{assumption}[Hölder property of layers]\label{ass: holder layer}
	There exists an exponent $0<\alpha_F \leq 1$, and, for any $1\le l\le L$, there exist a constant $K_{F,n}^{(l)} > 0$ such that the aggregations $F^{(l)}$ satisfy
		$$
        \left\lVert F^{(l)}\left(z_1,\left\lBrace(z_i, w_i)\right\rBrace_{i\geq 2}\right) - F^{(l)}\left(z_1',\left\lBrace (z'_i, w_i)\right\rBrace_{i \geq 2}\right)\right\rVert_{\infty} \le K_{F,n}^{(l)} \max_{1\leq i \leq n} \norm*{z_i-z'_i}_{\infty}^{\alpha_F}
		$$
		for all $z_i,z'_i \in \R^{d_{l-1}}$ and $w_i \in [0,1]$. %
		Moreover, the $K_{F,n}^{(l)}$ are bounded over $n$.
\end{assumption}
Note that the ``for all'' statement implies that the $z'_i$ can be permuted, and the bound must stay valid.
We finish by stating the same assumptions for the readout function in the invariant case.

\begin{assumption}[Readout function]\label{ass: readout}
	The readout function satisfies the following
	\begin{enumerate}[label=(\roman*)]
		\item\label{item 1 ass readout} 
            The readout function $R$ has a continuous counterpart $\RRR$, and we let $r_n$ be such that (recall the definitions of $R_{P,n}$ and $\RRR_{P}$ from~\cref{def: construction of the limit readout})
            \begin{equation}\label[inequality]{eq: def of r_n}
                \norm*{R_{ P,n}(f^{(L)}) - \RRR_{P}(f^{(L)})}_{\infty} \le r_n\,.
		\end{equation}
		\item\label{item 2 ass readout} 
            The readout function has bounded differences(\cref{def: bounded diff}).
            Since the $n$ bounded differences can be taken all equal due to permutation invariance, we call $C_n$ the common bounded difference of 
		$$ (x_1,\dots,x_n)\mapsto R\left( \left\lBrace f^{(L)}(x_1),\dots,f^{(L)}(x_n) \right\rBrace \right)$$
            at each coordinate.
		\item\label{item 3 ass readout} 
            There exists $K_{R,n} >0 $ such that 
            $$\left\Vert R\left(\left\lBrace z_i\right\rBrace_{i\le 1\le n}\right) - R\left(\left\lBrace z_i'\right\rBrace_{i\le 1\le n}\right) \right\Vert_{\infty} \le K_{R,n} \max_{1\le i\le n} \norm*{z_i - z_i' }_{\infty}\,,$$
		for all $z_i, z_i' \in \R^L$. 
        Moreover, the $K_{R,n}$ are bounded over $n$.
	\end{enumerate}

\end{assumption}
Again, the ``for all'' statement implies that the bound remains under any permutation of the $z_i'$. 
Note that, for simplicity, we do not involve any ``Holder'' exponent in the readout function, as all our examples can be treated without. 
Our main result is then the following.

\begin{theorem}[MPGNN convergence towards cMPGNN]\label{th: main result}
Under~\cref{ass: bounded,ass: continuous layer,,ass: holder layer}, 
    for any $0 < \rho \leq 1$, the following assertions are verified.
    \begin{itemize}
        \item \textbf{Equivariant case.}
            With probability at least $1- \rho$, it holds
            \begin{equation}\label[inequality]{eq: main result bound}
            \begin{split}
                & \MAE_X\left(\Theta_{G_n}(\iota_X(f)), \Theta_{W,P}(f)\right) \\
                & \le \sum_{l=1}^{L}A_n^{(l,L)}\left[\left(\frac{1}{2}\left(D_{n}^{(l)}\right)^2 n \ln\left(\frac{2^{L+2-l} d_l n}{\rho}\right)\right)^{\nicefrac{\alpha_\psi}{2}} + \widetilde D_{n}^{(l)} + s_{n}^{(l)}\right]^{\alpha_F^{L-l}}\,,
            \end{split}
            \end{equation}
        \item \textbf{Invariant case.}
            With additionally~\cref{ass: readout}, it holds with probability at least $1 - \rho$
            \begin{equation}\label[inequality]{eq: main result bound inv}
            \begin{split}
                & \left\Vert \overline{\Theta}_{G_n}(\iota_X(f)) - \overline{\Theta}_{W,P}(f) \right\Vert_{\infty} \\
                &\le K_{R,n}\sum_{l=1}^{L}A_n^{(l,L)}\left[\left(\frac{1}{2}\left(D_{n}^{(l)}\right)^2 n \ln\left(\frac{2^{L+3-l} d_l n}{\rho}\right)\right)^{\nicefrac{\alpha_\psi}{2}} + \widetilde D_{n}^{(l)} + s_{n}^{(l)}\right]^{\alpha_F^{L-l}} \\
                & \quad + C_n\sqrt{n\ln\left(\frac{4d_{L}}{\rho} \right)} + r_n \,.
            \end{split}
            \end{equation}
    \end{itemize}	
    In both cases, $A_n^{(l,L)} = \prod_{k=l+1}^{L}\left(K_{F,n}^{(k)}\right)^{\alpha_F^{L-k}}$ for $1\le l\le L$, with the usual convention that a product indexed by the empty set always equals $1$.

    Recall that $D_n^{(l)}, \widetilde D_n^{(l)}$ and $\alpha_\psi$ are respectively the bounded differences and the Hölder exponent from~\cref{ass: bounded}; $s_n^{(l)}$ is from~\cref{ass: continuous layer}; the $K_{F,n}^{(l)}$ and $\alpha_F$ are from~\cref{ass: holder layer}; and $r_n, K_{R,n}$ and $C_n$ are from~\cref{ass: readout}.
\end{theorem}

\begin{proof}[Sketch of proof](See Appendix~\ref{app: proof CV Mc diarmid} for full proof)
We prove the result by induction on the number of layers $L$. At each step, we bound $\Vert z_i^{(L)} - f^{(L)}(X_i)\Vert$ for all $i$. This is done by conditioning over $x_i$ and finding a bound of 
$$\left\Vert F^{(L)} \left(f^{(L-1)}(x_i), \multiSet*{\left(f^{(L-1)}(X_k),W(x_i,X_k)\right)}_{k\ne i}\right) - f^{(L-1)}(x_1)\right\Vert_{\infty} $$
	that does not depend on $x_i$, using a succession of triangular inequalities, the Hölder-type property of Assumption~\ref{ass: holder layer} and McDiarmid's inequality. We then turn it into a bound for $\lVert z_i^{(L)} - f^{(L)}(X_i)\rVert_{\infty}$ via the law of total probability and conclude with a union bound over $i$.
\end{proof}

We present some corollaries and their consequences on our pool of examples.
We can get a more explicit rate by disregarding multiplicative constants, as in the following corollary.
\begin{corollary}\label{cor: main result lesssim}
    Under the assumptions of~\cref{th: main result}, for $n$ large enough, and for $0< \rho< 1$, the following holds with probability at least $1- \rho$.
    \begin{equation}\label[inequality]{eq: main result bound lesssim}
	       \MAE_X\left(\Theta_{G_n}(\iota_X(f)), \Theta_{W,P}(f)\right) 
           \lesssim \left(D_n^2n\ln\left(\frac{n}{\rho} \right)\right)^{\nicefrac{\alpha_F^{L-1}\alpha_\psi}{2}} + \widetilde D_n^{\alpha_F^{L-1}} + s_n^{\alpha_F^{L-1}}\,,
	\end{equation}
	Where $D_n=\max_l D_n^{(l)}$, $\widetilde D_n=\max_l \widetilde D_n^{(l)}$, $s_n = \max_l s_n^{(l)}$ and the symbol $\lesssim$ hides some multiplicative constants which depend on $K_{F,n}^{(1)}, \dots K_{F,n}^{(L)}, K_{R,n}$ and are bounded over $n$. 
\end{corollary}

\begin{proof}
    Knowing that $(D_n)$ and $(s_n)$ both tend to zero, there is an integer $n_0$ such that, for any $n \ge n_0$, for each term of the sum in the right-hand side of~\cref{eq: main result bound}, the base under the exponent $\alpha_{L-l}$ is smaller than $1$. 
    After, roughly majorizing each term of the sum, we disregard the multiplicative constants that are bounded over $n$ under the symbol $\lesssim$. 
    Finally, we use the inequality $(x+y)^a \leq x^a + y^a$ for $a\leq 1$.
    Recall that the $K_{F,n}^{(l)}$ and $K_{R,n}$ are assumed bounded from~\cref{ass: holder layer,ass: readout}.
\end{proof}

Notice that~\cref{eq: main result bound} from~\cref{th: main result} was non-asymptotic, that is, valid for any integer $n$.
It is no longer in~\cref{cor: main result lesssim} above, 
~\cref{eq: main result bound lesssim} only holds for $n$ large enough where the order of magnitude of the ``large enough'' will depend on the example.

The asymptotic behavior of~\cref{eq: main result bound lesssim} is generally determined by $D_n$. 
If the latter does not decrease fast enough, the inequality becomes meaningless, as it does not yield to convergence. 
In particular, we have the following important corollary.

\begin{corollary}\label{cor: main th}
	If $D_n= o\left(\frac{1}{\sqrt{n\ln n}}\right)$ then $\MAE_X(\Theta_{G_n}(\iota_X(f)), \Theta_{W,P}(f))$ converges in probability towards $0$. 
    Moreover, if $D_n \sim  n^{-\beta}$ with $\beta>1/2$, then the convergence is almost sure.
\end{corollary}

\begin{proof}
    From~\cref{eq: main result bound lesssim}, the first part is immediate. 
	For the second part, denote $Y_n = \MAE_X(\Theta_{G_n}(\iota_X(f)), \Theta_{W,P}(f))$. 
    There is a constant $C$ such that for any $\rho > 0$, 
    $$
    Y_n \le C\left( \left(D_n^2n\ln\left(\frac{n}{\rho} \right)\right)^{\nicefrac{\alpha_F^{L-1}\alpha_\psi}{2}} + \widetilde D_n^{\alpha_F^{L-1}} + s_n^{\alpha_F^{L-1}} \right) \,,
    $$ 
    holds with probability as least $1 - \rho$.
    This is equivalent to
    \begin{equation}\label[inequality]{eq: main cor cv ps 1}
    \P\left( Y_n \ge \eps \right) \le n e^{- \frac{1}{nD_n^2}\left( \frac{\eps}{C}- \widetilde D_n^{\alpha_F^{L-1}} - s_n^{\alpha_F^{L-1}} \right)^{\nicefrac{-\alpha_F^{L-1}\alpha_\psi}{2}} }
    \,.
    \end{equation} 
    For any $\eps > 0$.
    We claim that $\P\left( Y_n > \eps \right) $ is summable for any $\eps$.
    Indeed, since $\widetilde{D}_n$ and $s_n$ tend to zero, 
    \[
    \left( \frac{\eps}{C}- \widetilde D_n^{\alpha_F^{L-1}} - s_n^{\alpha_F^{L-1}} \right)^{\nicefrac{-\alpha_F^{L-1}\alpha_\psi}{2}} = C_{\eps} + o(1)
    \,,
    \]
    where $C_{\eps} = \left( \frac{\eps}{C} \right)^{\nicefrac{-\alpha_F^{L-1}\alpha_\psi}{2}}$.
    Therefore, since $\frac{1}{nD_n^2} \sim n^{2\beta-1}$, and $\beta > 1 / 2$,~\cref{eq: main cor cv ps 1} gives,
    \begin{align*}
        \P\left( Y_n > \eps \right) & \le ne^{-\left( \frac{C_\eps}{nD_n^2} + o\left(\frac{1}{nD_n^2}\right) \right) }
        = ne^{-\left( \frac{C_\eps}{nD_n^2} + o\left(n^{2\beta-1}\right) \right) } \\
                                      & = ne^{-\left( C_\eps \left(n^{2\beta - 1} + o\left( n^{2\beta - 1} \right)\right)  + o\left(n^{2\beta-1}\right) \right) }
        = ne^{-\left( C_\eps n^{2\beta - 1} + O\left( n^{2\beta - 1} \right)\right)  } \\
                                      & O\left( ne^{-C_\eps n^{2\beta-1}} \right) = o\left( \frac{1}{n^2} \right)\,,
    \end{align*}
    which proves the summability of $\P\left( Y_n > \eps \right) $.
    We conclude by~\cref{lem: borel-cantelli} of Borel-Cantelli.
\end{proof}

This corollary provides a sufficient condition for an MPGNN on a random graph to converge to its continuous counterpart on the random graph model: in words, its aggregation function needs to have sharp enough bounded differences. 
Below we investigate whether our examples have such sharp bounded differences. 
Under mild regularity conditions, this is the case for all examples but~\cref{ex: 5}. The proof is in~\cref{sec: app examples proofs}.

\begin{proposition}\label{prop: main th on the examples}
	We present the application of~\cref{th: main result} on the Examples. 
    In all cases, the message functions $\m^{(l)}$ are supposed Lipschitz continuous and bounded. 
    Additional regularity assumptions are needed for some examples.
	\begin{itemize}
        \item \textbf{\Cref{ex: 1,ex: a}.}
            $D_n=O(1/n)$, $\widetilde D_n=0$, $s_n=r_n=0$, and $\alpha_\psi = \alpha_F = 1$. 
            The final rate is therefore $O\left(\sqrt{\ln n /n}\right)$.
        \item \textbf{\Cref{ex: 2,ex: b}.}
            Suppose that $W$ is bounded away from zero, that is, there is $a>0$ such that $W\ge a$. 
            Then $D_n=O(1/n)$, $\widetilde D_n=0$, $s_n=O(1/\sqrt{n})$, and $\alpha_\psi = \alpha_F = 1$. 
            The final rate is therefore $O\left(\sqrt{\ln n /n}\right)$.
        \item \textbf{\Cref{ex: 3,ex: c}.}
            Suppose there is $a, b >0$ and $K_c>0$ such that $a<c(x,y,t)<b$ and $\vert c(x,y,t) - c(x',y',t)\vert \le K_c(\Vert x-x'\Vert_{\infty} + \Vert y-y'\Vert_{\infty})$, $\forall x,x',y,y',t$. 
            Then $D_n=O(1/n)$, $\widetilde D_n=0$, $s_n=O\left(1/\sqrt{n}\right)$, and $\alpha_\psi = \alpha_F = 1$. 
            The final rate is therefore $O\left(\sqrt{\ln n /n}\right)$.
        \item \textbf{\Cref{ex: 4,ex: d}.}
            Suppose that $h$ is $\alpha_h$-Holder and that $h^{-1}$ is bounded, and that~\cref{ass: holder layer} is satisfied for some exponent $0<\alpha_F \leq 1$ (see examples below). 
            Then $D_n = O(1/n)$, $\widetilde D_n=O(1/n^{\nicefrac{\alpha_h}{2}})$, $s_n = O(1/n^{\nicefrac{\alpha_h}{2}})$, and $\alpha_\psi=\alpha_h$. 
        The final rate is therefore $O(\ln n/n)^{\nicefrac{\alpha_h \alpha_F^L}{2}})$. 
            For the exponents $\alpha_h, \alpha_F$, we can give several usual cases:
		\begin{itemize}
        \item If $h = x\mapsto x^{\nicefrac{1}{p}}$ and $h^{-1} = x\mapsto x^p$ (moment-based aggregation from~\cite{corso_principal_neighbourhood_agg_2020}), then $\alpha_h=1/p$ and $\alpha_F=1$.
			\item If $h^{-1}$ is Lipschitz with regard to,\ $x$ (and recall that $h$ is $\alpha_h$-Holder), then $\alpha_F = \alpha_h$.
			\item If the data domain is bounded and $h^{-1}$ and $h$ are both Lipschitz (e.g. for geometric mean on a bounded domain), then $\alpha_h = \alpha_F = 1$.
		\end{itemize}
    \item \textbf{\Cref{ex: 5,ex: e}.}
            The bounded differences do not converge to zero, therefore cannot satisfy~\cref{cor: main th}.
	\end{itemize}
\end{proposition}
\begin{proof}
	Calculation and verification of the Theorem's assumptions are done in~\cref{sec: app examples proofs}. 
\end{proof}

\Cref{tab: prop_on_the_examples} sums up these results. 
For a network with max aggregation, the bounded differences are not sharp enough for~\cref{th: main result} to conclude. 
We thus treat this case separately in the next section.

\begin{table}[htbp]
    \centering
	\small
	\begin{tabular}{@{}lcccc@{}} \toprule
		\textbf{Example} & $D_n$ & $\widetilde D_n$& $s_n$ & Convergence by~\cref{cor: main result lesssim} \\ \midrule
        \textbf{\labelcref{ex: 1}-\labelcref{ex: a}} & $O(1/n)$ & $0$ & $0$ & $O\left(\sqrt{\ln n / n}\right)$ \\
		\textbf{\labelcref{ex: 2}-\labelcref{ex: b}} & $O(1/n)$ & $0$ & $O\left(1/\sqrt{n}\right)$ & $O\left(\sqrt{\ln n / n}\right)$ \\
		\textbf{\labelcref{ex: 3}-\labelcref{ex: c}} & $O(1/n)$ & $0$ & $O\left(1/\sqrt{n}\right)$ & $O\left(\sqrt{\ln n / n}\right)$ \\
        \textbf{\labelcref{ex: 4}-\labelcref{ex: d}} & $O(1/n)$ & $O\left(n^{\nicefrac{-\alpha_h}{2}}\right)$ & $O\left(n^{\nicefrac{-\alpha_h}{2}}\right)$ & $O\left(\bigl(\ln n / n\bigr)^{\nicefrac{\alpha_h\alpha_F^L}{2}}\right)$ \\
		\textbf{\labelcref{ex: 5}-\labelcref{ex: e}} & $\Omega(1)$ & $-$ & $-$ &  \xmark \\ \bottomrule
	\end{tabular}
	\caption{Table summing up the convergence rates of this section. 
    See~\cref{prop: main th on the examples} and~\cref{cor: main result lesssim} for details.}
	\label{tab: prop_on_the_examples}
\end{table}

\subsubsection{Convergence of \emph{max} Aggregation Message-Passing GNNs}\label{sec: cv for max}

In this subsection, we specifically treat the example of max aggregation. 
Since the bounded differences method fails, we need another method to estimate the deviation between a \emph{maximum} message-passing on a large random graph and its continuous counterpart.

We shall start by observing a simple example where everything is real-valued and smooth.
Let $f$ be a feature map on the latent space $\X$ and $\m$ be a message function which we assume to be real-valued.
We call 
$$g(x,y)=W(x,y)\m(f(y)) \,,$$
such that the \emph{maximum} message-passing around a node $i$ is $\max_{j\neq i} g(X_i, X_j)$.
Moreover, we suppose that $g\colon \X^2 \to \R$ is $K_g$-Lipschitz continuous and that the measure $P$ is strictly positive.
Therefore, in virtue of~\cref{lem: esssup = sup}, the continuous counterpart of the \emph{maximum} message-passing around any point $x \in  \X$ is 
$$\esssup_{P} g(x,\cdot) = \sup g(x,\cdot) \,.$$
Our goal is to estimate
$$
\P(| \max_i g(x,X_i) - \sup g(x,\cdot) | \ge \eps)
\,,$$ 
for $\eps>0$ and $x \in \X$. 
By definition of the supremum and by independence of the $X_i$, we have that
\begin{align*}
         \P(| \max_i g(x,X_i) - \sup g(x,\cdot) | \ge \eps) 
        &= \P( \max_i g(x,X_i) \le \sup g(x,\cdot) - \eps) \\
        &= \P( g(x,X_1)  \le \sup g(x,\cdot) - \eps)^n  \\
        &= \P(|g(x,X_1) - \sup g(x,\cdot)|  \ge \eps)^n \\
        &=\bigl(1 - \P(|g(x,X_1) -  \sup g(x,\cdot)|  < \eps) \bigr)^n\,.
\end{align*}
By continuity and compactness, there is $x^*\in\X$ such that $\sup g(x,\cdot) = g(x,x^*)$, and, by Lipschitz continuity of g, for any $x\in \X$,
\[
    \Vert (x,X_1) - (x,x^*)\Vert = \Vert X_1-x^*\Vert < \eps / K_g \implies | g(x,X_1)-g(x,x^*))| < \eps \,.
\]
Thus, we obtain the bound
\begin{align}
    \P(| \max_i g(x,X_i) - \sup g(x,\cdot) | \ge \eps)
    & = \bigl(1 - \P(|g(x,X_1) -  g(x,x^*)|  < \eps) \bigr)^n \notag \\
    & \le \left(1 - \P(\lvert X_1-x^*\rvert < \eps / K_g) \right)^n \notag \\
    & = \bigl(1 - P(B(x^*,\eps / K_g)\cap\X) \bigr)^n \label[inequality]{eq: max estimation by ball}\,.
\end{align}
where $B(x^*,\eps / K_g)$ is the open ball, for the infinite norm, of center $x^*$ and radius $\eps / K_g$ in $\R^d$. 

We have now reached a point where, if we seek to go further, we need to be able to give an approximation of the measure of a ball in $\X$. 
To this end, we introduce the notion of ``retention'' of the Lebesgue measure which we call the \emph{volume retaining property}. 
The purpose is to estimate from below the measure of a ball centered anywhere in $\X$.

\begin{definition}[Volume retaining property]\label{def: volume retaining}
	We say that the probability space $(\X,P)$ has the $(r_0,\kappa)$-volume retaining property if for any $r\le r_0$ and for any $x\in\X$,
    \begin{equation}\label[inequality]{eq: volume retaining}
		P(B(x,r)\cap \X) \ge \kappa \lambda_d(B(x,r))\,.
	\end{equation}
	Where $B(x,r)$ is the ball of center $x$ and radius $r$ and $\lambda_d$ is the classical $d$-dimensional Lebesgue measure in $\R^d$.
\end{definition}

Notice that the volume retention property implies that the measure must be strictly positive.

The condition from~\cref{def: volume retaining} is simultaneously a condition on the probability measure $P$ as well as the geometry of $\X$. 
In the particular case where $P$ itself is the Lebesgue measure, this becomes a purely geometrical condition on the shape of $\X$ at the boundary.
For instance, it is no difficulty to see that the unit hypercube $[0,1]^d$ has the $(1,1/2^d)$-volume retaining property.
On the other hand, a typical shape that makes this property fail is an arbitrarily sharp peak. 
For example, the peak at the contact point of the complementary of two tangent open disks. 

This hypothesis is also standard in other related contexts in which estimating the measure of balls is required in order to obtain some convergence rates.
In Set Estimation, it is a cornerstone assumption. 
The goal of this branch of statistics is to estimate the compact support of a probability distribution on a metric space from samples, which is typically achieved by considering the union of balls centered at points drawn from that distribution. 
To this end, the property from~\cref{def: volume retaining} was introduced by~\citep{cuevas_pattern_1990, cuevas_plugin_97, cuevas_rodriguez_casal_2004}. 
More contemporary, in Topological Data Analysis, it is used to measure the convergence rate of persistence diagrams, when the data is assumed to be drawn from a probability distribution supported on a compact metric space~\citep{chazal_subsampling_2015, chazal_convergence_2015, chazal_rate_2016}.

For a volume retaining probability space, we prove the following concentration inequality.

\begin{lemma}[Concentration inequality for volume retaining space]\label{lem: concentration bound for max}
	Let $g:\X^2 \to \R^{q}$ be $K_g$-Lipschitz and $(\X,P)$ have the $(r_0,\kappa)$-volume retaining property for some $r_0,\kappa>0$. 
    Recall that $\X\subset R^d$, then for any $\rho \ge e^{-n \kappa r_0^d 2^d}$, for any random variables $X_1, \dots, X_n \iid P$, with probability at least $1-\rho$, it holds
    \begin{equation}\label[inequality]{eq: lemma max concentration} 
        \lVert \max_{1\le i\le n}g(x,X_i) - \sup g(x,\cdot)\rVert_{\infty}  \le \frac{K_g}{2} \left(\frac{\ln(q/\rho)}{n \kappa} \right)^{1/d}\,.
    \end{equation}
\end{lemma}
\begin{proof}
	We write the proof assuming $q=1$, the case $q\ge 1$ follows easily by a union bound. 
    Clearly, volume-retention implies strict positiveness of the measure. 
    The calculation is exactly the same as conducted in the introductory part of this~\cref{sec: cv for max}, until we reach~\cref{eq: max estimation by ball}, where we must estimate
    \begin{equation}\label{eq: proof concentration max 1}
		\left(1 - P\left(B(x^*,\eps/K_g)\cap\X\right) \right)^n \,.
	\end{equation}
	Recall that we consider balls with regard to the infinity norm.
    Thus, by volume retention, for $\eps\le r_0K_g$,~\cref{eq: proof concentration max 1} is bounded by
	\begin{equation}
		\left(1 - P\left(B(x^*,\eps/K_g)\cap\X\right) \right)^n \le \left( 1  - \kappa \left(\frac{2\eps }{ K_g}\right)^d \right)^n \le e^{-n \kappa\left(\frac{2 \eps}{K_g}\right)^d}\,.
	\end{equation}
	Which implies that for $\rho \ge e^{-n \kappa r_0^d 2^d}$ with probability at least $1-\rho$,
    \begin{equation}\label[inequality]{eq: proof concentration max 2}
    |\max_{1\le i\le n}g(x,X_i) - \sup g(x,\cdot)| \le \frac{K_g}{2} \left(\frac{\ln(1/\rho)}{n \kappa} \right)^{1/d}\,,
    \end{equation}
    Finally,~\cref{eq: proof concentration max 2} combined with a union bound yields to the desired result for $q\ge 1$.
\end{proof}

Convergence in the case of maximum aggregation relies on the smoothness of the feature map $f$.
Therefore, we will need the following regularity property about the intermediate layers.

\begin{proposition}\label{prop: app Lipschitz of f^l for max}
	Assume that $W$, $f=f^{(0)}$, as well as the $m^{(l)}$ are all Lipschitz continuous, then the functions $f^{(0)},\dots,f^{(L)}$ are Lipschitz continuous too. 
    We denote by  $K_f=K_{f^{(0)}},\dots,K_{f^{(L)}}$ their Lipschitz constants.
\end{proposition}

\begin{proof}
	It is already assumed for $l=0$. 
    Suppose it is true for $l\ge 1$, we have 
    \[
        f^{(l+1)}(x)=\sup_y W(x,y)\m^{(l+1)}(f^{(l)}(y)) = \sup_y g(x,y) \,,
    \]
    where $g$ is $K_W \Vert \m^{(l+1)}\circ f^{(l)}\Vert_{\infty}+ K_{\m^{(l)}}K_{f^{(l)}}$-Lipschitz. 
    Then from~\cref{lem: Lipschitz of sup} $f^{(l+1)}$ is also Lipschitz continuous.
\end{proof}

We are now ready to state the non-asymptotic bound for an MPGNN with maximum aggregation.

\begin{theorem}[Non-asymptotic convergence of max-MPGNN towards cMPGNN]\label{th: main result max}
	Suppose that, $(\X,P)$ has the $(r_0,\kappa)$-volume retaining property and that $f,W$ and the $\m^{(l)}$ are Lipschitz continuous. 
    Let $\rho \ge 2ne^{-n \kappa r_0^d 2^d}$ and $n$ large enough for $0<\rho<1$ to hold, the following inequalities hold.
    \begin{itemize}
        \item \textbf{Equivariant case.}
            With probability at least $1-\rho$,
            \begin{equation}\label[inequality]{eq: max main result bound}
                \MAE_X\left(\Theta_{G_n}\left(\iota_X(f)\right), \Theta_{W,P}(f)\right)  \le \sum_{l=1}^{L}B^{(l,L)}\frac{K_{f^{(l)}}}{2}\left (\frac{1}{n \kappa } \ln\left( \frac{2^{L+1-l}nd_l}{\rho}\right) \right)^{\nicefrac{1}{d}} \,, 
            \end{equation}
        \item \textbf{Invariant case.} 
            With probability at least $1-\rho$, 
            \begin{equation}\label[inequality]{eq: max main result bound inv}
                \begin{split}
                    \left\Vert \overline{\Theta}_{G_n}(\iota_X(f)) - \overline{\Theta}_{W,P}(f) \right\Vert_{\infty} & \le \sum_{l=1}^{L}B^{(l,L)}\frac{K_{f^{(l)}}}{2}\left (\frac{1}{n \kappa } \ln\left( \frac{2^{L+2-l}nd_l}{\rho}\right) \right)^{\nicefrac{1}{d}} \\ 
                    & \quad + \left (\frac{1}{n} \ln\left( \frac{2d_L}{\rho}\right) \right)^{1/d}\,.
                \end{split}
            \end{equation}
    \end{itemize}
    Where $B^{(l,L)} = \prod_{k=l+1}^{L}K_{\m^{(k)}}$ with the usual convention that a product indexed by the empty set always equals $1$.
\end{theorem}

By grossly majoring each term of the sum in~\cref{eq: max main result bound}, and disregarding the constants, we get the following corollary.
We only write the statement for the equivariant case as the bound for the invariant case would be similar.
\begin{corollary}\label{cor: main result max}
    Under the assumptions of~\cref{th: main result max}, let $\rho \ge 2ne^{-n \kappa r_0^d 2^d}$ and $n$ large enough for $0<\rho<1$ to hold. 
    Then with probability at least $1-\rho$:
	\begin{equation*}
        \MAE_X\left(\Theta_{G_n}\left(\iota_X(f)\right), \Theta_{W,P}(f)\right) \lesssim \left(\frac{1}{n} \ln\left( \frac{n}{\rho}\right) \right)^{\nicefrac{1}{d}} \,,
    \end{equation*} 
\end{corollary}

Since we made an assumption that involves the volume of a $d$-dimensional ball, the convergence rate for \emph{max} convolution depends on the dimension of the latent space $\X \subset \R^d$, where it is roughly equal to $O\left(n^{-\nicefrac{1}{d}}\right)$, as opposed to the generally faster rate $O\left(n^{-\nicefrac{1}{2}}\right)$ obtained with the McDiarmid's method from~\cref{th: main result}. 
Intuitively, this is to be expected, as the fast rate is akin to the central limit theorem, while the rate for max convolution follows from the number of balls necessary to cover the latent space (covering numbers), which scales exponentially in its dimension~\citep{vershynin_high-dimensional_2018}.

\subsection{Experimental illustrations}

We illustrate the convergence rates from both~\cref{th: main result,th: main result max} on toy examples.
\begin{figure}[htpb]
	\centering
		\includegraphics[scale=0.31]{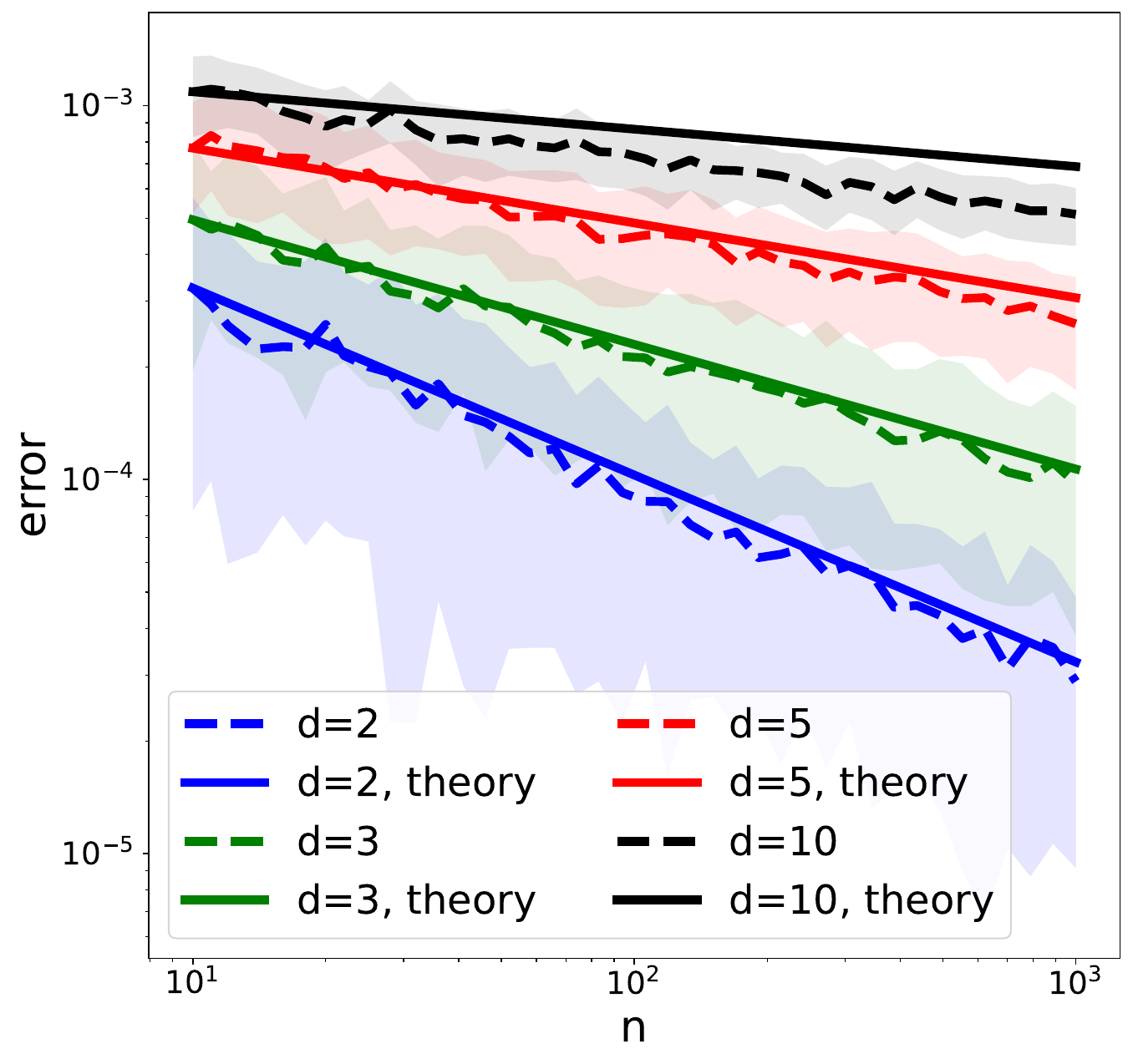} %
	\quad
		\includegraphics[scale=0.31]{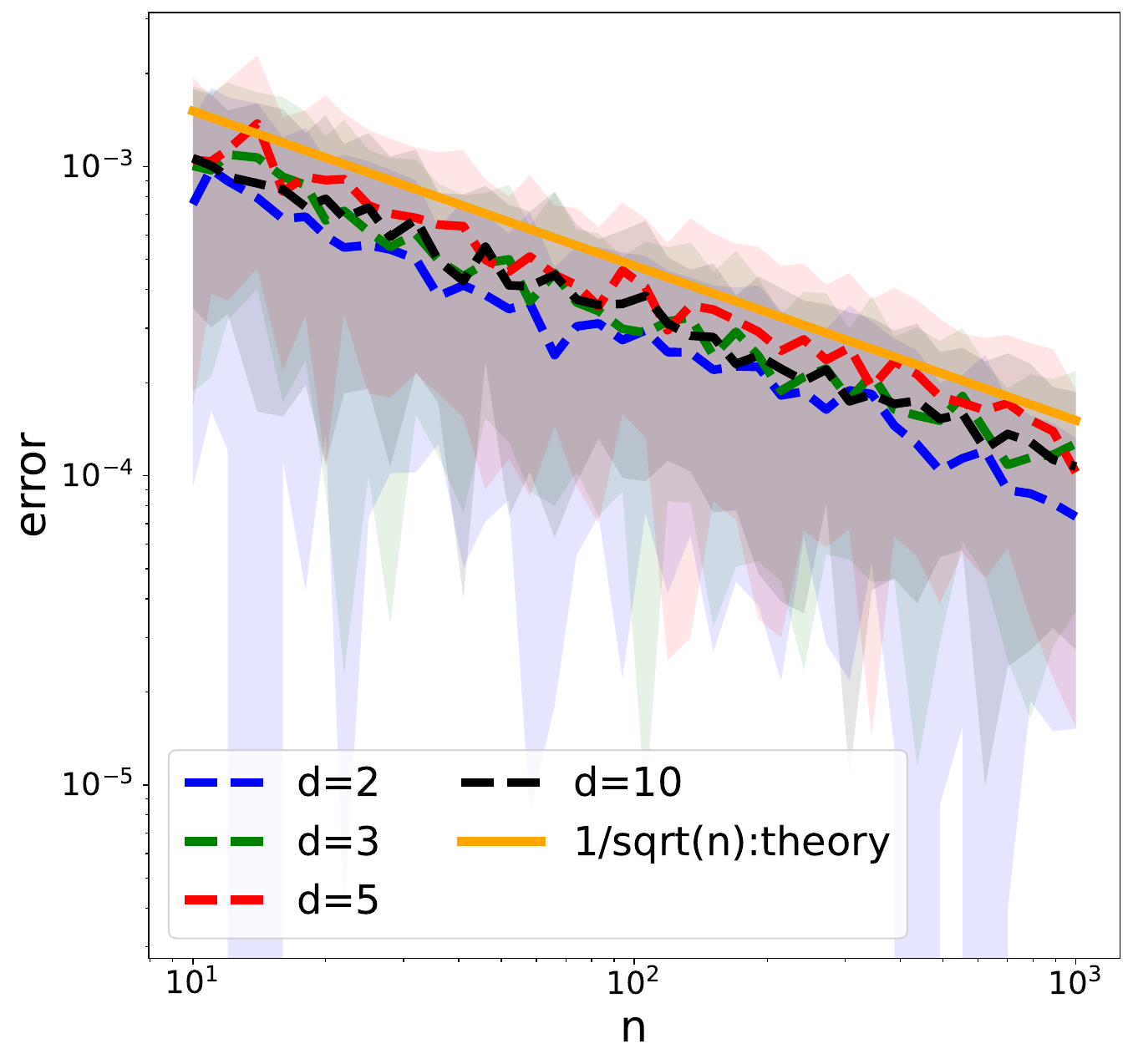} %
	\caption{Numerical experiences for observing the trends of the rates of convergence arising from~\cref{th: main result,th: main result max}.
        The plots are in logarithmic scale.
        Left: max aggregation. Right: mean aggregation. 
        For both figures, the dashed lines represent the experimental error as the graph size increases, while the full lines represent the theoretical rates arising from~\cref{th: main result,th: main result max}. 
This experiment has been conducted for various values of the latent space dimension $d$. The theoretical rates are $1/\sqrt{n}$ for a mean aggregation and $1/n^{\nicefrac{1}{d}}$ for a \emph{max} aggregation.} 
	\label{fig: expé max and mean}
\end{figure}

One difficulty in illustrating our convergence results is that the limit cGNN cannot be computed explicitly in most cases. 
Moreover, since our model of random graphs always produces dense (or even complete weighted) graphs, there are quite strong computational limits to testing very large $n$ to approximate the cGNN. 
Nevertheless, we found that performing several rounds of Monte-Carlo simulation ($50$ is our experiments) yields a reasonable approximation of the limit for the mean example (\cref{ex: 1}). 
Additionally, there is a trick for the max example (\cref{ex: 5}). 
When all parameters are \emph{nonnegative}, the non-linearity is increasing and nonnegative on $\R^+$ (\eg sigmoid), and, say, the input signal lives in $[0,1]^d$, the cGNN can be computed explicitly: it is obtained when all points are $x_i=(1,\ldots,1)$.
Hence we limit our experiments to~\cref{ex: 1,ex: 5}, and our aim is to highlight the influence of the input dimension $d$ on the convergence rate from~\cref{th: main result,th: main result max}: in theory, $O\left(1/\sqrt{n}\right)$ for~\cref{ex: 1} and $O\left(1/n^{\nicefrac{1}{d}}\right)$ for~\cref{ex: 5}. 
We leave other examples for future work (in particular the generalized mean~\cref{ex: 4}, which we found difficult to observe in practice).

For both experiments, the MPGNN has four layers. 
Each layer uses a single layer MLP with sigmoid activation function, and random weights in $[0,1]$. 
A \emph{mean} and a \emph{max} aggregation are respectively used on the left and the right sides of~\cref{fig: expé max and mean}. 
The input signal is a dot product with a random vector in $[0,1]^d$, and the latent variables are uniformly distributed in $[0,1]^d$. 
Each experiment is run for $d = 2, 3, 5$ and $10$. The MAE error is averaged over $50$ experiments, and the standard deviation is also reported in~\cref{fig: expé max and mean}. 

We indeed observe that the convergence rate is in $O\left(1/ \sqrt{n}\right)$ for a \emph{mean} aggregation, no matter the dimension of the latent space. 
Whereas for a \emph{max} aggregation, the speed follows $O\left(1/n^{\nicefrac{1}{d}}\right)$, with $d$ being the latent space's dimension. 
The standard deviation evolves as the error (note that the scale is logarithmic), we do not observe additional concentration phenomenon.

The reproducible code can be found in the following Github repository \url{https://github.com/Matthieu-Cordonnier/convergence-mpgnn/}.

\section{Conclusion}\label{sec: conclusion}
In this work, we have defined continuous counterparts of MPGNNs with very generic aggregation functions on a probability space with respect to a transition kernel. 
We then have shown that under certain conditions, cMPGNNs are limits of discrete MPGNNs on random graphs sampled from the corresponding random graph model. 
Until now, similar result were known for SGNNs, which are more restricted architectures, or for MPGNNs with a degree normalized mean aggregation. 
Our main contribution is to extend this to abstract MPGNNs with generic aggregation functions. 
All along this paper, a focus is given on examples based on mean or weighted mean aggregation (\cref{ex: 1,ex: 2,ex: 3,ex: 4}) and max aggregation (\cref{ex: 5}), but our theorems are not limited to these examples and is in fact verified for mild assumptions on the underlying model. 

Throughout this  paper, we have emphasized the fact that \emph{mean} and \emph{max} aggregation behave differently. 
Albeit, a link between the two still exists: as mentioned before, the generalized mean (\cref{ex: 4}) with moment $h= x \mapsto x^{\nicefrac{1}{p}}$ naturally converges to the maximum aggregation when $p \to \infty$. 
However the bounded difference proof gives a rate in $n^{\nicefrac{-1}{2p}}$, while for the max aggregation, our specific proof based on covering numbers for~\cref{th: main result max} yields $n^{\nicefrac{-1}{d}}$. 
This is intuitively the ``worst'' rate possible in dimension $d$, but is better when $p> d/2$. 
Hence future work could try to unify the two proofs, to obtain a smooth transition between the different rates.

\acks{This work was partially supported by the French National Research Agency in the framework of the « France 2030 » program (ANR-15-IDEX-0002), the LabEx PERSYVAL-Lab (ANR-11-LABX-0025-01), and the ANR grants GRANOLA (ANR-21-CE48-0009), GRANDMA (ANR-21-CE23-0006) and GRAVA  (ANR-18-CE40-0005).}

\appendix

\section{Proof of~\cref{th: main result}}\label{app: proof CV Mc diarmid}

We proceed to the proof of the~\cref{th: main result}.
In all this proof, we denote by $H^{(L)}(\rho)$ the upper bound from~\cref{eq: main result bound}: 
$$H^{(L)}(\rho) = \sum_{l=1}^{L}A_n^{(l,L)}\left[\left(\frac{1}{2}\left(D_{n}^{(l)}\right)^2 n \ln\left(\frac{2^{L+2-l} d_l n}{\rho}\right)\right)^{\nicefrac{\alpha_\psi}{2}} + \widetilde D_{n}^{(l)} + s_{n}^{(l)}\right]^{\alpha_F^{L-l}}.$$
We will prove the results by induction on the depth $L$, we will detail the demonstration for the equivariant case, and the invariant case will easily follow.

\begin{proof}[Proof of the equivariant case,~\cref{eq: main result bound} of~\cref{th: main result}.]
    We start with the equivariant case. 
    Calling $z_i^{(l)} = \left( \iota_X f \right)^{(l)}$ the signal at node $i$ of the intermediate layer $l$ of the GNN, we seek to bound
    \begin{align*}
    \MAE_X\left(\Theta_{G_n}(\iota_Xf), \Theta_{W,P}(f)\right) & =  \max_{1\le i\le n} \left\Vert z^{(L)}_i - \iota_X(f^{(L)})_i\right\Vert\\
    & = \max_{1\le i\le n} \left\Vert z^{(L)}_i - f^{(L)}(X_i)\right\Vert.
    \end{align*}

	We prove the result by induction on $L$. 
    Let $\rho>0$, we recall some notations from~\cref{def: construction of continuous limit}:
	\begin{equation*}
		 F_{P,n}^{(l+1)}\left(f^{(l)},W\right)(x) = \E_{X_2,\dots,X_n}\left[F^{(l+1)} \left(f^{(l)}(x), \left\lBrace\left(f^{(l)}(X_k),W(x,X_k)\right)\right\rBrace_{2\leq k\leq n}\right)  \right] \,,
	\end{equation*}
	and,
	\begin{equation*}
		\FF_P^{(l+1)}\left(f^{(l)},W\right)(x)
		= \FF^{(l+1)}_{P}\left(f^{(l)}(x),\left(f^{(l)},W(x,\cdot)\right)\right)= f^{(l+1)}(x)\,.
	\end{equation*}
	Suppose $L=1$, we shall find a quantity that bounds  all the $\left\Vert z^{(1)}_i - f^{(1)}(X_i)\right\Vert$, for $i$ ranging from $1$ to $n$, with probability at least $1-\rho/n$. 
    Thereby, by a union bound, this quantity will bound their maximum with probability at least $1-\rho$.
	Choose an index $i\in\{1,\dots,n\}$ and let $x_i\in\X$, consider 
    $$\left\Vert F^{(1)} \left(f^{(0)}(x_i), \left\lBrace\left(f^{(0)}(X_k),W(x_i,X_k)\right)\right\rBrace_{k\ne i}\right) - f^{(1)}(x_i) \right\Vert_{\infty}.$$ 
	From a triangular inequality, and by definition of $s_n$ (\cref{eq: def of s_n}), we get the majoration
    \begin{align}\label[inequality]{eq: md proof 1}
			& \left\Vert F^{(1)} \left(f^{(0)}(x_i), \left\lBrace\left(f^{(0)}(X_k),W(x_i,X_k)\right)\right\rBrace_{k\ne i}\right) - f^{(1)}(x_i) \right\Vert_{\infty} \notag\\
			& \le \left\Vert F^{(1)} \left(f^{(0)}(x_i), \left\lBrace\left(f^{(0)}(X_k),W(x_i,X_k)\right)\right\rBrace_{k\ne i}\right) - F_{P,n}^{(1)}\left(f^{(0)},W\right)(x_i)\right\Vert_{\infty}\notag \\
			& \qquad + \left\Vert F_{P,n}^{(1)}\left(f^{(0)},W\right)(x_i) - f^{(1)}(x_1)\right\Vert_{\infty} \notag \\
			& \le \left\Vert F^{(1)} \left(f^{(0)}(x_i), \left\lBrace\left(f^{(0)}(X_k),W(x_i,X_k)\right)\right\rBrace_{k\ne i}\right) - F_{P,n}^{(1)}\left(f^{(0)},W\right)(x_i)\right\Vert_{\infty} 
			+ s_{n}^{(1)}
            \,.
	\end{align}
	Now let us bound~\cref{eq: md proof 1} from above with high probability using our Hölder version of McDiarmid's inequality,~\cref{th: Multi dimendional McD Holder}, on 
    $$(x_1,\ldots,x_{i-1},x_{i+1},\ldots,x_n) \mapsto F^{(1)} \left(f^{(0)}(x_i), \left\lBrace\left(f^{(0)}(x_k),W(x_i,x_k)\right)\right\rBrace_{k\ne i}\right) \,,$$ 
    as a multivariate function of the $n-1$ variables $x_k$ for $k\neq i$. 
	We obtain that for any $x_i\in \X$, with probability at least $1-\rho/n$, 
    \begin{equation}\label[inequality]{eq: md proof 2}
		\begin{split}
			& \left\Vert F^{(1)} \left(f^{(0)}(x_i), \left\lBrace\left(f^{(0)}(X_k),W(x_i,X_k)\right)\right\rBrace_{k\ne i}\right) - f^{(1)}(x_i)\right\Vert_{\infty} \\
            &\quad \le  \left(\frac{1}{2}\left(D_{n}^{(1)}\right)^2 n \ln\left(\frac{2d_1n}{\rho}\right)\right)^{\nicefrac{\alpha_\psi}{2}} + \widetilde D_{n}^{(1)} + s_{n}^{(1)}\,.
		\end{split}
	\end{equation}
	Hence, by conditioning over $X_i$ and applying the law of total probability,~\cref{eq: md proof 2} yields with probability at least $1-\rho/n$
	\begin{equation*}\label{eq: md proof 3}
        \left\Vert z^{(1)}_i - f^{(1)}(X_i)\right\Vert_{\infty} \le \left(\frac{1}{2}\left(D_{n}^{(1)}\right)^2 n \ln\left(\frac{2d_1n}{\rho}\right)\right)^{\nicefrac{\alpha_\psi}{2}} + \widetilde D_{n}^{(1)} + s_{n}^{(1)}\,.
	\end{equation*}
	And, by a union bound,  we can conclude that with probability at least $1-\rho$
	\begin{equation*}
        \max_i \left\Vert z^{(1)}_i - f^{(1)}(X_i)\right\Vert_{\infty} \le \left(\frac{1}{2}\left(D_{n}^{(1)}\right)^2 n \ln\left(\frac{2d_1n}{\rho}\right)\right)^{\nicefrac{\alpha_\psi}{2}} + \widetilde D_{n}^{(1)} + s_{n}^{(1)} \le H^{(1)}(\rho)\,,
	\end{equation*}
    which proves the case $L=1$.

	Now, suppose the theorem true for $L\ge 1$. 
    For any node $i$, we have 
	\begin{align}
	& \left\Vert z^{(L+1)}_i - f^{(L+1)}(X_i)\right\Vert_{\infty} \notag \\
	&  \le \left\Vert z^{(L+1)}_i - F_{P,n}^{(L+1)}(f^{(L)},W)(X_i) \right\Vert_{\infty}\notag \\
	& \quad + \left\Vert F_{P,n}^{(L+1)}(f^{(L)},W)(X_i) + f^{(L+1)}(X_i)\right\Vert_{\infty}\notag \\
	&  \le \left\Vert z^{(L+1)}_i - F_{P,n}^{(L+1)}(f^{(L)},W)(X_i) \right\Vert_{\infty} + s_{n}^{(L+1)}\notag \\
	& \le \left\Vert z^{(L+1)}_i -F^{(L+1)} \left(f^{(L)}(X_i), \left\lBrace\left(f^{(L)}(X_k),W(X_i,X_k)\right)\right\rBrace_{ k\neq i}\right)  \right\Vert_{\infty}\notag \\
	& \quad + \left\Vert F^{(L+1)} \left(f^{(L)}(X_i), \left\lBrace\left(f^{(L)}(X_k),W(X_i,X_k)\right)\right\rBrace_{ k\neq i}\right)  - F_{P,n}^{(L+1)}(f^{(L)},W)(X_i) \right\Vert_{\infty} + s_{n}^{(L+1)}\notag\\
	& \le K_{F,n}^{(L+1)} %
	\max_{i}\left\Vert z^{(L)}_i - f^{(L)}(X_i)\right\Vert_{\infty}^{\alpha_F}\notag \\
    & \quad  + \left\Vert F^{(L+1)} \left(f^{(L)}(X_i), \left\lBrace\left(f^{(L)}(X_k),W(X_i,X_k)\right)\right\rBrace_{ k\neq i}\right) - F_{P,n}^{(L+1)}(f^{(L)},W)(X_i) \right\Vert_{\infty} + s_{n}^{(L+1)}\,, \label[inequality]{eq: md proof 4}
	\end{align}
	where the last~\cref{eq: md proof 4} comes from the Hölder-like regularity~\cref{ass: holder layer} on $F^{(L+1)}$.
    Taking the maximum over the vertices,~\cref{eq: md proof 4} yields
	\begin{align}
	&\max_i \left\Vert z^{(L+1)}_i - f^{(L+1)}(X_i)\right\Vert_{\infty}\notag \\
	& \le K_{F,n}^{(L+1)}\max_i\left\Vert z^{(L)}_i - f^{(L)}(X_i) \right\Vert_\infty^{\alpha_F} %
	\notag \\
	& \quad + \max_i\left\Vert F^{(L+1)} \left(f(X_i), \left\lBrace\left(f(X_k),W(X_i,X_k)\right)\right\rBrace_{ k\neq i}\right)-F_{P,n}^{(L+1)}(f^{(L)},W)(X_i) \right\Vert_{\infty}
	 + s_{n}^{(L+1)}\,.%
     \label[inequality]{eq: md proof 5}
	\end{align}
	Finally, we bound~\cref{eq: md proof 5} from above with high probability. 
    The first term is handled by the induction hypothesis. 
    For the second term, we employ the same technique as we did in the case $L=1$. 
    By conditioning over $X_i$, using the Hölder version McDiarmid's inequality, and a union bound, we obtain with probability at least $1-\rho$,
	\begin{align*}
	&  \max_i \left\Vert z^{(L+1)}_i - f^{(L+1)}(X_i)\right\Vert_{\infty} \\
	& \le %
    K_{F,n}^{(L+1)} H^{(L)}\left(\frac{\rho}{2}\right)^{\alpha_F} + \left(\frac{1}{2}\left(D_{n}^{(L+1)}\right)^2 n \ln\left(\frac{4d_{L+1}n}{\rho}\right)\right)^{\nicefrac{\alpha_\psi}{2}} + \widetilde D_{n}^{(L+1)} + s_{n}^{(L+1)}\\
	& \le %
    K_{F,n}^{(L+1)} \left( \sum_{l=1}^{L}A_n^{(l,L)}\left[\left(\frac{1}{2}\left(D_{n}^{(l)}\right)^2 n \ln\left(\frac{2^{L+3-l} d_l n}{\rho}\right)\right)^{\nicefrac{\alpha_\psi}{2}} + \widetilde D_{n}^{(l)} + s_{n}^{(l)}\right]^{\alpha_F^{L-l}} \right)^{\alpha_F} \\
	&\quad + \left(\frac{1}{2}\left(D_{n}^{(L+1)}\right)^2 n \ln\left(\frac{4d_{L+1}n}{\rho}\right)\right)^{\alpha_\psi/2} + \widetilde D_{n}^{(L+1)} + s_{n}^{(L+1)}\\
	& \le \sum_{l=1}^{L}%
    K_{F,n}^{(L+1)} (A_n^{(l,L)})^{\alpha_F}\left[\left(\frac{1}{2}\left(D_{n}^{(l)}\right)^2 n \ln\left(\frac{2^{L+3-l} d_l n}{\rho}\right)\right)^{\nicefrac{\alpha_\psi}{2}} + \widetilde D_{n}^{(l)} + s_{n}^{(l)}\right]^{\alpha_F^{L+1-l}} \\
    &\quad + \left(\frac{1}{2}\left(D_{n}^{(L+1)}\right)^2 n \ln\left(\frac{4d_{L+1}n}{\rho}\right)\right)^{\nicefrac{\alpha_\psi}{2}} + \overline D_{n}^{(L+1)} + s_{n}^{(L+1)}\\
	& = H^{(L+1)}(\rho).
	\end{align*}
	where we have used that $(x+y)^a \leq x^a + y^a$ for $a\leq 1$, as well as the recursive expression
    \begin{equation*}
        \begin{cases}
            A^{(l,l)}=1 \\
            A^{(l, L+1)} = K_{F,n}^{(L+1)}(A^{(l,L)})^{\alpha_F} \,,
        \end{cases}
    \end{equation*}
    which yields the result.
\end{proof}

We now turn to the invariant case, which follows as a corollary of~\cref{eq: main result bound}.
We will use~\cref{eq: main result bound}, and make additional use of McDiarmid's concentration bound.

\begin{proof}[Proof of the invariant case,~\cref{eq: main result bound inv} of~\cref{th: main result}.]
    Under~\cref{ass: bounded,ass: continuous layer,ass: holder layer,,ass: readout}, we obtain the following inequality.
	\begin{align*}
			& \left\lVert \overline{\Theta}_{G_n}\left(\iota_X(f)\right) - \overline{\Theta}_{W,P}(f) \right\rVert_{\infty} \\
			& \le \left\lVert R\left( \left\lBrace z^{(L)}_1,\dots,z_n^{(L)} \right\rBrace \right) - R\left( \left\lBrace f^{(L)}(X_1),\dots,f^{(L)}(X_n) \right\rBrace \right) \right\rVert_{\infty} \\
			& \quad + \left\lVert R\left( \left\lBrace f^{(L)}(X_1),\dots,f^{(L)}(X_n) \right\rBrace \right) - R_{P,n}\left(f^{(L)}\right)  \right\rVert_{\infty}\\
			& \quad + \left\lVert R_{P,n}(f^{(L)}) - \RRR_P\left(f^{(L)}\right) \right\rVert_{\infty} \\
			& \le K_{R,n} \max_i \left\lVert z^{(L)}_i - f^{(L)}(X_i)\right\rVert_{\infty} \\
			& \quad + \left\lVert R\left( \left\lBrace f^{(L)}(X_1),\dots,f^{(L)}(X_n) \right\rBrace \right) - R_{P,n}\left(f^{(L)}\right)  \right\rVert_{\infty} + r_n 
            \,.
	\end{align*}
    Using a union bound,~\cref{eq: main result bound} and McDiarmid's inequality, we get that, with probability at least $1-\rho$:
	\begin{equation*}
		\begin{split}
			\left\Vert \overline{\Theta}_{G_n}(\iota_X(f)) - \overline{\Theta}_{W,P}(f) \right\Vert_{\infty} & \le K_{R,n}H^{(L)}(\rho/2) %
			+ C_n\sqrt{n\ln\left(\frac{4d_{L}}{\rho} \right)} + r_n \,,
		\end{split}
	\end{equation*}
    which concludes the proof.
\end{proof}

\section{Proof of~\cref{th: main result max}}\label{app: proof cv max}

We prove~\cref{th: main result max}.
Until the end of the proof, we denote by $H^{(L)}(\rho)$ the left-hand side of~\cref{eq: max main result bound},
$$H^{(L)}(\rho) = \sum_{l=1}^{L}B^{(l,L)}\frac{K_{f^{(l)}}}{2}\left (\frac{1}{n \kappa } \ln\left( \frac{2^{L+1-l}nd_l}{\rho}\right) \right)^{\nicefrac{1}{d}} .$$

\begin{proof}[Proof of the equivariant case,~\cref{eq: max main result bound} of~\cref{th: main result max}] 
	Let $\rho>0$. We will prove the theorem by induction on $L$. 
    For $L=1$, let us note $g^{(1)}(x,y) = W(x,y)\m^{(1)}(f^{(0)}(y))$, such that $f^{(1)}(x) = \sup_y g^{(1)}(x,y)$. 
    The map $f^{(1)}$ is $K_{f^{(1)}} = K_{\m^{(1)}} K_{f^{(0)}} + \Vert\m^{(1)}\circ f^{(0)}\Vert_{\infty} K_W$ Lipschitz continuous from~\cref{prop: app Lipschitz of f^l for max}. 

    Fix a node $i\in\{1,\dots,n\}$ and let $x_i\in\X$, by~\cref{lem: concentration bound for max}, for $\rho\ge ne^{-n \kappa r_0^d 2^d}$, with probability at least $1-\rho/n$, we have
	$$ \lVert \max_{j\ne i}g^{(1)}(x,X_j) - \sup_{y\in\X}g^{(1)}(x,y)\rVert_{\infty}  
    \le \frac{K_{f^{(1)}}}{2}\left(\frac{\ln\left( nd_1 / \rho \right) }{n \kappa } \right)^{\nicefrac{1}{d}}\,.$$
	Thus, by conditioning over $X_i$ and using the law of total probability, with probability at least $1-\rho/n$, it holds
	$$ \lVert \max_{j\ne i}g^{(1)}(X_i,X_j) - \sup_{y\in\X}g^{(1)}(X_i,y)\rVert_{\infty}  
    \le \frac{K_{f^{(1)}}}{2}\left(\frac{\ln\left( nd_1 / \rho \right) }{n \kappa } \right)^{\nicefrac{1}{d}}\,.$$	
    Since $\max_{j\ne i}g^{(1)}(X_i,X_j) = z_i^{(1)}$ and $\sup_{y\in\X}g^{(1)}(X_i,y) = f^{(1)}(X_i)$, by maximizing over $i$ and doing a union bound, we obtain that with probability at least $1-\rho$, it holds
	$$\max_i \left\Vert z_i^{(1)} - f^{(1)}(X_i)\right\Vert_{\infty} 
    \le \frac{K_{f^{(1)}}}{2}\left(\frac{\ln\left( nd_1 / \rho \right) }{n \kappa } \right)^{\nicefrac{1}{d}}
    \le H^{(1)}( \rho ) \,,$$	
    for $\rho\ge ne^{-n \kappa r_0^d 2^d}$, and, in particular, for $\rho\ge 2ne^{-n \kappa r_0^d 2^d}$.
	That concludes the case $L=1$. 

    Now suppose the result true for $L\ge 1$, fix a node $i\in\{1,\dots,n\}$, we have the following bound, 
	\begin{align}
			& \left\Vert z^{(L+1)}_i - f^{(L+1)}(X_i)\right\Vert_{\infty}  
			 =  \left\Vert \max_{j\ne i} W(X_i,X_j)\m^{(L+1)}\left(z^{(L)}_j\right) - \sup_{y\in\X} W(X_i,y)\m^{(L+1)}(f^{(L)}(y))\right \Vert_{\infty} \notag \\
			& \le \left\Vert \max_{j\ne i} W(X_i,X_j)\m^{(L+1)}(z^{(L)}_j) - \max_{j\ne i} W(X_i,X_j)\m^{(L+1)}(f^{(L)}(X_j)) \right\Vert_{\infty} \notag \\
			& \quad +  \left\Vert \max_{j\ne i} W(X_i,X_j)\m^{(L+1)}(f^{(L)}(X_j)) -  \sup_{y\in\X} W(X_i,y)\m^{(L+1)}(f^{(L)}(y)) \right\Vert_{\infty} \notag \\
			& \le   K_{\m^{(L+1)}} \max_{j\ne i} \left\Vert z^{(L)}_j - f^{(L)}(X_j)\right\Vert_{\infty} \notag \\
            & \quad +   \left\Vert \max_{j\ne i} W(X_i,X_j)\m^{(L+1)}(f^{(L)}(X_j)) -  \sup_{y\in\X} W(X_i,y)\m^{(L+1)}(f^{(L)}(y)) \right\Vert_{\infty} \label[inequality]{eq: proof bound max 0}
            \,,
	\end{align}
	where the last~\cref{eq: proof bound max 0} uses~\cref{lem: app properties of max}, the fact that $|W| \le 1$, and the Lipschitz continuity of $\m^{(L+1)}$. 
    Thus taking the maximum over $i$, we get 
    \begin{equation}\label[inequality]{eq: proof bound max 1}
		\begin{split}
			& \max_i\left\Vert z^{(L+1)}_i - f^{(L+1)}(X_i)\right\Vert_{\infty} \\
			& \le  K_{\m^{(L+1)}} \max_{ i} \left\Vert z^{(L)}_j - f^{(L)}(X_j)\right\Vert_{\infty}\\
			& \quad +   \max_i\left\Vert \max_{j\ne i} W(X_i,X_j)\m^{(L+1)}(f^{(L)}(X_j)) -  \sup_{y\in\X} W(X_i,y)\m^{(L+1)}(f^{(L)}(y)) \right\Vert_{\infty} 
            \,.
		\end{split}
	\end{equation} 
	Now we bound~\cref{eq: proof bound max 1} with high probability. 
    We use the induction hypothesis for the first term. 
    For the second term, we set $g^{(L+1)}(x,y) = W(x,y)\m^{(L+1)}(f^{(L)}(y))$ and use~\cref{lem: concentration bound for max} on $g$ which is $K_{f^{(L+1)}}=K_{\m^{(L+1)}} K_{f^{(L)}} + \Vert \m^{(L+1)}\circ f^{(L)}\Vert_{\infty} K_W$ Lipschitz. 
    The process is the same as in the case $L=1$. 
    By conditioning over $X_i$ followed by a union bound, we obtain that for $\rho \ge 2e^{-n \kappa r_0^d 2^d}$, with probability at least $1-\rho$, it holds,
	\begin{align*}
		&  \max_i\left\lVert z^{(L+1)}_i - f^{(L+1)}(X_i)\right\rVert_{\infty} \\
		& \le K_{\m^{(L+1)}} H^{(L)}(\rho/2) + \frac{K_{f^{(L+1)}}}{2}\left(\frac{\ln\left( 2nd_{L+1} / \rho\right)}{n\kappa}\right)^{\nicefrac{1}{d}} \\
		& = \sum_{l=1}^{L}K_{\m^{(L+1)}}B^{(l,L)}\frac{K_{f^{(l)}}}{2}\left (\frac{\ln\left( 2^{L+2-l}nd_l / \rho\right)}{n\kappa}\right)^{\nicefrac{1}{d}} 
		+ \frac{K_{f^{(L+1)}}}{2}\left( \frac{\ln\left( 2nd_{L+1} / \rho\right)}{n\kappa}\right)^{\nicefrac{1}{d}} \\
		& = H^{(L+1)}(\rho)\,.
	\end{align*}
    where we have used the recursive expression
    \begin{equation}
        \begin{cases}
            & B^{(l,l)} = 1 \\
            & B^{(l,L+1)} = K_{m^{(L+1)}} B^{(l,L)}
        \end{cases}
    \end{equation}
\end{proof}

We finish with the proof of the invariant case, which follows from the previous result.

\begin{proof}[Proof of the invariant case,~\cref{eq: max main result bound inv} of~\cref{th: main result max}]
    The final readout is simply a maximum (respectively a supremum), over the nodes of the graph, hence a triangular inequality gives,
	\begin{equation*}
		\begin{split}
			&\left\Vert \overline{\Theta}_{G_n}\left(\iota_Xf\right) - \overline{\Theta}_{W,P}(f) \right\Vert_{\infty} = \left\Vert \max_i z_i^{(L)} - \sup f^{(L)} \right\Vert_{\infty} \\
			& \le \left\Vert \max_i z_i^{(L)} - \max_i f^{(L)}(X_i)\right\Vert_{\infty} + \left\Vert \max_i f^{(L)}(X_i) -  \sup f^{(L)} \right\Vert_{\infty} \\
			& \le \max_i\left\Vert z^{(L)}_i - f^{(L)}(X_i)\right\Vert_{\infty} + \left\Vert \max_i f^{(L)}(X_i) -  \sup f^{(L)} \right\Vert_{\infty}.
		\end{split}
	\end{equation*}
	Using the bound for the equivariant case and~\cref{lem: concentration bound for max} on $f^{(L)}$, we obtain the result.
\end{proof}

\section{Examples}\label{sec: app examples proofs}

In this section, we put all computations related to the examples, we prove~\cref{prop: continuous-examples,prop: main th on the examples}. 
For notational convenience, we drop any subscript or superscript $(l)$ referring to layers. 
Recall that $m = m^{(l)}$ is supposed Lipschitz and bounded, we denote $K_{m}$ its Lipschitz constant and $\lVert m\rVert_{\infty}=\sup_x \lVert m(x)\rVert_{\infty}$ (on the respective domain of $x$ at each layer).
We divide this section into subsections for each example.

\subsection{\Cref{ex: 1,ex: a}: Convolutional message-passing with \emph{mean} aggregation}\label{sec: app calculation for example 1}
We prove~\cref{prop: continuous-examples,prop: main th on the examples} on all the examples.
We verify~\cref{ass: bounded,ass: continuous layer,,ass: holder layer} on~\cref{ex: 1,ex: 2,ex: 3,ex: 4}, but not~\cref{ex: 5}.
For the latter, we simply verify that the bounded differences fail to be sharp enough.

\paragraph{\Cref{prop: continuous-examples} for~\cref{ex: 1,ex: a}.}	
By independence and identical distribution of the random variables and linearity of the expected value, the convergence in~\cref{eq: def construction limite part 2} is actually an equality for all integer $n$. 
For all $x$,
	$$ \E\left[\frac{1}{n}\sum_i W(x,X_i)\m(f(X_i))\right]= \E\left[W(x,X_1)\m(f(X_1))\right]=\int_{\X}W(x,y)\m(y) \d P(y)\,.$$
Clearly this remains true when replacing $P$ by $\phi\cdot P$, $f$ by $\phi\cdot f$ and $W$ by $\phi\cdot W$ for any $\phi\in\Aut(\X)$. 

\paragraph{\Cref{ass: continuous layer} for~\cref{ex: 1,ex: a}.}
The above calculation yields $s_n=0$.

\paragraph{\Cref{ass: bounded} for~\cref{ex: 1,ex: a}.}
We consider $\psi=\id$, such that $\widetilde D_n=0$ and $K_\psi=1$.

Let $x_1,\ldots, x_n$ and $x_2',\dots,x_n'$  be such that $x_i=x_i'$ except at $k=2$.
\begin{align*}
		&\left\Vert F\left(f(x_1), \left\lBrace\left(f(x_i),W(x_1,x_i)\right)\right\rBrace_{2\leq i\leq n}\right) - F\left(f(x_1), \left\lBrace\left(f(x_i'),W(x_1,x_i')\right)\right\rBrace_{2\leq i\leq n}\right) \right\Vert_{\infty} \\
		& = \frac{1}{n-1}\left\Vert W(x_1,x_2)\m(f(x_2)) - W(x_1,x_2')\m(f(x_2'))\right\Vert_\infty \\
		& = 2\Vert m\Vert_\infty / (n-1) \\
        & = O(1/n)\,,
\end{align*}
since $\m$ is bounded. 
Hence, $D_n = O(1/n)$.
			
\paragraph{\Cref{ass: holder layer} for~\cref{ex: 1,ex: a}.} 
We have 
\begin{align*}
	&\left\Vert F(z_1,\left\lBrace(z_i, w_i)\right\rBrace_{i\geq 2}) - F(z'_1,\left\lBrace z'_i, w_i\right\rBrace_{i\geq 2})\right\Vert_{\infty} \\
    & \le \left\lVert\frac{1}{n-1}\sum_{2\le i\le n}  w_i \m(z_i) -  \frac{1}{n-1}\sum_{2\le i\le n} w_i \m(z_i') \right\rVert_{\infty} \\
	&  \le K_{\m}\max_i\left\Vert z_i - z_{i}' \right\Vert_\infty\,.
\end{align*}
Hence~\cref{ass: holder layer} is satisfied with $K_{F,n}=K_m$ and $\alpha_F=1$.
			
\subsection{\Cref{ex: 2,ex: b}: Degree normalized message-passing}\label{calculation for example 2}

Recall here that we assume that $W(\cdot, \cdot)\geq a >0$.

\paragraph{\Cref{prop: continuous-examples} for~\cref{ex: 2,ex: b}.}

We have for all $x$,
$$\int_{\X} \frac{W(x,y)\m(f(y))}{\int_{\X}W(x,t) \d P(t)}\d P(y) =\frac{\E\left[ W(x,X)\m(f(X))\right]}{\E\left[ W(x,X)\right]}\,, $$
where $X \sim P$.
Let $x\in \X$ be fixed. 
For simplicity, we compute $s_{n+1}$ (\ie we consider $n$ random variables and not $n-1$). 
Then,
\begin{align}
	& \left\Vert \E_{X_i}\left[\frac{\frac{1}{n}\sum_i W(x,X_i)\m(f(X_i))}{\frac{1}{n}\sum_i W(x,X_i)} \right] - \frac{\E\left[ W(x,X)\m(f(X))\right]}{\E\left[ W(x,X)\right]} \right\Vert_{\infty} \notag \\
	& = \left\Vert \E_{X_i}\left[\frac{\frac{1}{n}\sum_i W(x,X_i)\m(f(X_i))}{\frac{1}{n}\sum_i W(x,X_i)} - \frac{\E\left[ W(x,X)\m(f(X))\right]}{\E\left[ W(x,X)\right]} \right] \right\Vert_{\infty} \notag \\
	& \le \E_{X_i}\Biggl[ \biggl\lVert \frac{\frac{1}{n}\sum_i W(x,X_i)\m(f(X_i))}{\frac{1}{n}\sum_i W(x,X_i)} - \frac{\E\left[ W(x,X)\m(f(X))\right]}{\E\left[ W(x,X)\right]} \biggr\rVert_{\infty}\Biggr]  \notag \\
	& \le \E\Biggl[ \biggl\lVert \frac{1}{n}\sum_i W(x,X_i)\m(f(X_i))\biggr\rVert_\infty \biggl\lvert\frac{1}{\frac{1}{n}\sum_i W(x,X_i)} -\frac{1}{\E\left[ W(x,X)\right]}\biggr\rvert \notag \\
	& \qquad + \biggl\lVert \frac{1}{n}\sum_i W(x,X_i)\m(f(X_i)) - \E\left[ W(x,X)\m(f(X))\right]\biggr\rVert_\infty \biggr\lvert\frac{1}{\E\left[ W(x,X)\right]}\biggr\rvert \Biggr] \notag \\
	& \le \frac{\lVert \m\rVert_{\infty}}{a^2}
    \E\left[ \biggl\lvert\frac{1}{n}\sum_i W(x,X_i) -\E\left[ W(x,X)\right]\biggr\rvert \right]  \notag \\
	& \qquad + \frac{1}{a}\E\left[ \biggl\lVert \frac{1}{n}\sum_i W(x,X_i)\m(f(X_i)) - \E\left[ W(x,X)\m(f(X))\right]\biggr\rVert_\infty \right] \,, \label{eq: proof ex 2-b 0}
\end{align}
since $0 < a \le W \le 1$.
Consequently, using the formula $\E(X)=\int_{t>0}\P(X>t) \d t$ for X nonnegative, we get that this last quantity~\labelcref{eq: proof ex 2-b 0} is equal to
\begin{multline}\label{eq: proof ex 2-b 1}
		\frac{\left\Vert \m\right\Vert_{\infty}}{a^2} \int_{t>0} \P\left( \biggl\lvert \E\left[ W(x,X)\right]  -\frac{1}{n}\sum_i W(x,X_i)   \biggr\rvert > t \right) \d t \\
		+ \frac{1}{a}\int_{t>0} \P\left( \biggl\lVert  \frac{1}{n}\sum_i W(x,X_i)\m(f(X_i)) -   \E\left[ W(x,X)\m(f(X))\right] \biggr\rVert_{\infty} > t\right)  \d t \,.
\end{multline}

Finally, we use McDiarmid inequality (which turns out to be the same as Hoeffding inequality for a sum of independent random variables). 
It is easy to check that the concerned multivariate maps have bounded differences of the form $c_i=C/n$ for all $i$. 
Therefore, there are some positive constants $C_1, C_2$ \emph{independent of $x$} such that~\cref{eq: proof ex 2-b 1} is bounded by 
$$C_1\int_{t>0} e^{-n C_2 t^2}  \d t %
= O\left(1/\sqrt{n}\right)\to 0.$$
Since this is true for all $x$, and that this bound is independent of $x$, we obtain convergence in $L^\infty_P$ norm. 
This remains true when replacing $P$ by $\phi\cdot P$, $f$ by $\phi\cdot f$ and $W$ by $\phi\cdot W$ for any $\phi\in\Aut(\X)$.
Hence we have shown~\cref{prop: continuous-examples}. 
			
\paragraph{\Cref{ass: continuous layer} for~\cref{ex: 2,ex: b}.}
The above calculation yields $s_n= O\left(1/\sqrt{n}\right)$.

\paragraph{\Cref{ass: bounded} for~\cref{ex: 2,ex: b}.} 
Let us first introduce some intermediate computations. 
Let $x_1,\ldots, x_n$ and $x_2',\dots,x_n'$ be such that $x_i=x_i'$ except at $i=2$. 
Denote by $z_i=f(x_i)$, $z'_i=f(x_i)$, $w_i=W(x_1,x_i)$, and $w'_i=W(x'_1,x'_i)$ for short. 
\begin{align}
	&\left\Vert F\left(f(x_1), \left\lBrace\left(f(x_i),W(x_1,x_i)\right)\right\rBrace_{2\leq i\leq n}\right) - F\left(f(x_1), \left\lBrace\left(f(x_i'),W(x_1,x_i')\right)\right\rBrace_{2\leq i\leq n}\right) \right\Vert_{\infty} \notag\\
    & = \norm*{F\left( z_1, \multiSet{z_i,w_i} \right) - F\left( z_1, \multiSet{z_i',w_i'} \right) }_{\infty} \notag \\
	& = \left\Vert \frac{\sum_{2\le i\le n} w_i \m(z_i)}{\sum_i w_i} - \frac{\sum_{2\le i\le n} w_{i}' \m(z'_i)}{ \sum_i w_{i}'} \right\Vert_{\infty} \notag\\
	& \le \frac{1}{\sum_i w_i}\left\Vert \sum_{i} w_i \m(z_i) - \sum_{i} w_{i}' \m(z'_i) \right\Vert_{\infty} + \left\Vert \sum_{i} w_{i}' \m(z'_i) \right\Vert_{\infty} \left|\frac{1}{\sum_i w_i} - \frac{1}{\sum_i w'_i}\right| \notag\\
    &\leq \frac{1}{(n-1) a}\left\Vert \sum_{i} w_i \m(z_i) - \sum_{i} w_{i}' \m(z'_i) \right\Vert_{\infty} + \frac{\|m\|_\infty}{(n-1) a^2}\left|\sum_i w_i - \sum_i w'_i\right| \,. \label[inequality]{eq: proof 2-b inter}
\end{align}
Then, again we consider $\psi=\id$, such that $\widetilde D_n=0$ and $K_\psi=1$. 
Applying~\cref{eq: proof 2-b inter} with $z_i=z'_i$ and $w_i=w'_i$ except for $i=2$, we prove~\cref{ass: bounded} with 
$$
D_n = \frac{1}{(n-1) a} + \frac{\|m\|_\infty}{(n-1) a^2} = O\left(1/n\right)\,.
$$

\paragraph{\Cref{ass: holder layer} for~\cref{ex: 2,ex: b}.}
Again using~\cref{eq: proof 2-b inter} with $w_i=w'_i$ and using the Lipschitz property of $m$,
\begin{align*}
    & \norm*{F\left( z_1, \multiSet{z_i,w_i} \right) - F\left( z_1', \multiSet{z_i',w_i} \right) }_{\infty} \notag \\
	&\le \frac{1}{(n-1) a}\left\Vert \sum_{i} w_i \m(z_i) - \sum_{i} w_{i} \m(z'_i) \right\Vert_{\infty} \\
	&\le \frac{K_m}{a} \max_i \left\Vert z_i - z'_{i} \right\Vert_{\infty}\,.
\end{align*}
Hence we prove~\cref{ass: holder layer} with $K_F = \frac{K_m}{a}$ and $\alpha_F=1$.
			
\subsection{\Cref{ex: 3,ex: c}: Attentional message-passing}\label{calculation for example 3}

Call $V(x,y) = c\left(f(x),f(y),W(x,y)\right)$. 
We are basically brought to the previous example with $V$ instead of $W$.

\paragraph{\Cref{prop: continuous-examples} for~\cref{ex: 3,ex: c}.} 
Using $V$ instead of $W$, we are brought to the previous example and therefore~\cref{prop: continuous-examples} is satisfied.

\paragraph{\Cref{ass: continuous layer} for~\cref{ex: 3,ex: c}.}
By the same argument,~\cref{ass: continuous layer}, is fulfilled with $s_n=O\left(1/\sqrt{n}\right)$.

\paragraph{\Cref{ass: bounded} for~\cref{ex: 3,ex: c}.} 
From the previous Example,~equation \cref{eq: proof 2-b inter} remains valid with $w_i = V(x_1,x_i), w'_i = V(x'_1,x'_i)$ instead of $W$. 
Hence~\cref{ass: bounded} is proven with $D_n = O\left(1/\sqrt{n}\right)$ and $\psi=\id$.

\paragraph{\Cref{ass: holder layer} for~\cref{ex: 3,ex: c}.} 
Again $\psi=\id$. 
Using again~\cref{eq: proof 2-b inter} but with $v_i = c(z_1,z_i,w_i)$, where $c(x,y,w)\le b$, and $\vert c(x,y,w) - c(x',y',w)\vert \le K_c(\Vert x-x'\Vert_{\infty} + \Vert y-y'\Vert_{\infty})$, we get:
\begin{align*}
	&\left\Vert F\left(z_1, \left\lBrace\left(z_i,w_i\right)\right\rBrace_{2\leq k\leq n}\right) - F\left(z'_1, \left\lBrace\left(z_i',w_i\right)\right\rBrace_{2\leq k\leq n}\right) \right\Vert_{\infty} \notag\\
	& \leq \frac{b}{(n-1) a}\left\Vert \sum_{i} v_i \m(z_i) - \sum_{i} v'_{i} \m(z'_i) \right\Vert_{\infty} + \frac{b\|m\|_\infty}{(n-1) a^2}\left|\sum_i v_i - \sum_i v'_i\right| \\
	&\leq \frac{b}{a} \max_i \left\Vert v_i \m(z_i) - v'_{i} \m(z'_{i}) \right\Vert_{\infty} + \frac{b\|m\|_\infty}{a^2} \max_i \left|v_i - v'_{i}\right|  \\
	&\leq \frac{b}{a} \left(\|m\|_\infty K_c\left(\|z_1-z'_1\|_\infty + \max_i\|z_i-z'_{i}\|_\infty\right) + bK_m \max_i \|z_i-z'_{i}\|_\infty\right) \\
	&+ \frac{b\|m\|_\infty}{a^2} K_c\left(\|z_1-z'_1\|_\infty + \max_i\|z_i-z'_{i}\|_\infty\right) \\
    & \le \left(\frac{2b \norm{n}_{\infty}K_c + b^2K_m}{a} + \frac{2bmK_c}{a^2}\right) \max_i \norm{z_i-z_i'}_\infty \,.
\end{align*}
Hence we have shown~\cref{ass: holder layer} with $\alpha_F=1$.

\subsection{\Cref{ex: 4,ex: d}: Generalized mean}\label{calculation for example 4}

\paragraph{\Cref{prop: continuous-examples} for~\cref{ex: 4,ex: d}.} 
We call $g(x,y)=h^{-1}(W(x,y)\m(f(y)))$ and again consider $n$ neighbors (instead of $n-1$). 
For any fixed $x \in \X$, we have 
\begin{align*}
	&\left\Vert \E_{X_i} \left[ h\biggl( \frac{1}{n} \sum_i  g(x,X_i) \biggr) \right] - h\bigl( \E_X \left[  g(x,X) \right] \bigl) \right\Vert_\infty \\
	&\leq \E_{X_i}\Biggl[\biggl\lVert h\biggl( \frac{1}{n} \sum_i  g(x,X_i) \biggr) - h\bigl( \E_X \left[ g(x,X) \right]\bigr)  \biggr\rVert_\infty \Biggr] \\
	&\leq K_h\E_{X_i}\left[\left\Vert \frac{1}{n} \sum_i g(x,X_i) - \E_X  g(x,X) \right\Vert_\infty^{\alpha_h} \right]\,.
\end{align*}
Again, this is equal to 
$$
K_h\int_t \P\left( \left\Vert  \frac{1}{n}\sum_i g(x,X_i) -   \E\left[ g(x,X)\right] \right\Vert_{\infty}^{\alpha_h} > t\right)  \d t \,,
$$
and since $h^{-1}$, therefore $g$, is bounded we can use Hoeffding inequality to show that this is bounded for some constants $C_1, C_2$ by
$$
C_1\int_{t>0} e^{-n C_2 t^{2/\alpha_h}}  \d t %
	= O(n^{-\nicefrac{\alpha_h}{2}})\to 0 \,.
$$
Since this is true for all $x$ and this bound is independent of $x$, we obtain convergence in $L^\infty_P$ norm. 
This remains true when replacing $P$ by $\phi\cdot P$, $f$ by $\phi\cdot f$ and $W$ by $\phi\cdot W$ for any $\phi\in\Aut(\X)$.
Hence we have shown~\cref{prop: continuous-examples}. %

\paragraph{\Cref{ass: continuous layer} for~\cref{ex: 4,ex: d}.}
The above calculation yields $s_n = O\left(n^{-\frac{\alpha_h}{2}}\right)$.

\paragraph{\Cref{ass: bounded} for~\cref{ex: 4,ex: d}.} 
Here we consider $\psi = h$. 
For~\cref{item 1 assumtion bounded} of~\cref{ass: bounded}, we have 
\begin{equation*}
	\psi^{-1}\left(F(f(x_1), \left\lBrace(f(x_i),W(x_1,x_i)) \right\rBrace_{i\ge 2})\right) = \frac{1}{n-1}\sum_{i\geq 2} h^{-1}(W(x_1,x_i)\m(f(x_i)) \,.
\end{equation*}
So we are brought back to~\cref{ex: 1} when $h^{-1}$ is bounded, and therefore $D_n = O(1/n)$.

For~\cref{item 2 assumtion bounded}, calling again $g(x,y)=h^{-1}(W(x,y), \m(f(y)))$, we have
\begin{align*}
	&\Big\Vert \psi\left(\E \psi^{-1} \left( F\left(f(x_1), \left\lBrace( f(X_i),W(x_1,X_i)) \right\rBrace_{i\ge 2} \right)\right)\right) - \E F(f(x_1), \left\lBrace(f(X_i),W(x_1,X_i)) \right\rBrace_{i\ge 2} )\Big\Vert_\infty \\
	&= \left\Vert h\left( \E \frac{1}{n-1} \sum_i  g(x,X_i) \right) - \E \left[ h\biggl( \frac{1}{n-1} \sum_i  g(x,X_i) \biggr) \right] \right\Vert_\infty\\
	&= \left\Vert h\left( \E_X \left[  g(x,X) \right] \right) - \E_{X_i} \left[ h\biggl( \frac{1}{n-1} \sum_i  g(x,X_i) \biggr) \right] \right\Vert_\infty \\
    & = O\left(n^{-\nicefrac{\alpha_h}{2}}\right)\,,
\end{align*}
as per the computation above for~\cref{prop: continuous-examples}. 
Hence $\widetilde D_n = O\left(n^{-\nicefrac{\alpha_h}{2}}\right)$.

Finally, $\psi=h$ is $\alpha_h$-Holder so~\cref{item 3 assumtion bounded} from~\cref{ass: bounded} is satisfied with $\alpha_\psi = \alpha_h$.

\paragraph{\Cref{ass: holder layer} for~\cref{ex: 4,ex: d}.} 
This is to be treated in a case-by-case manner. 
For the examples mentioned in~\cref{prop: main th on the examples}.
\begin{description}
    \item[Power mean.] 
        Consider $h=x \mapsto x^{\nicefrac{1}{p}}$, and $m$ nonnegative of dimension $1$ (the reasoning can be done in each dimension).
        Define vectors $S, S'$ of size $n-1$ such that $S_i = w_{i}m(f(x_i))$ and similarly for $S'$. 
        Denoting by $\lVert S\rVert_p = \left(\sum_i S_i^p\right)^{\nicefrac{1}{p}}$ the $p$-norm, we get
        \begin{align*}
            &\left\vert F\left(f(x_1), \left\lBrace\left(f(x_i),W(x_1,x_i)\right)\right\rBrace_{2\leq i\leq n}\right) - F\left(f(x_1), \left\lBrace\left(f(x_i'),W(x_1,x_i')\right)\right\rBrace_{2\leq i\leq n}\right) \right\vert \\
            &= \frac{1}{(n-1)^{\nicefrac{1}{p}}}\left\vert \|S\|_p - \|S'\|_p \right\vert \leq \frac{1}{(n-1)^{\nicefrac{1}{p}}}\|S - S'\|_p \leq \|S-S'\|_\infty \\
            &\leq K_m \max_i \|z_i-z'_i\|_\infty
        \,.
	\end{align*}
	Hence $\alpha_F=1$.
    \item[Lipschitz $h^{-1}$.] 
        We have
        \begin{align*}
            &\left\Vert F\left(f(x_1), \left\lBrace\left(f(x_i),W(x_1,x_i)\right)\right\rBrace_{2\leq i\leq n}\right) - F\left(f(x_1), \left\lBrace\left(f(x_i'),W(x_1,x_i')\right)\right\rBrace_{2\leq i\leq n}\right) \right\Vert_\infty \\
            &\leq K_h\left(\frac{1}{n-1}\left\Vert \sum_i h^{-1}(w_i m(z_i)) - h^{-1}(w_i m(z'_i)) \right\Vert_\infty\right)^{\alpha_h} \\
            &\leq K_h (K_m\max_i \|z_i-z'_i\|_\infty)^{\alpha_h}
            \,,
        \end{align*}
        since $h^{-1}$ is Lipschitz.
    \item[Lipschitz $h$ and $h^{-1}$] 
        It is the previous case with $\alpha_h=1$.
\end{description}

\subsection{\Cref{ex: 5,ex: e}: Convolutional Message-Passing with \emph{max} aggregation}\label{calculation for example 5}
For this example, we need to prove~\cref{prop: continuous-examples} and that the Mcdiarmid's method fails, \ie that the bounded differences do not tend to zero.

\paragraph{\Cref{prop: continuous-examples} for~\cref{ex: 5,ex: e}.}
We call $g(x,y)=W(x,y)\m(f(y))$. 
We start by the case when $g$ is real-valued, since $g$ is continuous and $P$ is strictly positive, $\esssup_{P} g(x,\cdot) = \sup g(x,\cdot)$ for all $x$ by~\cref{lem: esssup = sup}. 
Let $\eps>0$ and $x \in \X$. 
By definition of the supremum and by independence of the $X_i$, we have that
\begin{align}\label{eq: proof 4-d 1}
        & \P(| \max_i g(x,X_i) - \sup g(x,\cdot) | \ge \eps) \notag \\
        &= \P( \max_i g(x,X_i) \le \sup g(x,\cdot) - \eps) \notag \\
        &= \P( g(x,X_1)  \le \sup g(x,\cdot) - \eps)^n \notag \\
        &= \P(|g(x,X_1) - \sup g(x,\cdot)|  \ge \eps)^n\,.
\end{align}
By continuity and compactness, there is $x^*\in\X$ such that $\sup g(x,\cdot) = g(x,x^*)$, so~\cref{eq: proof 4-d 1} is equal to
\begin{equation}\label{eq: proof 4-d 2}
    \begin{split}
        & \P(|g(x,X_1) -  g(x,x^*)|  \ge \eps)^n \\
        & = \bigl(1 - \P(|g(x,X_1) -  g(x,x^*)|  < \eps) \bigr)^n\,.
    \end{split}
\end{equation} 
By continuity and compactness again, g is uniformly continuous so there is $\delta>0$ such that $\Vert (x,X_1) - (x,x^*)\Vert = \Vert X_1-x^*\Vert < \delta$ implies $| g(x,X_1)-g(x,x^*))| < \eps$. 
Thus,~\cref{eq: proof 4-d 2} is bounded from above by
	\begin{equation}\label{eq: proof 4-d 3}
		\left(1 - \P(\Vert X_1-x^*\Vert < \delta) \right)^n = \bigl(1 - P(B(x^*,\delta)\cap\X) \bigr)^n \,,
	\end{equation}
where $B(x^*,\delta)$ is the open ball of center $x^*$ and radius $\delta$ in $\R^d$. 
To finish let us justify that the measure of the $B(x^*,\delta)\cap\X$ when $x$ runs over $\X$ is bounded from below. 
Suppose this would not be the case, \ie that the measure of a ball of radius $\delta$ centered in $\X$ could be arbitrarily small. 
By compactness, up to sub-sequence extraction, we can assume there is $(x_k) \in \X^{\N}$ such that $x_n\to x \in \X$ and $P(B(x_k,\delta)\cap\X) \le 1/2^k$. 
Call $U=B(x,\delta/2)\cap \X$, there is rank $k_0$ such that $\forall k\ge k_0$, $x_k\in U$. 
Thus $U\subset B(x_k,\delta)\cap\X\ \forall k\ge k_0$ yielding $P(U) \le 1/2^k\ \forall k \ge k_0$, that is, $P(U)=0$. 
Impossible since $U$ is a nonempty relative open set of $\X$.\\
So there is $\eta > 0$ independent of $x$ such that $P(B(x^*,\delta)\cap\X)>\eta$ and, coming back to~\cref{eq: proof 4-d 3}:
	\begin{equation}\label{eq: proof 4-d 4}
		\P(| \max_i g(x,X_i) - \sup g(x,\cdot) | \ge \eps) \le (1-\eta)^n\,.
	\end{equation} 

If $g$ is vector-valued, say in $\R^{q}$, call $g_1,\dots,g_{q}$ its components and $\eta_k$ such that $g_k$ satisfies~\cref{eq: proof 4-d 4} with $\eta=\eta_k$. 
Then by a union bound we have 
\begin{equation}\label{eq: proof 4-d 5}
    \P(\Vert \max_i g(x,X_i) - \sup g(x,\cdot) \Vert_{\infty} \ge \eps) \le \sum_{k=1}^{q}(1-\eta_k)^n\,.
\end{equation}
At the end of the day, by letting $Z=\Vert \max_i g(x,X_i) - \sup g(x,\cdot) \Vert_{\infty}$, we have for any $\eps>0$
\begin{equation}
    \begin{split}
        \Vert \E(\max_i g(x,X_i)) - \sup g(x,\cdot) \Vert_{\infty} &\le  \E(Z) \\
        &= \E(Z\ind{Z\ge \eps}) + \E(Z\ind{Z <\eps})\\
        &\le 2\Vert g\Vert_{\infty} \sum_{k=1}^{q}(1-\eta_k)^n + \eps.
    \end{split}
\end{equation}
Again, since the right-hand side does not depend on $x$, this concludes the uniform convergence.
To conclude the proof, we are left to check that the strict positiveness of $P$ as well as the continuity of $f$ and $W$ are preserved by the action of homeomorphisms. 
It is clear for maps' continuity. Let $\phi\in\Hom(\X)$ and $U\subset \X$ a relative nonempty open of $\X$,
$$\phi\cdot P (U) = P(\phi^{-1}(U)) >0$$
since $\phi^{-1}(U)$ is a nonempty open of $\X$ as $\phi$ is continuous.

\paragraph{Bounded differences are $ \Omega(1)$.} 
Here we check the bounded differences are $ \Omega(1)$, \ie they do not tend to zero.

Call $g(x,y) = W(x,y)f^{(l-1)}(y)$, and $(g_1,\dots,g_{d_l})$ its components which are real functions. We suppose $g$ not constant, so there is $k$ such that $g_k$ is not constant, say $k=1$.
By compactness and continuity of $g_1$ there is $x^*$ such that $g(x,x^*) = sup_y g(x,y)$. 
Since $g_1$ is not constant, for any $n$, there exist $x_1, \dots,x_n$ such that $g(x,x_1), \dots, g(x,x_n)$ are all strictly smaller that $ g(x,x^*)$. 
Up to reordering them, we suppose $g(x,x_1)=\max_{2\le i\le n}g(x,x_i)$ and call $a = | g_1(x,x^*) - g_1(x,x_1) | >0 $.
\begin{align*}
        & \Vert \max \{g(x,x^*),g(x,x_2),\dots,g(x,x_n) \} - \max \{g(x,x_1), g(x,x_2),\dots,g(x,x_n) \} \Vert_{\infty}\\
        & \ge |  \max \{g_1(x,x^*),g_1(x,x_2),\dots,g_1(x,x_n) \} - \max \{g_1(x,x_1), g_1(x,x_2),\dots,g_1(x,x_n) \} | \\
        & = | g_1(x,x^*) - g_1(x,x_1) | \\
        & > a\,.
\end{align*}
Overall, for any $n$,
$$a < \sup \Vert \max \{g(x,x_1),\dots,g(x,x_n) \} - \max \{g(x,x_1'),\dots,\dots,g(x,x_n') \} \Vert_{\infty}
\,,$$
where the supremum is taken over $x,x_2,\dots,x_n,x_2',\dots,x_n' \in \X$ such that $(x_2,\dots,x_n)$ and $(x_2',\dots,x_n')$ differ from only one component. 
This proves that the bounded differences are $ \Omega(1)$. 
			
\section{Useful results}\label{app: Useful results}
In this section, we gather useful theorems, and other third-party lemmas.
\begin{theorem}[McDiarmid inequality~\citep{boucheron_concentration_2013}]\label{th: McDiarmid} 
    Suppose $\mathcal{E}$ is a probability space and let $f: \mathcal{E}^n \to \R$ be a function of $n$ variables. 
    Suppose that $f$ satisfies the bounded differences property with the $n$ nonnegative constants $c_1, \dots, c_n$. 
    Then for any independent random variables $X_1,\dots,X_n$ in $\mathcal{E}$, for any $\eps>0$:
	$$\P(|f(X_1,\dots,X_n)-\E(f(X_1,\dots,X_n))|>\eps) \leq 2e^{-\frac{2\eps^2}{\sum_{i=1}^{n}c_i^2}}\,.$$
\end{theorem}

Notice that the $X_i$ are not required to be identically distributed. 
By a union bound and reformulating~\cref{th: McDiarmid} to a bound with high probability one can obtain the following result for function taking multidimensional values.

\begin{corollary}[Multi-dimensional McDiarmid inequality]\label{th: Multi dimendional McD}
	Suppose that $f: \mathcal{E}^n \to \R^d$ satisfies a vectorial version on the bounded difference: $\Vert f(x)-f(x')\Vert_{\infty} \leq c_i$ whenever $x$ and $x'$ differ only from the $i$th component. 
    Then for any independent random variables $X_1,\dots,X_n$ in $\mathcal{E}$, for any $\rho>0$ 
	$$\Vert f(X_1,\dots,X_n)-\E(f(X_1,\dots,X_n))\Vert_{\infty}\le \sqrt{\frac{1}{2}\sum_{i=1}^{n}c_i^2 \ln\left(\frac{2d}{\rho}\right)}\,,$$
	holds with probability at least $1-\rho$. 
\end{corollary}

\begin{corollary}[McDiarmid inequality with Holder transform]\label{th: Multi dimendional McD Holder}
	Let $f: \mathcal{E}^n \to \CC \subset \R^d$ and $\psi:\CC \to \mathcal{D}$ an invertible function which is $\alpha$-Hölder,
	$$
	\lVert \psi(y) - \psi(y')\rVert_\infty \leq K_\psi \lVert y-y'\rVert_{\infty}^\alpha\,,
	$$
    and such that,
	$$
		\lVert \psi^{-1}(f(x))-\psi^{-1}(f(x')) \rVert_{\infty} \leq c_i
	$$
	whenever $x$ and $x'$ differ only from the $i$th component, and
    $$
	\lVert \psi(\E \psi^{-1}(f(X_1,\dots,X_n)))-\E(f(X_1,\dots,X_n)) \rVert_{\infty} = u_n \to 0
	$$ 
	for independent random variables $X_1,\dots,X_n$ in $\mathcal{E}$. Then, for any $\rho>0$: 
	$$
	\lVert f(X_1,\dots,X_n)-\E(f(X_1,\dots,X_n))\rVert_{\infty}\le K_\psi\left( \frac{1}{2}\sum_{i=1}^{n}c_i^2 \ln\left(\frac{2d}{\rho}\right) \right)^{\nicefrac{\alpha}{2}} + u_n \,.
	$$
	holds with probability at least $1-\rho$. 
\end{corollary}
\begin{proof}
	We write $f(X_1,\dots,X_n) = f(X)$ and:
	\begin{align*}
		\lVert f(X)-\E(f(X))\rVert_{\infty} &\leq \lVert \psi(\psi^{-1}(f(X)))-\psi(\E \psi^{-1}(f(X)))\rVert_{\infty} \\
		&\quad + \lVert \psi(\E \psi^{-1}(f(X)))-\psi(\psi^{-1}(\E f(X)))\rVert_{\infty} \\
		&\leq K_\psi  \lVert \psi^{-1}(f(X))-\E \psi^{-1}(f(X))\rVert_{\infty}^\alpha + \lVert \psi \left( \E \psi^{-1}(f(X))\right)-\E f(X) \rVert_{\infty}
	\end{align*}
We apply McDiarmid's inequality (\cref{th: Multi dimendional McD}) on the first term and bound the second by $u_n$ to obtain the result.
\end{proof}

\begin{lemma}\label{lem: esssup = sup}
	Suppose $P$ is strictly positive \ie for all $U\subset\R^d$, $P(U\cap\X)>0$ if and only if $U\cup\X$ is nonempty. Then for any continuous map $f: \X \to \R$,
	$$ \uesssup{P} f = \sup f < +\infty\,.$$ 
\end{lemma}
\begin{proof}
	Clearly $\esssup_{P} f \le \sup f$ and $\sup f < +\infty$ by continuity and compactness. 
    Suppose that $\esssup_{P} f < \sup f$ then there is $M$ such that $\esssup_{P} f < M < \sup f$. 
    By definition of $\sup f$, the set $(f>M) = f^{-1}(]M;+\infty[)$ is nonempty, it is also a relative open of $\X$ since it is the inverse image of an open by a continuous map. 
    Thus, this set has a strictly positive measure, which yields a contradiction with the fact that $\esssup_{P} f < M$.
\end{proof}

\begin{lemma}\label{lem: app properties of max}
	Let $(a_i)_{i\in I}$ and $(b_i)_{i\in I}$ be two finite families of vectors in $\R^m$. 
    We have the following properties:
	\begin{enumerate}[label=(\roman*)]
		\item\label{item: lemma max 1} $\Vert \max_i a_i \Vert_{\infty} \le \max_i \Vert a_i\Vert_{\infty}$.
		\item\label{item: lemma max 2} $\Vert \max_i a_i - \max_i b_i \Vert_{\infty} \le  \max_i \Vert a_i - b_i \Vert_{\infty}$.
	\end{enumerate}
    Where the maximums are the be understood component-wise.
\end{lemma}
\begin{proof}
	We start with~\cref{item: lemma max 1}. 
    Suppose $m=1$, trivially, $a_i \le |a_i|$ which implies ${\max_i a_i \le \max_i |a_i|}$, and therefore
    \begin{equation}\label[inequality]{eq: lemma max proof 1}
    |\max_i a_i| \le |\max_i |a_i| | = \max_i |a_i|\,.
    \end{equation}
    For $m\ge 1$, denote $a_i^{(k)}$ (resp. $b_i^{(k)}$ ) the $k$th coordinate of $a_i$ (reap. $b_i$ ) for $1\le k\le m$.
    Using~\cref{eq: lemma max proof 1}, we easily verify that
    \[
    \Vert \max_i a_i \Vert_{\infty} = \max_k |\max_i a_i^{(k)}| \le \max_k \max_i |a^{(k)}_i| = \max_i \max_k|a^{(k)}_i| = \max_i \Vert a_i \Vert_{\infty}\,.
    \]
    
	We turn to~\cref{item: lemma max 2}. 
    We start with the case $m=1$. 
    Let $i_a$ (resp. $i_b$) be an index that realizes $\max_i a_i$ (resp. $\max_i b_i$), we have
    \[
        \max_i a_i - \max_i b_i = a_{i_a} - b_{i_b} = a_{i_a} - b_{i_a} + b_{i_a} - b_{i_b} \,.
    \]
    However, by definition of $i_b$, the last member $b_{i_a} - b_{i_b} \le 0$, which yields 
    \[
    \max_i a_i - \max_i b_i \le a_{i_a} - b_{i_a} \le \max_i a_i - b_i \le \vert \max_i a_i - b_i \vert \le \max_i \vert a_i - b_i\vert \,,
    \]
    where the last inequality follows from~\cref{eq: lemma max proof 1}.
	Then, the same calculation will yield 
    \[
    \max_i b_i - \max_i a_i \le \max_i \vert b_i - a_i\vert = \max_i \vert a_i - b_i\vert \,,
    \]
    which proves the desired result.
	For $m\ge 1$, simply apply the previous result to each coordinate,
    \begin{align*}
        \Vert \max_i a_i - \max_i b_i \Vert_{\infty} & = \max_k | \max_i a_i^{(k)} - \max_i b_i^{(k)} | \\
                                                     & \le \max_k  \max_i | a_i^{(k)} - b_i^{(k)} | \\ 
                                                     & \le \max_i  \max_k | a_i^{(k)} - b_i^{(k)} | \\ 
                                                     & =  \max_i \Vert a_i - b_i \Vert_{\infty} \,.
    \end{align*}
\end{proof}

\begin{lemma}\label{lem: Lipschitz of sup}
	Let $\X$ be compact and $g: \X \times \X \to \R^m$ be $K_g$-Lipschitz continuous. 
    Then, $f$ defined by
    $$
    \declareFonction{f}{\X}{\R^m}{x}{\sup_{y\in \X} g(x,y)}\,.
    $$ 
    where the supremum is taken component-wise, is also $K_g$-Lipschitz continuous.
\end{lemma}
\begin{proof}
	Start with the case $m=1$. 
    Let $x,x'\in \X$, by continuity and compactness, the supremums defining $f(x)$ and $f(x')$ are reached, \ie there exist $x^*,x'^*$ such that ${f(x) = g(x,x^*)}$ and $f(x')=g(x',x'^*)$. 
    Then,
    \[
    f(x) - f(x') = g(x,x^*)-g(x',x^*) + g(x',x^*) - g(x',x'^*) \,,
    \]
    However, by definition of the supremum, $g(x',x^*) - g(x',x'^*) \le 0$, which yields
    $$
    f(x) - f(x') \le  g(x,x^*)-g(x',x^*)\le K_g \Vert x-x'\Vert_{\infty}
    \,,$$ 
    by Lipschitz continuity of $g$ on $\X^2$.
    Permuting $x$ and $x'$ and repeating the same computation will yield to
    $$
    f(x') - f(x) \le K_g \Vert x-x'\Vert_{\infty}
    \,,$$ 
    such that we obtain the desired Lipschitz condition on $f$.

	For $m\ge 1$, denote $g_i$ the components of $g$ for $1\le i\le m$.
    It is immediate that, since $g$ is $K_g$-Lipschitz, each real-valued $g_i$ is $K_g$-Lipschitz too.
    Therefore, according to the case $m=1$, each $f_i\colon x\mapsto \sup_y g_i(x,y)$ is also $K_g$-Lipschitz,.
    We conclude that for any $x,x' \in \X$,
    \[
    \Vert f(x) - f(x') \Vert_{\infty} = \max_i |f_i(x) - f_i(x') |  \le  K_g \Vert x-x' \Vert_{\infty} \,.
    \]
\end{proof}

\begin{lemma}\label{lem: borel-cantelli}
Let $(X_n)_n$ be a sequence of random variables.
If the series of general term $\P\left( \abs{X_n} > \eps \right)$ is summable for all $\eps >0$, then $X_n$ tends to zero almost surely.
\end{lemma} 

\begin{proof}
    Let $\Omega$ be an abstract probability space which implicitly defines the random variables $X_n$.
    In the language of probabilities, the assertion ``$X_n$ tends to zero'' is equivalent to, 
    $$\forall k \geq 1,\ \forall \omega \in \Omega,\ \exists N(\omega) \mid \forall n\geq N(\omega),\  |X_n(\omega)| \leq \frac{1}{k} \,.$$
    This translates into the set-theoretic language of events by
    $$\left( X_n \rightarrow 0 \right)  = \bigcap_{k\geq 1} \bigcup_{N\geq 0} \bigcap_{n\geq N} \left( |X_n| \leq \frac{1}{k} \right) \,.$$
    We claim that this event is almost sure.
    Indeed,
    $$\P\left(X_n \rightarrow 0\right) = 1 - \P\left(\bigcup_{k\geq 1} \bigcap_{N\geq 0} \bigcup_{n\geq N} \left( |X_n| > \frac{1}{k} \right) \right) 
    = 1 - \P\left(\bigcup_{k\geq 1} \left( \varlimsup|X_n|>\frac{1}{k} \right) \right)\,. $$
    However, by the Borel-Cantelli lemma, $\P\left(\varlimsup|X_n|>\epsilon\right) = 0$ for all $\eps$, therefore, for $\eps = 1 / k$ in particular.
    Finally, since a countable union of null sets is null, we obtain the result.
\end{proof}

\bibliography{biblio}

\begin{thebibliography}{58}
\providecommand{\natexlab}[1]{#1}
\providecommand{\url}[1]{\texttt{#1}}
\expandafter\ifx\csname urlstyle\endcsname\relax
  \providecommand{\doi}[1]{doi: #1}\else
  \providecommand{\doi}{doi: \begingroup \urlstyle{rm}\Url}\fi

\bibitem[B{\"o}ker et~al.(2024)B{\"o}ker, Levie, Huang, Villar, and Morris]{boker_finegrained_2023}
Jan B{\"o}ker, Ron Levie, Ningyuan Huang, Soledad Villar, and Christopher Morris.
\newblock Fine-grained expressivity of graph neural networks.
\newblock \emph{Advances in Neural Information Processing Systems}, 36, 2024.

\bibitem[Boucheron et~al.(2013)Boucheron, Lugosi, and Massart]{boucheron_concentration_2013}
Stéphane Boucheron, Gábor Lugosi, and Pascal Massart.
\newblock \emph{Concentration inequalities: a nonasymptotic theory of independence}.
\newblock Oxford university press, Oxford, 2013.
\newblock ISBN 978-0-19-953525-5.

\bibitem[Bullen(2013)]{bullen_handbook_mean_2013}
Peter~S Bullen.
\newblock \emph{Handbook of means and their inequalities}, volume 560.
\newblock Springer Science \& Business Media, 2013.

\bibitem[Cerviño et~al.(2023)Cerviño, Ruiz, and Ribeiro]{cervino_transference_2023}
Juan Cerviño, Luana Ruiz, and Alejandro Ribeiro.
\newblock Learning by transference: Training graph neural networks on growing graphs.
\newblock \emph{IEEE Transactions on Signal Processing}, 71:\penalty0 233--247, 2023.
\newblock \doi{10.1109/TSP.2023.3242374}.

\bibitem[Chazal et~al.(2015{\natexlab{a}})Chazal, Fasy, Lecci, Michel, Rinaldo, and Wasserman]{chazal_subsampling_2015}
Fr{\'e}d{\'e}ric Chazal, Brittany Fasy, Fabrizio Lecci, Bertrand Michel, Alessandro Rinaldo, and Larry Wasserman.
\newblock Subsampling methods for persistent homology.
\newblock In \emph{International Conference on Machine Learning}, pages 2143--2151. PMLR, 2015{\natexlab{a}}.

\bibitem[Chazal et~al.(2015{\natexlab{b}})Chazal, Glisse, Labru{{\`e}}re, and Michel]{chazal_convergence_2015}
Fr{{\'e}}d{{\'e}}ric Chazal, Marc Glisse, Catherine Labru{{\`e}}re, and Bertrand Michel.
\newblock Convergence rates for persistence diagram estimation in topological data analysis.
\newblock \emph{Journal of Machine Learning Research}, 16\penalty0 (110):\penalty0 3603--3635, 2015{\natexlab{b}}.
\newblock URL \url{http://jmlr.org/papers/v16/chazal15a.html}.

\bibitem[Chazal et~al.(2016)Chazal, Massart, and Michel]{chazal_rate_2016}
Fr{\'e}d{\'e}ric Chazal, Pascal Massart, and Bertrand Michel.
\newblock {Rates of convergence for robust geometric inference}.
\newblock \emph{Electronic Journal of Statistics}, 10\penalty0 (2):\penalty0 2243 -- 2286, 2016.
\newblock \doi{10.1214/16-EJS1161}.
\newblock URL \url{https://doi.org/10.1214/16-EJS1161}.

\bibitem[Chen et~al.(2019)Chen, Li, and Bruna]{chen2018supervised}
Zhengdao Chen, Lisha Li, and Joan Bruna.
\newblock Supervised community detection with line graph neural networks.
\newblock In \emph{International Conference on Learning Representations}, 2019.
\newblock URL \url{https://openreview.net/forum?id=H1g0Z3A9Fm}.

\bibitem[Corso et~al.(2020)Corso, Cavalleri, Beaini, Li{\`o}, and Veli{\v{c}}kovi{\'c}]{corso_principal_neighbourhood_agg_2020}
Gabriele Corso, Luca Cavalleri, Dominique Beaini, Pietro Li{\`o}, and Petar Veli{\v{c}}kovi{\'c}.
\newblock Principal neighbourhood aggregation for graph nets.
\newblock \emph{Advances in Neural Information Processing Systems}, 33:\penalty0 13260--13271, 2020.

\bibitem[Crane(2018)]{crane2018probabilistic}
H.~Crane.
\newblock \emph{Probabilistic Foundations of Statistical Network Analysis}.
\newblock Chapman \& Hall/CRC Monographs on Statistics and Applied Probability. CRC Press, 2018.
\newblock ISBN 9781351807326.
\newblock URL \url{https://books.google.fr/books?id=LERnDwAAQBAJ}.

\bibitem[Cuevas(1990)]{cuevas_pattern_1990}
Antonio Cuevas.
\newblock On pattern analysis in the non-convex case.
\newblock \emph{Kybernetes}, 19\penalty0 (6):\penalty0 26--33, 1990.

\bibitem[Cuevas and Fraiman(1997)]{cuevas_plugin_97}
Antonio Cuevas and Ricardo Fraiman.
\newblock A plug-in approach to support estimation.
\newblock \emph{The Annals of Statistics}, 25\penalty0 (6):\penalty0 2300--2312, 1997.
\newblock ISSN 00905364.
\newblock URL \url{http://www.jstor.org/stable/2959033}.

\bibitem[Cuevas and Rodríguez-Casal(2004)]{cuevas_rodriguez_casal_2004}
Antonio Cuevas and Alberto Rodríguez-Casal.
\newblock On boundary estimation.
\newblock \emph{Advances in Applied Probability}, 36\penalty0 (2):\penalty0 340–354, 2004.
\newblock \doi{10.1239/aap/1086957575}.

\bibitem[Cybenko(1989)]{cybenkot_approximation_1989}
George Cybenko.
\newblock Approximation by superpositions of a sigmoidal function.
\newblock \emph{Mathematics of control, signals and systems}, 2\penalty0 (4):\penalty0 303--314, 1989.
\newblock \doi{10.1007/BF02551274}.

\bibitem[de~Carvalho(2016)]{de_carvalho_mean_2016}
Miguel de~Carvalho.
\newblock Mean, {What} do {You} {Mean}?
\newblock \emph{The American Statistician}, 70\penalty0 (3):\penalty0 270--274, July 2016.
\newblock ISSN 0003-1305, 1537-2731.
\newblock \doi{10.1080/00031305.2016.1148632}.
\newblock URL \url{https://www.tandfonline.com/doi/full/10.1080/00031305.2016.1148632}.

\bibitem[Defferrard et~al.(2016)Defferrard, Bresson, and Vandergheynst]{defferrard_convolutional_2016}
Micha\"{e}l Defferrard, Xavier Bresson, and Pierre Vandergheynst.
\newblock Convolutional neural networks on graphs with fast localized spectral filtering.
\newblock In D.~Lee, M.~Sugiyama, U.~Luxburg, I.~Guyon, and R.~Garnett, editors, \emph{Advances in Neural Information Processing Systems}, volume~29. Curran Associates, Inc., 2016.
\newblock URL \url{https://proceedings.neurips.cc/paper_files/paper/2016/file/04df4d434d481c5bb723be1b6df1ee65-Paper.pdf}.

\bibitem[Fey and Lenssen(2019)]{fey_pytorch}
Matthias Fey and Jan~E. Lenssen.
\newblock Fast graph representation learning with {PyTorch Geometric}.
\newblock In \emph{ICLR 2019 Workshop on Representation Learning on Graphs and Manifolds}, 2019.
\newblock URL \url{https://arxiv.org/abs/1903.02428}.

\bibitem[Fout et~al.(2017)Fout, Byrd, Shariat, and Ben-Hur]{Fout2017ProteinIP}
Alex Fout, Jonathon Byrd, Basir Shariat, and Asa Ben-Hur.
\newblock Protein interface prediction using graph convolutional networks.
\newblock In \emph{NIPS}, 2017.

\bibitem[Gilmer et~al.(2017)Gilmer, Schoenholz, Riley, Vinyals, and Dahl]{gilmer_neural_2017}
Justin Gilmer, Samuel~S. Schoenholz, Patrick~F. Riley, Oriol Vinyals, and George~E. Dahl.
\newblock Neural message passing for quantum chemistry.
\newblock In Doina Precup and Yee~Whye Teh, editors, \emph{Proceedings of the 34th International Conference on Machine Learning}, volume~70 of \emph{Proceedings of Machine Learning Research}, pages 1263--1272. PMLR, 06--11 Aug 2017.
\newblock URL \url{https://proceedings.mlr.press/v70/gilmer17a.html}.

\bibitem[Goldenberg et~al.(2010)Goldenberg, Zheng, Fienberg, Airoldi, et~al.]{goldenberg2010survey}
Anna Goldenberg, Alice~X Zheng, Stephen~E Fienberg, Edoardo~M Airoldi, et~al.
\newblock A survey of statistical network models.
\newblock \emph{Foundations and Trends{\textregistered} in Machine Learning}, 2\penalty0 (2):\penalty0 129--233, 2010.

\bibitem[Gori et~al.(2005)Gori, Monfardini, and Scarselli]{gori_new_2005}
M.~Gori, G.~Monfardini, and F.~Scarselli.
\newblock A new model for learning in graph domains.
\newblock In \emph{Proceedings. 2005 {IEEE} {International} {Joint} {Conference} on {Neural} {Networks}, 2005.}, volume~2, pages 729--734 vol. 2, July 2005.
\newblock \doi{10.1109/IJCNN.2005.1555942}.
\newblock ISSN: 2161-4407.

\bibitem[Hamilton et~al.(2017)Hamilton, Ying, and Leskovec]{hamilton_inductive_2017}
Will Hamilton, Zhitao Ying, and Jure Leskovec.
\newblock Inductive representation learning on large graphs.
\newblock In I.~Guyon, U.~Von Luxburg, S.~Bengio, H.~Wallach, R.~Fergus, S.~Vishwanathan, and R.~Garnett, editors, \emph{Advances in Neural Information Processing Systems}, volume~30. Curran Associates, Inc., 2017.
\newblock URL \url{https://proceedings.neurips.cc/paper_files/paper/2017/file/5dd9db5e033da9c6fb5ba83c7a7ebea9-Paper.pdf}.

\bibitem[Hornik(1991)]{hornik_approximation_1991}
Kurt Hornik.
\newblock {Approximation capabilities of multilayer feedforward networks}.
\newblock \emph{Neural Networks}, 4\penalty0 (2):\penalty0 251--257, 1991.
\newblock \doi{10.1016/0893-6080(91)90009-T}.
\newblock URL \url{http://www.sciencedirect.com/science/article/pii/089360809190009T}.

\bibitem[Hu et~al.(2020)Hu, Fey, Zitnik, Dong, Ren, Liu, Catasta, and Leskovec]{hu_open_graph_benchmark_2020}
Weihua Hu, Matthias Fey, Marinka Zitnik, Yuxiao Dong, Hongyu Ren, Bowen Liu, Michele Catasta, and Jure Leskovec.
\newblock Open graph benchmark: Datasets for machine learning on graphs.
\newblock In H.~Larochelle, M.~Ranzato, R.~Hadsell, M.F. Balcan, and H.~Lin, editors, \emph{Advances in Neural Information Processing Systems}, volume~33, pages 22118--22133. Curran Associates, Inc., 2020.
\newblock URL \url{https://proceedings.neurips.cc/paper_files/paper/2020/file/fb60d411a5c5b72b2e7d3527cfc84fd0-Paper.pdf}.

\bibitem[Jegelka(2022)]{jegelka_theory_2022}
Stefanie Jegelka.
\newblock Theory of {Graph} {Neural} {Networks}: {Representation} and {Learning}, April 2022.
\newblock URL \url{http://arxiv.org/abs/2204.07697}.
\newblock arXiv:2204.07697 [cs, stat].

\bibitem[Keriven(2022)]{keriven_oversmoothing_2022}
Nicolas Keriven.
\newblock Not too little, not too much: a theoretical analysis of graph (over)smoothing.
\newblock In S.~Koyejo, S.~Mohamed, A.~Agarwal, D.~Belgrave, K.~Cho, and A.~Oh, editors, \emph{Advances in Neural Information Processing Systems}, volume~35, pages 2268--2281. Curran Associates, Inc., 2022.
\newblock URL \url{https://proceedings.neurips.cc/paper_files/paper/2022/file/0f956ca6f667c62e0f71511773c86a59-Paper-Conference.pdf}.

\bibitem[Keriven and Peyr\'{e}(2019)]{keriven_universal_2019}
Nicolas Keriven and Gabriel Peyr\'{e}.
\newblock Universal invariant and equivariant graph neural networks.
\newblock In H.~Wallach, H.~Larochelle, A.~Beygelzimer, F.~d'Alch\'{e} Buc, E.~Fox, and R.~Garnett, editors, \emph{Advances in Neural Information Processing Systems}, volume~32. Curran Associates, Inc., 2019.
\newblock URL \url{https://proceedings.neurips.cc/paper_files/paper/2019/file/ea9268cb43f55d1d12380fb6ea5bf572-Paper.pdf}.

\bibitem[Keriven and Vaiter(2023)]{keriven_what_2023}
Nicolas Keriven and Samuel Vaiter.
\newblock What functions can graph neural networks compute on random graphs? the role of positional encoding.
\newblock In A.~Oh, T.~Naumann, A.~Globerson, K.~Saenko, M.~Hardt, and S.~Levine, editors, \emph{Advances in Neural Information Processing Systems}, volume~36, pages 11823--11849. Curran Associates, Inc., 2023.
\newblock URL \url{https://proceedings.neurips.cc/paper_files/paper/2023/file/271ec4d1a9ff5e6b81a6e21d38b1ba96-Paper-Conference.pdf}.

\bibitem[Keriven et~al.(2020)Keriven, Bietti, and Vaiter]{keriven_convergence_2020}
Nicolas Keriven, Alberto Bietti, and Samuel Vaiter.
\newblock Convergence and {Stability} of {Graph} {Convolutional} {Networks} on {Large} {Random} {Graphs}.
\newblock In H.~Larochelle, M.~Ranzato, R.~Hadsell, M.~F. Balcan, and H.~Lin, editors, \emph{Advances in {Neural} {Information} {Processing} {Systems}}, volume~33, pages 21512--21523. Curran Associates, Inc., 2020.
\newblock URL \url{https://proceedings.neurips.cc/paper/2020/file/f5a14d4963acf488e3a24780a84ac96c-Paper.pdf}.

\bibitem[Keriven et~al.(2021)Keriven, Bietti, and Vaiter]{keriven_universality_2021}
Nicolas Keriven, Alberto Bietti, and Samuel Vaiter.
\newblock On the universality of graph neural networks on large random graphs.
\newblock In M.~Ranzato, A.~Beygelzimer, Y.~Dauphin, P.S. Liang, and J.~Wortman Vaughan, editors, \emph{Advances in Neural Information Processing Systems}, volume~34, pages 6960--6971. Curran Associates, Inc., 2021.
\newblock URL \url{https://proceedings.neurips.cc/paper_files/paper/2021/file/38181d991caac98be8fb2ecb8bd0f166-Paper.pdf}.

\bibitem[Kipf and Welling(2017)]{kipf_semi-supervised_2017}
Thomas~N. Kipf and Max Welling.
\newblock {Semi-Supervised Classification with Graph Convolutional Networks}.
\newblock In \emph{Proceedings of the 5th International Conference on Learning Representations}, ICLR '17, 2017.
\newblock URL \url{https://openreview.net/forum?id=SJU4ayYgl}.

\bibitem[Kolmogorov and Castelnuovo(1930)]{kolmogorov_moyenne_1930}
A.N. Kolmogorov and G.~Castelnuovo.
\newblock \emph{Sur la notion de la moyenne}.
\newblock G. Bardi, tip. della R. Accad. dei Lincei, 1930.
\newblock URL \url{https://books.google.fr/books?id=iUqLnQEACAAJ}.

\bibitem[Kortvelesy et~al.(2023)Kortvelesy, Morad, and Prorok]{kortvelesy_generalized_2023}
Ryan Kortvelesy, Steven Morad, and Amanda Prorok.
\newblock Generalised f-mean aggregation for graph neural networks.
\newblock In \emph{Thirty-seventh Conference on Neural Information Processing Systems}, 2023.
\newblock URL \url{https://openreview.net/forum?id=JMrIeKjTAe}.

\bibitem[Lei and Rinaldo(2015)]{lei_consestency_2015}
Jing Lei and Alessandro Rinaldo.
\newblock Consistency of spectral clustering in stochastic block models.
\newblock \emph{Ann. Statist.}, 43\penalty0 (1), February 2015.
\newblock ISSN 0090-5364.
\newblock \doi{10.1214/14-AOS1274}.
\newblock URL \url{http://arxiv.org/abs/1312.2050}.
\newblock arXiv: 1312.2050.

\bibitem[Levie(2023)]{levie_graphonsignal_2023}
Ron Levie.
\newblock A graphon-signal analysis of graph neural networks.
\newblock In A.~Oh, T.~Naumann, A.~Globerson, K.~Saenko, M.~Hardt, and S.~Levine, editors, \emph{Advances in Neural Information Processing Systems}, volume~36, pages 64482--64525. Curran Associates, Inc., 2023.
\newblock URL \url{https://proceedings.neurips.cc/paper_files/paper/2023/file/cb7943be26bb34f036c7e4068c490903-Paper-Conference.pdf}.

\bibitem[Levie et~al.(2021)Levie, Huang, Bucci, Bronstein, and Kutyniok]{levie_transferability_2021}
Ron Levie, Wei Huang, Lorenzo Bucci, Michael Bronstein, and Gitta Kutyniok.
\newblock Transferability of spectral graph convolutional neural networks.
\newblock \emph{Journal of Machine Learning Research}, 22\penalty0 (272):\penalty0 1--59, 2021.
\newblock URL \url{http://jmlr.org/papers/v22/20-213.html}.

\bibitem[Lovász(2012)]{lovasz_large_2012}
László Lovász.
\newblock \emph{Large {Networks} and {Graph} {Limits}}, volume~60 of \emph{Colloquium {Publications}}.
\newblock American Mathematical Society, Providence, Rhode Island, December 2012.
\newblock ISBN 978-0-8218-9085-1 978-1-4704-1583-9.
\newblock \doi{10.1090/coll/060}.
\newblock URL \url{http://www.ams.org/coll/060}.

\bibitem[Mallat(2012)]{mallat_scat_2012}
St{\'{e}}phane Mallat.
\newblock {Group Invariant Scattering}.
\newblock \emph{Communications on Pure and Applied Mathematics}, 65\penalty0 (10):\penalty0 1331--1398, 2012.
\newblock ISSN 00103640.
\newblock \doi{10.1002/cpa.21413}.

\bibitem[Maron et~al.(2019{\natexlab{a}})Maron, Ben-Hamu, Serviansky, and Lipman]{maron_provably_2019}
Haggai Maron, Heli Ben-Hamu, Hadar Serviansky, and Yaron Lipman.
\newblock Provably powerful graph networks.
\newblock In H.~Wallach, H.~Larochelle, A.~Beygelzimer, F.~d'Alch\'{e} Buc, E.~Fox, and R.~Garnett, editors, \emph{Advances in Neural Information Processing Systems}, volume~32. Curran Associates, Inc., 2019{\natexlab{a}}.
\newblock URL \url{https://proceedings.neurips.cc/paper_files/paper/2019/file/bb04af0f7ecaee4aae62035497da1387-Paper.pdf}.

\bibitem[Maron et~al.(2019{\natexlab{b}})Maron, Fetaya, Segol, and Lipman]{maron_universality_2019}
Haggai Maron, Ethan Fetaya, Nimrod Segol, and Yaron Lipman.
\newblock On the universality of invariant networks.
\newblock \emph{Proceedings of the 36th International Conference on Machine Learning}, 97, 2019{\natexlab{b}}.

\bibitem[Maskey et~al.(2022)Maskey, Levie, Lee, and Kutyniok]{maskey_generalization_2022}
Sohir Maskey, Ron Levie, Yunseok Lee, and Gitta Kutyniok.
\newblock Generalization analysis of message passing neural networks on large random graphs.
\newblock In S.~Koyejo, S.~Mohamed, A.~Agarwal, D.~Belgrave, K.~Cho, and A.~Oh, editors, \emph{Advances in Neural Information Processing Systems}, volume~35, pages 4805--4817. Curran Associates, Inc., 2022.
\newblock URL \url{https://proceedings.neurips.cc/paper_files/paper/2022/file/1eeaae7c89d9484926db6974b6ece564-Paper-Conference.pdf}.

\bibitem[Maskey et~al.(2023)Maskey, Levie, and Kutyniok]{maskey_transferability_2023}
Sohir Maskey, Ron Levie, and Gitta Kutyniok.
\newblock Transferability of graph neural networks: {An} extended graphon approach.
\newblock \emph{Applied and Computational Harmonic Analysis}, 63:\penalty0 48--83, 2023.
\newblock ISSN 10635203.
\newblock \doi{10.1016/j.acha.2022.11.008}.
\newblock URL \url{https://linkinghub.elsevier.com/retrieve/pii/S1063520322000987}.

\bibitem[McDiarmid(1989)]{mcdiarmid_inequality_1989}
Colin McDiarmid.
\newblock \emph{On the method of bounded differences}, page 148–188.
\newblock London Mathematical Society Lecture Note Series. Cambridge University Press, 1989.
\newblock \doi{10.1017/CBO9781107359949.008}.

\bibitem[Morris et~al.(2019)Morris, Ritzert, Fey, Hamilton, Lenssen, Rattan, and Grohe]{morris_weisfeiler_2019}
Christopher Morris, Martin Ritzert, Matthias Fey, William~L Hamilton, Jan~Eric Lenssen, Gaurav Rattan, and Martin Grohe.
\newblock Weisfeiler and leman go neural: Higher-order graph neural networks.
\newblock In \emph{Proceedings of the AAAI conference on artificial intelligence}, volume~33, pages 4602--4609, 2019.

\bibitem[Morris et~al.(2024)Morris, Frasca, Dym, Maron, İsmail~İlkan Ceylan, Levie, Lim, Bronstein, Grohe, and Jegelka]{morris_futur_2024}
Christopher Morris, Fabrizio Frasca, Nadav Dym, Haggai Maron, İsmail~İlkan Ceylan, Ron Levie, Derek Lim, Michael Bronstein, Martin Grohe, and Stefanie Jegelka.
\newblock Future directions in the theory of graph machine learning.
\newblock In \emph{Forty-first International Conference on Machine Learning}, 2024.
\newblock URL \url{https://openreview.net/forum?id=wBr5ozDEKp}.

\bibitem[Papp and Wattenhofer(2022)]{papp2022theoretical}
P{\'a}l~Andr{\'a}s Papp and Roger Wattenhofer.
\newblock A theoretical comparison of graph neural network extensions.
\newblock In \emph{International Conference on Machine Learning}, pages 17323--17345. PMLR, 2022.

\bibitem[Ruiz et~al.(2020)Ruiz, Chamon, and Ribeiro]{ruiz_graphon_neural_networks_2020}
Luana Ruiz, Luiz Chamon, and Alejandro Ribeiro.
\newblock Graphon neural networks and the transferability of graph neural networks.
\newblock \emph{Advances in Neural Information Processing Systems}, 33:\penalty0 1702--1712, 2020.

\bibitem[Ruiz et~al.(2021)Ruiz, Gama, and Ribeiro]{ruiz_stability_transferability_2021}
Luana Ruiz, Fernando Gama, and Alejandro Ribeiro.
\newblock Graph neural networks: Architectures, stability, and transferability.
\newblock \emph{Proceedings of the IEEE}, 109\penalty0 (5):\penalty0 660--682, 2021.
\newblock \doi{10.1109/JPROC.2021.3055400}.

\bibitem[Scarselli et~al.(2008)Scarselli, Gori, Tsoi, Hagenbuchner, and Monfardini]{scarselli_graph_2008}
Franco Scarselli, Marco Gori, Ah~Chung Tsoi, Markus Hagenbuchner, and Gabriele Monfardini.
\newblock The graph neural network model.
\newblock \emph{IEEE Transactions on Neural Networks}, 20\penalty0 (1):\penalty0 61--80, 2008.

\bibitem[Tremblay et~al.(2018)Tremblay, Gon{\c c}alves, and Borgnat]{tremblay_design_2017}
Nicolas Tremblay, Paulo Gon{\c c}alves, and Pierre Borgnat.
\newblock {Design of graph filters and filterbanks}.
\newblock In Petar~M. Djuri{\'c} and C{\'e}dric Richard, editors, \emph{{Cooperative and Graph Signal Processing}}, pages 299--324. {Academic Press}, June 2018.
\newblock \doi{10.1016/B978-0-12-813677-5.00011-0}.
\newblock URL \url{https://inria.hal.science/hal-01675375}.

\bibitem[Veličković et~al.(2017)Veličković, Cucurull, Casanova, Romero, Liò, and Bengio]{velickovic_graph_2018}
Petar Veličković, Guillem Cucurull, Arantxa Casanova, Adriana Romero, Pietro Liò, and Yoshua Bengio.
\newblock Graph attention networks.
\newblock \emph{6th International Conference on Learning Representations}, 2017.

\bibitem[Vershynin(2018)]{vershynin_high-dimensional_2018}
Roman Vershynin.
\newblock \emph{High-dimensional probability: An introduction with applications in data science}, volume~47.
\newblock Cambridge university press, 2018.

\bibitem[Vignac et~al.(2020)Vignac, Loukas, and Frossard]{vignac_building_2020}
Cl\'{e}ment Vignac, Andreas Loukas, and Pascal Frossard.
\newblock Building powerful and equivariant graph neural networks with structural message-passing.
\newblock In H.~Larochelle, M.~Ranzato, R.~Hadsell, M.F. Balcan, and H.~Lin, editors, \emph{Advances in Neural Information Processing Systems}, volume~33, pages 14143--14155. Curran Associates, Inc., 2020.
\newblock URL \url{https://proceedings.neurips.cc/paper_files/paper/2020/file/a32d7eeaae19821fd9ce317f3ce952a7-Paper.pdf}.

\bibitem[{Von Luxburg} et~al.(2008){Von Luxburg}, Belkin, and Bousquet]{vonluxburg_clust_2008}
Ulrike {Von Luxburg}, Mikhail Belkin, and Olivier Bousquet.
\newblock {Consistency of spectral clustering}.
\newblock \emph{Annals of Statistics}, 36\penalty0 (2):\penalty0 555--586, 2008.
\newblock ISSN 00905364.
\newblock \doi{10.1214/009053607000000640}.

\bibitem[Weisfeiler and Leman(1968)]{weisfeiler_reduction_1968}
Boris Weisfeiler and Andrei Leman.
\newblock The reduction of a graph to canonical form and the algebra which appears therein.
\newblock \emph{nti, Series}, 2\penalty0 (9):\penalty0 12--16, 1968.

\bibitem[Wu et~al.(2019)Wu, Souza, Zhang, Fifty, Yu, and Weinberger]{wu_simplifying_19}
Felix Wu, Amauri Souza, Tianyi Zhang, Christopher Fifty, Tao Yu, and Kilian Weinberger.
\newblock Simplifying graph convolutional networks.
\newblock In Kamalika Chaudhuri and Ruslan Salakhutdinov, editors, \emph{Proceedings of the 36th International Conference on Machine Learning}, volume~97 of \emph{Proceedings of Machine Learning Research}, pages 6861--6871. PMLR, 09--15 Jun 2019.
\newblock URL \url{https://proceedings.mlr.press/v97/wu19e.html}.

\bibitem[Wu et~al.(2021)Wu, Pan, Chen, Long, Zhang, and Yu]{wu_comprehensive_2021}
Zonghan Wu, Shirui Pan, Fengwen Chen, Guodong Long, Chengqi Zhang, and Philip~S. Yu.
\newblock A {Comprehensive} {Survey} on {Graph} {Neural} {Networks}.
\newblock \emph{IEEE Trans. Neural Netw. Learning Syst.}, 32\penalty0 (1):\penalty0 4--24, January 2021.
\newblock ISSN 2162-237X, 2162-2388.
\newblock \doi{10.1109/TNNLS.2020.2978386}.
\newblock URL \url{http://arxiv.org/abs/1901.00596}.
\newblock arXiv: 1901.00596.

\bibitem[Xu et~al.(2019)Xu, Hu, Leskovec, and Jegelka]{xu_how_WL_2019}
Keyulu Xu, Weihua Hu, Jure Leskovec, and Stefanie Jegelka.
\newblock How powerful are graph neural networks?
\newblock In \emph{International Conference on Learning Representations}, 2019.
\newblock URL \url{https://openreview.net/forum?id=ryGs6iA5Km}.

\end{thebibliography}
	
\end{document}